\documentclass[12pt]{article}
\usepackage[margin=1in]{geometry}

\usepackage{microtype}
\usepackage{setspace}
\usepackage{graphicx}
\usepackage{subfigure}
\usepackage{booktabs} % for professional tables
\usepackage{tcolorbox} % for professional tables
\usepackage{natbib}

\usepackage{hyperref}
\usepackage{bookmark}

\usepackage{morenotations2}

\usepackage[capitalize]{cleveref}
\usepackage{rotating}

\crefname{lemma}{Lemma}{Lemmas}
\crefname{definition}{Definition}{Definitions}

\renewcommand\cite[1]{\citep{#1}}

\begin{document}

\title{\papertitle}
\author{
    Alexander Soen\textsuperscript{\dag, $\diamond$} \quad
    Hisham Husain\textsuperscript{$\diamond$} \quad
    Richard Nock\textsuperscript{\ddag, \dag} \\
    \\[1pt]
    \small{
    Australian National University\textsuperscript{\dag} \quad 
    Amazon\textsuperscript{$\diamond$} \quad
    Google Research\textsuperscript{\ddag}
    }
}

\date{}
\maketitle

\begin{abstract}
    We introduce a boosting algorithm to pre-process data for fairness. 
    %\noteRN{
    Starting from an initial fair but inaccurate distribution, our approach shifts towards better data fitting while still ensuring a minimal fairness guarantee. To do so, it learns the sufficient statistics of an exponential family with boosting-compliant convergence.
    %}.
    %\st{Our approach involves degrading the fairness of an initial distribution to be closer to input data. This ends up constructing an exponential family distribution, where its sufficient statistics consists of weak learners from the boosting procedure}.
    Importantly,
    %with careful selection of the initial distribution 
    %\noteAS{talk about leverage? step size of updates?} 
    %\fix{%
    %and controlling the step size of each boosting step,
    %}
    we are able to theoretically prove that the learned distribution will have a \emph{representation rate} and \emph{statistical rate} data fairness guarantee. Unlike recent optimization based pre-processing methods, our approach can be easily adapted for continuous domain features. Furthermore, when the weak learners are specified to be decision trees, the sufficient statistics of the learned distribution can be examined to provide clues on sources of (un)fairness. Empirical results are present to display the quality of result on real-world data.
\end{abstract}
\vspace{3pt}
\section{Introduction}
\label{sec:introduction}

It is hard to exaggerate the importance that fairness has now taken within Machine Learning (ML) \citep{cwvrvOP,ckvDP,wmFR} (and references therein). ML being a data processing field, a common upstream source of biases leading to downstream discrimination is the data itself \citep{czWU,kmmUR,de2019bias}. Targeting data is especially important when the downstream use-case is unknown. As such, pre-processing algorithm have been introduced to debias data and mitigate unfairness in potential downstream tasks. These approaches commonly target two sources of (un)fairness \citep{ckvDP}: \ding{172} the balance of different social groups across the dataset; and \ding{173} the balance of positive outcomes between different social groups. 

Recently, optimization based approaches have been proposed to directly learn debias distributions~\citep{cwvrvOP,ckvDP}. %\noteRN{
Such approaches share two drawbacks: tuning parameters to guarantee realizable solutions or fairness guarantees can be tedious, and the distribution learned has the same \textit{discrete} support as the training sample. This latter drawback inherently prevents extending the distribution (support) and its guarantees (fairness) "beyond the training sample's".
%}
%\st{However, the computational lever of these approaches are approaches are inherently tied to the optimization libraries used} \noteAS{This is a weak argument, especially when looking at the runtimes (Appendix)}. %\st{Furthermore, these approaches are restricted to considering discrete domains.}
On the other hand, Generative Adversarial Networks (GANs) based approaches have been leveraged to learn fair distributions \citep{cbmdssGM,xyzwFF,rgTF}. These can naturally be adapted to continuous spaces. Unfortunately, in discrete spaces rounding or domain transformations need to be used.
Furthermore, these approaches often lack any theoretical fairness guarantees; or requires casual structures to be known \citep{vkbvDG}. In addition, the fairness of GAN based approaches are difficult to estimate a prior to training, where fairness is achieved via a regularization penalty or separate training rounds.
A shared downside of optimization and GAN approaches have is the lack of interpretability of how the fair distributions are learned.

\begin{figure}
    \centering
    \includegraphics[width=0.8\columnwidth, trim={37pt 55pt 38pt 38pt},clip]{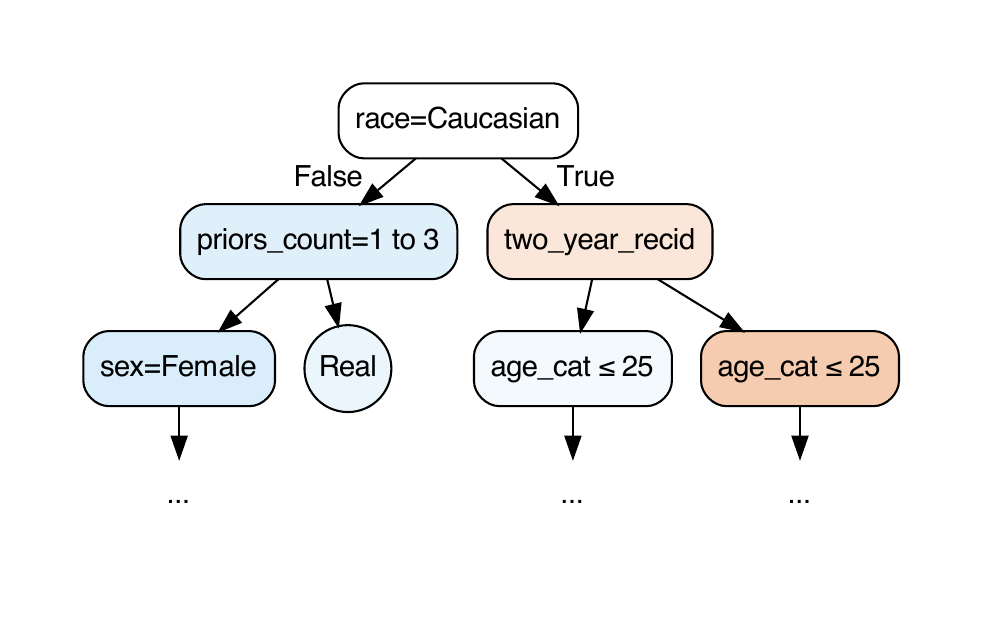}%
    \negativespace
    \caption{%
    Our approach learns weak learners at each boosting iteration of the algorithm. The weak learner attempts to differentiate between
    the `generated' current distribution (more blue) versus the `real' data (more red) --- whose predicted `realness' are used to update the current distribution.
    The following shows, the first 3 levels of a decision tree weak learner produced in the first round of boosting \mollifiera (with the \compas dataset and sensitive attribute = race). The decision tree unveils a race and sex-dependent segregation. The circular node depicts a leaf node.}%
    \label{fig:wl_teaser}%
    \negativespace
    \negativespace
\end{figure}%

\paragraph{Our Contributions} We propose a novel boosting algorithm to pre-process data: Fair Boosted Density Estimation (\fbde). Intuitively, we start with a `fair' initial distribution that is updated to move closer to the true data with each boosting step.
The learned distribution consists of an exponential family distribution with sufficient statistics given by the weak learners (WLs).
Notably, the computational complexity of our approach is proportional to the number of boosting steps and, thus, is clearly readable.
Additionally, our approach can naturally be used in both discrete and continuous spaces by choosing an appropriate initial distribution and WLs. In the case when we are examining a discrete domain (a common setting for social data benchmarks \citep{fmssAF}), interpretable WLs can be used to examine how \fbde is learning its distribution, see \cref{fig:wl_teaser}. Theoretically, we provide two styles of fairness guarantees (\cref{thm:exact_fairness,thm:relative_fairness}) which balance between fairness and `aggressiveness' of boosting updates. In addition, we provide a boosting-style analysis of convergence of \fbde (\cref{thm:kl_drop,thm:statistical_difference}). We also extend the theory of \emph{mollifiers} previously used in privacy \citep{hbcnLD} to characterize sets of fair distributions we are boosting within (\cref{lem:relative_fairness_guarantee,lem:relative_fair_distribution}). Finally, we empirically evaluate our approach with respect to data, prediction, and clustering fairness; and further present an empirical test of \fbde used for continuous domains.

%
%We summarize our contributions as
%%
%\begin{itemize}
%    \item Propose an algorithm \fbde to learn debiased distributions with fairness guarantees (\cref{thm:exact_fairness,thm:relative_fairness});
%    \item Provide a boosting-style convergence analysis of \fbde (\cref{thm:kl_drop,thm:statistical_difference});
%    \item Introduce and analysis the notation of a fair mollifier, which characterizes the space of `reltively' fair distributions (\cref{lem:relative_fairness_guarantee,lem:relative_fair_distribution});
%    \item Verify \fbde empirically through various experiments.
%\end{itemize}

\paragraph{Related Work} 
Pre-processing provides a flexible approach to algorithmic fairness in ML. Although there other points of intervention in the fairness pipeline (\eg, in-processing or post-processing fairness \citep[Section 6.2]{zvggFC}) and entirely different notation of fairness (\eg, individual fairness \citep{dhprzFT}), we limit our work to subgroup fairness for pre-processing data. That is, instead of approaches which targets fairness at a model prediction level, we instead target the data itself. 
We also note that boosting methods have seen success in post-processing methods~\citep{kgzMB,sakmmnsxFW}.

%\citet{cwvrvOP,ckvDP} both present an optimization approach to directly learn a debiased distribution.
Other pre-processing methods not already discussed include re-labeling or re-weighting~\citep
{kzLR,ckpBC,kkzDT,kcDP}. Although many of these approaches are computationally efficient, they often do not consider elements in the domain which do not appear in the dataset and lack fairness guarantees. Similarly, repair methods aim to change the input data to break the dependence on sensitive attributes for trained classifiers~\citep{jlAA,ffmsCA}. To achieve distributional repair, one can employ frameworks of optimal transport~\citep{gbflOF} or counterfactual distributions~\citep{wucRW}.

Limitations of pre-processing methods in the presence of distribution shift has also been examined: under all possible distribution shifts, there does not exist a generic non-trivial fair representation~\citep{lechner2021impossibility}. However, we note that fairness guarantees breaking due distribution shift is not a problem exclusive to pre-processing algorithms~\citep{schrouff2022diagnosing}.
Furthermore, this limitation is pessimistic and some data shifts may not significantly harm fairness. Nevertheless, from a practical standpoint one should look out for fairness degradation from distribution shift.
\section{Setting and Motivation}
\label{sec:setting}

Let \( \mathcal{X} \) be a domain of inputs, \( \mathcal{Y} \) be labels, and \( \mathcal{S} \) be a set of sensitive attributes (separate from \( \mathcal{X} \)).
We assume that both \( \mathcal{S}, \mathcal{Y} \) are finite. Unlike prior works (\ie, \citet{cwvrvOP,ckvDP}) we \emph{do not} assume that \( \mathcal{X} \) is discrete.
Denote \( \mathcal{Z} = \mathcal{X} \times \mathcal{Y} \times \mathcal{S} \).
We short hand the tuple \( (x, y, s) \) as \( z \), \ie, for a function \( f \) we have \( f(z) = f(x, y, s) \).
We further let \( \mathcal{D}(\mathcal{Z}) \) denote the set of distributions with common support \( \mathcal{Z} \). 
With slight abuse of notation, we distinguish between marginals, conditional, and joint distributions of \( \meas{P} \in \mathcal{D}(\mathcal{Z}) \) by its arguments, \ie, \( \meas{P}(x \mid y, s) \) vs \( \meas{P}(y, s) \).
%As typically assumed, we consider sampling as i.i.d.; notationally we do not differentiate between true and sampled measure to simplify exposition.
All proofs are deferred to the Appendix.

The goal of pre-processing is to correct an input distribution \( \meas{P} \in \mathcal{D}(\mathcal{Z}) \) to an output distribution \( \meas{Q} \in \mathcal{D}(\mathcal{Z}) \) which adheres to or improves a fairness criteria on the distribution.
We focus on two common data fairness criteria~\citep{ckvDP}: representation rate and statistical rate.

The first of these criteria simply measures the balance of representation of the different sensitive subgroups.
%One common criteria for data is \emph{statistical rate}~\citep{ckvDP} which aims to balance positive predictions across subgroups.

\begin{definition}\label{def:representation_rate}
    A density \( \meas{P} \in \mathcal{D}(\mathcal{Z}) \) has \( \rho \)-\textbf{representation rate} if
    \snegativespace
    \begin{equation}\label{eq:pairwise_representation_rate}
        \RR(\meas{P}, s, s') \defeq \frac{\meas{P}(s)}{\meas{P}(s')} \geq \rho, \quad \forall s, s' \in \mathcal{S}
    \end{equation}
    and define the \textbf{representation rate of a distribution} \( \meas{P} \) as \( \RR(\meas{P}) \defeq \min_{s,s' \in \mathcal{S}} \RR(\meas{P}, s, s') \in [0, 1] \).
\end{definition}

The second of these criteria measure the balance of a prediction outcome across different subgroups.

\begin{definition}\label{def:statistical_rate}
    A density \( \meas{P} \in \mathcal{D}(\mathcal{Z}) \) has \( \tau \)-\textbf{statistical rate} (w.r.t. fixed label \( y \in \mathcal{Y} \)) if
    \snegativespace
    \begin{equation}\label{eq:pairwise_statistical_rate}
        \SR(\meas{P}, s, s'; y) \defeq \frac{\meas{P}(y \mid s)}{\meas{P}(y \mid s')} \geq \tau, \quad \forall s, s' \in \mathcal{S}
    \end{equation}
    and define the \textbf{statistical rate of a distribution} \( \meas{P} \) (w.r.t. \( y\)) as \( \SR(\meas{P}) \defeq \min_{s,s' \in \mathcal{S}} \SR(\meas{P}, s, s'; y) \in [0, 1] \).
\end{definition}

The selection of label \( y \in \mathcal{Y} \) used to specify statistical rate typically corresponds to that corresponding to a positive outcome. As such, the balance being measured corresponds to the rate of positive outcomes across different subgroups. We take this convention and represent the advantaged group as \( \Y = +1 \). It should be noted that statistical rate can also be interpreted as a form of \emph{discrimination control}, inspired by the ``80\% rule'' \citep{cwvrvOP}.

Intuitively for both notions of fairness, a distribution is maximally fair when the fairness parameter is equal to 1 (\ie, \( \tau = 1 \) or \( \rho = 1\)). With such a rate requirement, the constraint requires probability of representation or the positive label rate to be equal for all subgroups.  

%For conciseness, we leave the selection of \( y \) implicit until specification in experiments.

We present an algorithm which provides an estimate of an input distribution \( \meas{P} \) which satisfies a statistical rate \( \tau \). Although we are targeting statistical rate, the learned distribution also has strong guarantees for representation rate.
%
%Alternatively to positive prediction rate of subgroups, the conditionals making the ratio in \eqref{eq:pairwise_statistical_rate} can be replaced to generate different fairness criteria. For instance, switching conditionals to marginals ``\( \meas{P}(s) \)'' gives \emph{representation rate} fairness \citep{ckvDP}. 
%We note, that our approach can be easily adapted to different ratio constraints. However, to simplify notation we only discuss statistical rate in the main text.

\subsection{Why Data Fairness?}

A reasonable question one might make is why target the data and not just make models `fair' directly. We highlight a few examples of why one might want, or even require, fairness within the data itself. 

Firstly, one might want to make data fair to allow for models trained further down in the ML pipeline to be fair. A tantalizing question data fairness faces is how it influences model fairness downstream in ML pipelines. It has previously been shown that downstream classifier compound representation injustices within the data \citep[Theorem 1]{de2019bias}. Furthermore, it has been previously shown experimentally that pre-processing approaches can improve prediction fairness for various metrics \citep{cwvrvOP,ckvDP}. As such, providing `fair' representations can be a critical components of a ML pipeline. 

Also, providing data fairness guarantees can also potentially provide improvement for non-classification notions of fairness, \ie, clustering, see experiments in \cref{sec:experiments}.

Lastly, we would also like to highlight that finding fair distributions has recently become of independent interest to certify the fairness of models \citep{kang2022certifying}. The goal is to provide a certificate of a models performance under distribution shift restricted to distribution which are statistical rate fair. As a side-effect, the evaluated model can be proven to be `fair'~\citep[Proposition 1]{kang2022certifying}.
\section{Fair Boosted Density Estimator}
\label{sec:fbde}

We propose \emph{Fair Boosted Density Estimation} (\fbde), a boosting algorithm which iteratively degrades a fair but `inaccurate' initial distribution \(\meas{Q}_{0}\) to become closer to an input data distribution, we denote as \(\meas{P}\). In particular, the user specifies a statistical rate target / budget \( \tau \in (0, 1] \) which controls the fairness degradation size of each iterative boosting update. Pseudo-code is given in \cref{algo:fbde}.

\subsection{Distribution Estimator}

The learned fair distribution \( \meas{Q}_{T} \) is an exponential family distribution constructed by iteratively aggregating classifiers.
The distribution consists of two main components: \ding{172} the initially fair distribution; and \ding{173} the boosting updates.

\paragraph{For the initial distribution} we require it to be more statistical rate fair than the target \( \tau \).
We also require the initial distribution to be representation rate fair to get a guarantee for representation rate.
As such, we define an initial distribution \( \meas{Q}_{\init} \) hierarchically. First, we specify the label-sensitive marginal given an initial budget \( \tau < \sr_{0} \leq 1 \):
\negativespace
\negativespace
\begin{align}
    \meas{Q}_{\init}(s) &\defeq {1} / {\vert \mathcal{S} \vert}; \label{eq:init_sensitive} \\
    \meas{Q}_{\init}(\Y = 1 | s) &\defeq \max\{ \meas{P}(\Y = 1 | s), \sr_{0} \cdot p_{\max} \},
    \label{eq:init_conditional}%
\negativespace
\end{align}
where \( p_{\max} = \max_{s \in \mathcal{S}} \meas{P}(\Y = 1 | s) \)\footnote{The initialization specified  increases fairness by increasing positive outcomes. Depending on the application, one may want to take the opposite strategy: replacing \cref{eq:init_conditional} by \( \min\{ \sr_{0} \cdot \meas{P}(\Y = 1 | s), p_{\min}) \} \) where \( p_{\min} = \min_{s \in \mathcal{S}} \meas{P}(\Y = 1 | s)\). }. Finally, we specify the conditional \( \meas{Q}_{\init}(x | y, s) \) for each \( (y, s) \in \mathcal{Y} \times \mathcal{S} \) by fitting an empirical distribution if \( \X \) is discrete, or a normal distribution if \( \X \) is continuous.
One can verify that \( \SR(\meas{Q}_{\init}) \geq \sr_{0}\) and \( \RR(\meas{Q}_{\init}) = 1 \). Alternatively, \cref{eq:init_sensitive} can also be altered to satisfy a weaker representation rate budget \( \rr_{0} \).
As suggested by the notation, we denote \( \sr_{0} \) and \( \rr_{0} \) as the statistical rate and representation of the initial distribution, respectively.

One additional consideration we should make is in the circumstance when finding an adequate \( \meas{Q}_{\init}\) is difficult. For instance, in low data regimes, approximating \( \meas{Q}_{\init}(x | y, s) \) for each \( (y, s) \) could be difficult (especially for under represented demographic subgroups). This can be particularly catastrophic when we are estimating the conditional distribution in a way such that by having no examples of inputs occurring in the training data, the estimated distribution has zero support for those inputs, \ie, an empirical distribution. A remedy for such as situation comes from taking a mixture distribution of our proposed initial distribution (as per \cref{eq:init_sensitive,eq:init_conditional}) and a `prior' distribution. \citet[Section 3.1 ``Prior distributions'']{ckvDP} discusses a similar approach for interpolating between distributions. \cref{sec:mixing_prior} presents additional discussion and an experimental example using real world datasets.

\paragraph{In the boosting step} of our algorithm, binary classifiers \( c_{t} : \mathcal{Z} \to \mathbb{R} \) are used to make an exponential reweighting of the previous iteration's distribution \( \meas{Q}_{t-1} \).
In particular, the classifiers \( c_{t} \) acts as a \emph{discriminator} to distinguish between real \( \meas{P} \) and fake samples \( \meas{Q}_{t-1} \) (via the \( \sign(c_{t}({z})) \)). We assume that more positive outputs of \( c_{t} \) indicate greater `realness'.
We take the common technical assumption that each \( c_{t} \) have bounded output: \( c_{t}(z) \in [-C, C] \) for some \( C > 0 \) \citep{ssIBj}.

\begin{algorithm}[t]
\caption{\fbde(\( \WL, T, \tau, \meas{Q}_{\init}, \leverage_{t} \))}
\begin{algorithmic}[1] \label{algo:fbde}
\STATE \textbf{input}: Weak learner \( \WL \), \(\#\) iter. \(T\), \( \SR \) \( \tau \),
\\\hspace{0.4cm}init. \(\meas{Q}_{\init} \), input dist. \( \meas{P} \), leverage \( \leverage_{t} \);
\STATE \( \meas{Q}_{0} \gets \meas{Q}_{\init} \)   \hspace{3mm} (with \( \sr_{0} > \tau \))
\FOR{$t = 1,\ldots,T$}
    %\STATE \( \leverage_t \gets f(t, \tau) \)
    \STATE \( c_t \gets \WL (\meas{P},\meas{Q}_{t-1}) \)
    \STATE \( \meas{Q}_t \propto \meas{Q}_{t-1} \cdot \exp(\leverage_{t} c_t) \)
\ENDFOR
\STATE \textbf{return}: $\meas{Q}_T$  \hspace{3mm} (for fairness \( \leverage_{t} \in \{\leverage_{t}^{\rm E}, \leverage_{t}^{\rm R} \} \))
\end{algorithmic}
\end{algorithm}

We assume we have a weak learner \( \WL(\meas{P}, \meas{Q}_{t-1}) \) which is used to produce such discriminators \( c_{t} \) at each step of the algorithm. Thus, after \( t \) iterations the learned distribution have the following exponential family functional form:
\begin{equation}\label{eq:joint_fbde}
    \meas{Q}_{t}({x}, {y}, {s}) = \frac{1}{Z_{t}} \exp(\leverage_{t} c_{t}({x, y, s})) \meas{Q}_{t-1}({x}, {y}, {s}),
\end{equation}
where \( \leverage_{t} \geq 0\) are \emph{leveraging coefficients} and \( Z_{t} \) is the normalizer of \( \meas{Q}_{t}({x}, {y}, {s}) \). The sufficient statistics of this exponential family distribution is exactly the learned \( c_{1} \ldots c_{t} \). Intuitively, the less (more) `real' a \( c_{t} \) deems an input to be, the lower (higher) the reweighting \( \exp( \leverage_{t} c_{t}(\cdot)) \) makes.

\subsection{Fairness Guarantees}

To establish fairness guarantees for statistical rate fairness, we require multiplicative parity of positive subgroup rates \( \meas{Q}_{t}(\Y = 1 | s) \). To calculate these rates, we define the following normalization terms:
\negativespace
\begin{align*}
    Z_{t}({y, s}) &= \int_{\mathcal{X}} \exp(\leverage_{t} c_{t}(z)) \, \dmeas{Q}_{t-1}(x | y, s); \\
    Z_{t}({s}) &= \int_{\mathcal{X} \times \mathcal{Y}} \exp(\leverage_{t} c_{t}(z)) \, \dmeas{Q}_{t-1}(x, y | s).
\negativespace
\end{align*}
We verify that \( \meas{Q}_{t}(y, s) = \meas{Q}_{t-1}(y, s) \cdot (Z_{t}(y, s) / Z_{t}) \) and \( \meas{Q}_{t}(s) = \meas{Q}_{t-1}(s) \cdot (Z_{t}(s) / Z_{t}) \).

This recursive definition of these marginal distributions provides a convenient lower bound of the statistical rate.

\begin{lemma}\label{lem:bounded_boosted_sr}%
    Suppose that \( \meas{Q}_{0} \) has statistical rate \( \sr_{0} \). Then for all \( s, s' \in \mathcal{S} \) we have
    \negativespace
    \begin{equation}
        \SR(\meas{Q}_{t}, s, s'; y) \geq \sr_{0} \cdot \prod_{i=1}^{t} \left( \frac{Z_{i}(s')}{Z_{i}(s)} \cdot \frac{Z_{i}(y, s)}{Z_{i}(y, s')} \right).
    \negativespace
    \end{equation}
\end{lemma}
It should be noted that for each of the normalizers there is a hidden dependence on \( \meas{Q}_{0} \) and \( \leverage_{i} \)'s.
As such, the pairwise statistical rate of the boosted distributions are determined by two factors: \ding{172} the initial \( \sr_{0} \); and \ding{173} the leveraging coefficients \( \leverage_{i} \). We note that similar argumentation can be made for representation rate fairness. In our approach, the coefficients \( \leverage_{i} \) are taken to be a function of the iteration number \( i \), initial fairness \( \sr_{0} \), and fairness budget \( \tau \).

Thus given the initial distribution specified above, we propose two leveraging schemes \( \leverage_{t} \) which can be used to accommodate different fairness guarantees.

\paragraph{Exact fairness} guarantees that fairness holds irrespective of other parameters of \( \meas{Q}_{T} \) (\ie boosting steps \( T \)). This type of fairness guarantee can be established by setting the leveraging coefficient as \( \eleverage_{t} := - ({C2^{t+1}})^{-1} \log (\tau / \sr_{0}) \).
\begin{theorem}
    \label{thm:exact_fairness}
    Suppose that \( \leverage_{t} = \eleverage_{t} \), then \( \SR(\meas{Q}_{T}) > \tau \) and \( \RR(\meas{Q}_{T}) > \rr_{0} \sqrt{\tau / \sr_{0}} \) for \( T \geq 1 \).
\end{theorem}
This setting is not just appealing for its absolute fairness it provides, but also the exponentially decreasing leverage \( \eleverage_{t} \) --- which fits the setting where only a few classifiers are required to provide a good estimate of \( \meas{P} \) (or are enough to break the WLA, \cref{defWLA}).
For representation rate, the \( \sqrt{\tau / \sr_{0}} \) term only degrades the initial representation rate \( \rr_{0} \) slightly when the statistical rate budget \( \tau \) is high (which is assumed as we want to achieve a fair \( \meas{Q}_{T} \)). For instance, when \( \tau = 0.8 \) and we use the initial distribution proposed (\( \rr_{0} = 1\)), then \( \RR(\meas{Q}_{T}) > \sqrt{\tau} \approx 0.894 \).

Given that the exact fairness guarantees hold regardless of \( T \), one could theoretically keep adding classifiers \( c_{t} \) forever. In practice this never happens, and thus we explore a notion of `relative' fairness. Instead of exact fairness, the guarantee on fairness gradually becomes weaker `relative' to an initial fairness constraint over update iterations.

\paragraph{Relative fairness} proves a fairness guarantee which degrades gracefully with the number of boosting iterations. To do so, we define \( \rleverage_{t} := - ({4Ct})^{-1} \log (\tau / \sr_{0}) \).

\begin{theorem}
    \label{thm:relative_fairness}
    Suppose that \( \leverage_{t} = \rleverage_{t} \), then \( \SR(\meas{Q}_{T}) > \tau^{1 + \log T} \) and \( \RR(\meas{Q}_{T}) > \rr_{0} (\sqrt{\tau / \sr_{0}})^{1 + \log T} \) for \( T \geq 1 \).
\end{theorem}

Notice the key boosting difference with \cref{thm:exact_fairness}: the sum of the series of leveraging coefficients diverges, so relative fairness accommodates for more aggressive boosting schemes. \cref{tab:boosting_rates} summarizes the implications of the two theorems. It is not surprising that both leveraging schemes display $\leverage_{t}^{.} \rightarrow 0$ as $\tau / \sr_{0} \rightarrow 1$, as maximal fairness forces the distribution to stick to $\meas{Q}_0$ and is therefore data oblivious. Differences are apparent when we consider the number of boosting iterations $T$: should we boost for $T=5$, we still get \( \SR(\meas{Q}_{T}) > \tau^{2.3} \) and \( \RR(\meas{Q}_{T}) > \tau^{1.15} \) with relative fairness (taking \( \sr_{0} = \rr_{0} = 1 \)), which can still be reasonable depending on the problem. In real-world data scenarios (see \cref{sec:experiments}), we find that large numbers of boosting iterations can be taken without major degradation of fairness.

\begin{table}[t]%
    \negativespace
\caption{Summary of of different leveraging schemes \( \leverage^{.}_{t} \) in \fbde. }
\label{tab:boosting_rates}
\centering
\begin{small}%
\begin{sc}%
\begin{tabularx}{\columnwidth}{MMMD}
\toprule
Fair  & \( \leverage_{t} \) & \( \SR(\meas{Q}_{t}) \) & Size \( \varepsilon_{t} \) (\( \sr_{0} = 1 \)) \\
\midrule
Exact & \( O(2^{-t}) \) & \( \tau \) & \( - \log \tau \) \\
Rel.  & \( O(t^{-1}) \) & \( \tau^{\Omega(\log t)} \) & \( - (1 + \log t) \log \tau \) \\
\bottomrule
\end{tabularx}
\end{sc}
\end{small}
\negativespace
\negativespace
\end{table}%

\subsection{Sampling for \texorpdfstring{\( \meas{Q}_{T} \)}{Q\_T}}

A desirable property for the learned \( \meas{Q}_{T} \) would be the ability to sample efficiently from it. In practice, to utilize the weak learner \( \WL \) to create classifiers \( c_{t} \), we require samples from \( \meas{Q}_{t-1} \). In the case when \( \X \) consists of a discrete finite domain, one can simply enumerate the possible inputs and sample from a large multinomial distribution. 
However, when \( \X \) consists of continuous random variables, we cannot use the same strategy.
%A key difficulty comes from the requirement of calculating the normalizing constant \( Z_{t} \).
%
Instead, we utilize Langevin Monte Carlo (LMC) and a sampling trick to efficiently calculate samples. It should be noted that LMC can be directly applied to sampling \( \meas{Q}_{t} \) if \( \Y, \SSS \) are approximated to be continuous, \ie via \citet{gshdmOI} or \citet{cbmdssGM,xyzwFF,rgTF}.

We first note that the conditional distribution \( \meas{Q}_{t}(x | y, s) \) is of a similar functional form given by \cref{eq:joint_fbde}. As such, a natural way to sample from these conditionals is via LMC sampling algorithms using the distribution's \emph{score function}:%
\negativespace
\begin{equation}\label{eq:boosting_cond_score}%
    \grad \log \meas{Q}_{t}(x | y, s) = \sum_{i=1}^{t} \leverage_{i} \grad c_{i}(z) + \grad \log \meas{Q}_{0}(x | y, s),
\snegativespace
\end{equation}%
where \( z = (x, y, s) \), as noted in \cref{sec:setting}.
Of course, we require the classifiers \( c_{t} \) and log-likelihood of the initial distribution \( \meas{Q}_{0} \) to be differentiable. The former can be achieved by taking \( c_{t} \)'s to be simple neural networks. The latter can be achieved using our proposed initial distribution (taking \( \meas{Q}_{0}(x | y, s) \) to be normal distributions).

Secondly, the marginal distribution \( \meas{Q}_{t}(y, s) \) can be calculated by \emph{sampling from \( \meas{Q}_{0} \)}:
\negativespace
\begin{equation}\label{eq:marginal_by_sampling_init}
    \hspace{-1em}\meas{Q}_{t}(y, s) \propto \meas{Q}_{0}(y, s)
    \mathop{\expect}_{x | y, s \sim \meas{Q}_{0}} \left[
    \exp\left(\sum_{k=1}^{t} \leverage_{k} c_{k}(z)\right) \right],
\snegativespace
\end{equation}
where the expectation can be approximated by sampling from \( \meas{Q}_{0}(x | y, s) \).
%\noteRN{Remove: let the reviewers ask} \st{We note that this approximation is similar to the ``reparameterization trick''.} \noteAS{Maybe something better?}

With these ingredients, we can now state our sampling routine. For any \( t \), we first calculate the current marginal \( \meas{Q}_{t}(y, s) \) via \cref{eq:marginal_by_sampling_init}. We can now hierarchically sample the joint distribution by 
first sampling from \( \meas{Q}_{t}(y, s) \), and then sampling from \( \meas{Q}_{t}(x | y, s) \) using LMC. 

\paragraph{\fbde as Boosting the Score}
An interesting perspective of \fbde is to examine the overdamped Langevin diffusion process which can be used to sample \( \meas{Q}_{t} \). Assuming the necessary assumptions of continuity and differentiability (including \( \Y \) and \( \SSS \) to simplify the narrative), the (joint) diffusion process of \( \meas{Q}_{t} \), as per \eqref{eq:boosting_cond_score}, is given by
\negativespace
\begin{equation}\label{eq:boosting_diffusion}
   \dot{\Z}_{t} = \sum_{i=1}^{t} \leverage_{i} \grad c_{i}(\Z) + \grad \log \meas{Q}_{0}(\Z) + \sqrt{2} \dot{\W},
\negativespace
\end{equation}
where \( \W \) is standard Brownian motion and \( \Z = (\X, \Y, \SSS)\). In the limit (w.r.t. the time derivative), the distribution of \( \Z \) approaches \( \meas{Q}_{t} \)~\citep{rtEC}.

Removing the initial summation, \cref{eq:boosting_diffusion} simplifies to the Langevin diffusion process for the initial diffusion process --- the score function of \( \meas{Q}_{0} \) `pushes' the process to areas of high likelihood. Reintroducing the first summation term, as the goal of \( c_{i}(.) \)'s is to predict the `realness' of samples, the term defines a vector fields which `points' towards the most `real' direction in the distribution. Thus, the WLs \( c_{i}(.) \) can be interpreted as correction terms which `pushes' the diffusion process closer to \( \meas{P} \).

\subsection{Convergence}

To discuss properties of convergence \fbde has, we first introduce a variant of \emph{weak learning assumption} of boosting~\citep{hbcnLD,cnBD}.

\begin{definition}[WLA]\label{defWLA}
    A learner \( \WL(\cdot, \cdot) \) satisfies the {\bf weak learning assumption} (WLA) for \( \gamma_{\meas{P}}, \gamma_{\meas{Q}} \in (0, 1] \) iff for all \( \meas{P}, \meas{Q} \in \mathcal{D}(\mathcal{Z}) \), \( \WL(\meas{P}, \meas{Q}) \) produces a \emph{discriminator} \( c : \mathcal{Z} \to \mathbb{R} \) satisfying \( \expect_{\meas{P}}[c] \geq C \cdot \gamma_{\meas{P}} \) and \( \expect_{\meas{Q}}[-c] \geq C \cdot \gamma_{\meas{Q}} \).
\end{definition}

Intuitively, the WLA constants \( \gamma_{\meas{P}}, \gamma_{\meas{Q}} \) measures a degree of separability between the two distributions \( \meas{P} \) and \( \meas{Q} \) --- with \( \gamma_{\meas{P}} = \gamma_{\meas{Q}} = 1 \) giving maximal separability. 
 As such, in our algorithm, as \( \meas{Q}_{t} \rightarrow \meas{P} \) finding higher boosting constants \( \gamma^{t}_{\cdot} \) becomes harder. Although our WLA may appear dissimilar to the typical WLAs which rely on a single inequality \citep{ssIBj}, it has been shown that \cref{defWLA} is equivalent \citep[Appendix Lemma 7]{cnBD}. When applying the WLA to \fbde, we refer to constants \( \gamma_{\meas{P}}^{t}, \gamma_{\meas{Q}}^{t}\) when referring to the WL learned via \( \WL({\meas{P}}, \meas{Q}_{t-1}) \), \ie, learning the \( t \)\textsuperscript{th} classifier \( c_{t} \).

 Using the WLA, we examine the progress we can make per boosting step. In particular, we examine the drop in KL between input \( \meas{P} \) and successive boosted densities \( \meas{Q}_{t-1}, \meas{Q}_{t} \).

\begin{theorem}\label{thm:kl_drop}
    Let \( \leverage_{t} \) be any leverage such that \( \leverage_{t} \leq 1 \) for all \( t \), \( \tau / \sr_{0} > \exp(-4C) \), and \( C > 0 \). If WL satisfies WLA (for \( \meas{P}, \meas{Q}_{t-1}\)) with \( \gamma_{\meas{P}}^t, \gamma_{\meas{Q}}^t \), then:
    \snegativespace
    \begin{equation}\label{eq:kl_drop}
        \kl(\meas{P}, \meas{Q}_{t-1}) - \kl(\meas{P}, \meas{Q}_{t}) \geq \leverage_{t} \cdot \Lambda_{t},
    \end{equation}
    where \( \Lambda_{t} =  \gamma_{\meas{P}}^{t} \cdot C + \Gamma(\var_{\meas{Q}_{t-1}}[c_{t}], \gamma_{\meas{Q}}^t) \); \( \var_{\meas{Q}_{t}}[c_{t}] \) denotes the variance of \( c_{t} \); and \( \Gamma(., .) \) is decreasing \wrt the first argument and increasing \wrt the second argument.
\end{theorem}

The full definition of \( \Gamma(., .) \) can be found in the Appendix. Intuitively, more accurate WLs (which gives higher \( \gamma_{\cdot}^{\cdot}\)'s) leads to a higher KL drop.
The above is a strict improvement upon \citet[Theorem 5]{hbcnLD} and, therefore, extends the theoretical guarantees for boosted density estimation (not just \fbde) at large. In particular, we find that a constraint on the variance of the classifier allows us to remove the two boosting regimes analysis, prevalent in \citet{hbcnLD} --- the conditions for \( > 0 \) KL drop only depends on the variance and WLA constants now.

Specifically, Theorem~\ref{thm:kl_drop} allows for smaller values of \( \gamma_{\meas{Q}}^{t} \) to result in a \( > 0 \) KL drop than \citet[Theorem 5]{hbcnLD} (which requires \( \gamma_{\meas{Q}}^{t} > 1/3 \) for \( C = \log 2 \)).
In particular, taking \( C = \log 2 \) then having \( \var_{\meas{Q}_{t-1}}[c_{t}] < C^2 \cdot 2/3 \approx 0.32 \) allows us to have a \( > 0 \) KL drop with smaller \( \gamma_{\meas{Q}}^{t} \) constants (with an extended discussion in \cref{sec:analyzing_kl_drop_bound}).
Notably, given the bounded nature of the classifier, the variance is already bounded with \( \var_{\meas{Q}_{t-1}}[c_t] \leq C^2 \approx 0.48 \) for \( C = \log 2 \) --- thus the condition itself is reasonable.
In addition, we should note that convergence, unlike in the privacy case~\citep{hbcnLD}, depends on how the fairness \( \tau \) parameter interacts with the update's leveraging coefficient \( \leverage_{t} \).

In addition to consider the per iteration KL drop, we consider ``how far from \( \meas{Q}_{0} \)'' we progressively get, in an information-theoretic sense. For this, we define \( \Delta(\meas{Q}) \defeq \kl(\meas{P}, \meas{Q}_{0}) - \kl(\meas{P}, \meas{Q}) \). To simplify the analysis, we assume that constants \( \gamma_{\meas{P}}, \gamma_{\meas{Q}} \) are fixed throughout the boosting process (or simply taking worse-case constants).

\begin{theorem} \label{thm:statistical_difference}
    Suppose that \( \lambda = -\log( \tau / \sr_{0}) \), \( \alpha(\gamma) \:= \min_{t} \Gamma( \var_{\meas{Q}_{t-1}}[c_t], \gamma) / (\gamma C) \), and \( C, T > 0 \).
    \\
    \( \bullet \) If \( \vartheta_{t} := \vartheta_{t}^{\rm E} \), then%
    \negativespace
    \snegativespace
    \begin{equation*}
        \Delta(\meas{Q}_{T}) \in \lambda \left( 1 - 2^{-T} \right) \left[ \frac{\gamma_{\meas{P}} + \gamma_{\meas{Q}} \cdot \alpha(\gamma_{\meas{Q}})}{2}, 1  \right];%
    \negativespace
    \end{equation*}
    \( \bullet \) If \( \vartheta_{t} := \vartheta_{t}^{\rm R} \), then 
    \snegativespace
    \begin{equation*}
        \Delta(\meas{Q}_{T}) \in \frac{\lambda}{2} \left[ \frac{\gamma_{\meas{P}} + \gamma_{\meas{Q}} \cdot \alpha(\gamma_{\meas{Q}})}{2} \left( \frac{1}{T} + \log T \right), (1 + \log T)  \right],
    \negativespace
    \end{equation*}
    where \( a[b, c] = [ab, ac] \).
\end{theorem}

In addition to the worse-case convergence rates (lower bounds) typically analyzed in boosting algorithms, \cref{thm:statistical_difference} also characterizes the best-case scenario (upper bounds) achievable by our algorithm.
%\noteRN{
Notably, the gap, computed as the ratio best-to-worst, solely depends on the parameters of the WLA, and thus is guaranteed to be reduced as the WL becomes better.
%}
%\st{Notably this best-case scenario directly relates to the fairness budget (see UUU)}.
As \( T \rightarrow \infty \), then with the exact leverage \( \leverage_{t}^{\rm E} \) we have that \( \Delta(\meas{Q}_{T}) = \Theta(-\log (\tau / \sr_{0})) \). Additionally assuming \( \gamma_{\meas{P}}, \gamma_{\meas{Q}} \rightarrow 1 \) and setting \( \var_{\meas{Q}_{t-1}}[c_{t}] \rightarrow 0 \) uniformly, the upper and lower bound exactly matches at \( -\log(\tau / \sr_{0})\) --- our bound is tight. Despite this, with the relative leverage \( \leverage_{t}^{\rm R} \), there will always be a \( \Omega(- \log(\tau / \sr_{0})) \) gap between the upper and lower bound of \( \Delta(\meas{Q}_{T})\). Furthermore, in contrast to the geometric convergence of the bounds in the exact fairness case, in relative fairness the bounds grow logarithmically with \( T \). However, we do note that for relative leverage we can achieve a tight bound if we do not attempt to simplify the series \( \sum_{k}^{T} \leverage_{k} \), see \cref{thm:delta_kl_gen_bounds}.

On particularly nice interpretation of \cref{thm:statistical_difference} is to utilize the lower bounds of \( \Delta(\meas{Q}_{T})\) to determine sufficient conditions for small KL. In particular, we find the sufficient number of boosting steps \( T \) to make \( \kl(\meas{P}, \meas{Q}_{T}) \) small.
\begin{corollary}
    \label{cor:suff_boost_steps}
    Suppose that the same condition in \cref{thm:statistical_difference} hold and \( \varepsilon < \kl(\meas{P}, \meas{Q}_{0}) \). Then \( \kl(\meas{P}, \meas{Q}_{T}) < \varepsilon \), for:
    \\
    \( \bullet \) \( \vartheta_{t} := \vartheta_{t}^{\rm E} \) if
    \begin{equation*}
        T > \frac{1}{\log 2} \cdot \log\left( \left(1 - 2 \cdot \frac{\kl(\meas{P}, \meas{Q}_{0}) - \varepsilon}{\lambda \cdot (\gamma_{\meas{P}} + \gamma_{\meas{Q}} \cdot \alpha(\gamma_{\meas{Q}}))} \right)^{-1} \right).
    \end{equation*}
    \( \bullet \) \( \vartheta_{t} := \vartheta_{t}^{\rm R} \) if
    \begin{equation*}
        T > \exp\left( 4 \cdot \frac{\kl(\meas{P}, \meas{Q}_{0}) - \varepsilon}{\lambda \cdot (\gamma_{\meas{P}} + \gamma_{\meas{Q}} \cdot \alpha(\gamma_{\meas{Q}}))} \right).
    \end{equation*}
\end{corollary}

One should note, that when comparing the exact leverage versus the relative versus in \cref{cor:suff_boost_steps}, we are comparing the growth of two functions: \( x \mapsto \log((1 - 2x)^{-1}) \) for exact leverage; and \( x \mapsto \exp(4x) \) for relative leverage. The former dominates asymptotic dominates the latter. That is, using the exact leverage will require more boosting updates to achieve same drop in KL; which follows our intuition that relative leveraging has better boosting convergence at the cost of decaying fairness. 
\snegativespace
\section{Mollifier Interpretation}
\label{sec:mollifier}

Previously it has been shown that boosted density estimation for privacy can be interpreted as boosting within a set of a set of private densities --- dubbed a mollifier \citep{hbcnLD}. To study the set of distributions \fbde boosts within, we adapt the \emph{relative mollifier} construction for fairness.

\begin{definition}\label{def:relmol}
    %\noteAS{@Richard + @Hisham: Should we include the absolute continuity footnote here?}\noteRN{No, we need space: let the reviewers ask -- or put it in supplement}
    Let \( \meas{Q}_{0} \) be a reference distribution. Then the \( \varepsilon \)-fair {\bf relative mollifier} \( \mathcal{M}_{\varepsilon, \meas{Q}_{0}} \) is the set of distributions \( \meas{Q} \) satisfying \( \forall s, s^{\prime} \in \mathcal{S} \)
    \snegativespace
    \begin{equation} 
        \max \left\{\frac{\SR(\meas{Q}, s, s^{\prime}; y)}{\SR(\meas{Q}_{0}, s, s^{\prime}; y)}, \frac{\SR(\meas{Q}_{0}, s, s^{\prime}; y)}{\SR(\meas{Q}, s, s^{\prime}; y)} \right\}  \leq e^{\varepsilon}.\label{eq:relative_mollifier}
    \negativespace
    \end{equation}
\end{definition}

The reference distribution \( \meas{Q}_{0} \) can be interpreted as the initial distribution chosen in \fbde.
Notably, one can change the constraint in \cref{eq:relative_mollifier} to accommodate for different notions of data fairness, \ie, replacing ``\( \SR \)'' to ``\( \RR \)''.
The following \cref{lem:relative_fairness_guarantee,lem:relative_fair_distribution} holds identically for representation rate style mollifiers --- with the first lemma providing a fairness guarantees for \emph{all} elements in \( \mathcal{M}_{\varepsilon, \meas{Q}_{0}} \). 
%Notably, \emph{all} elements of \( \mathcal{M}_{\varepsilon, \meas{Q}_{0}} \) have a fairness guarantee.

\begin{lemma}\label{lem:relative_fairness_guarantee}
    If \( \meas{Q} \in \mathcal{M}_{\varepsilon, \meas{Q}_{0}} \), then \( \SR(\meas{Q}) \geq \sr_{0} \cdot \exp(- \varepsilon) \).
\end{lemma}

\begin{figure*}[t]
    \centering
    \hfill
    \includegraphics[width=0.46\textwidth,trim={0pt, 10pt, 0pt, 0pt},clip]{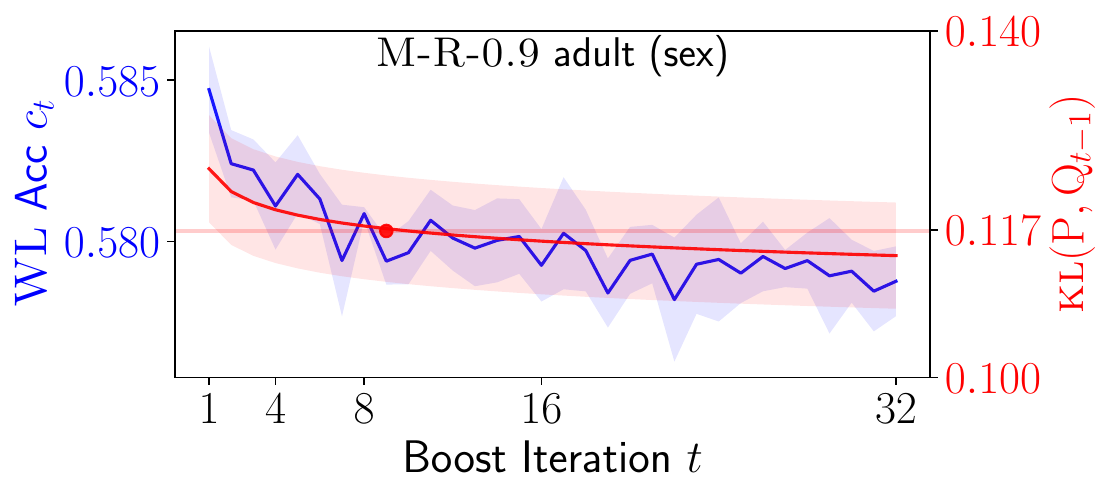}%
    \hfill
    \includegraphics[width=0.46\textwidth,trim={0pt, 10pt, 0pt, 0pt},clip]{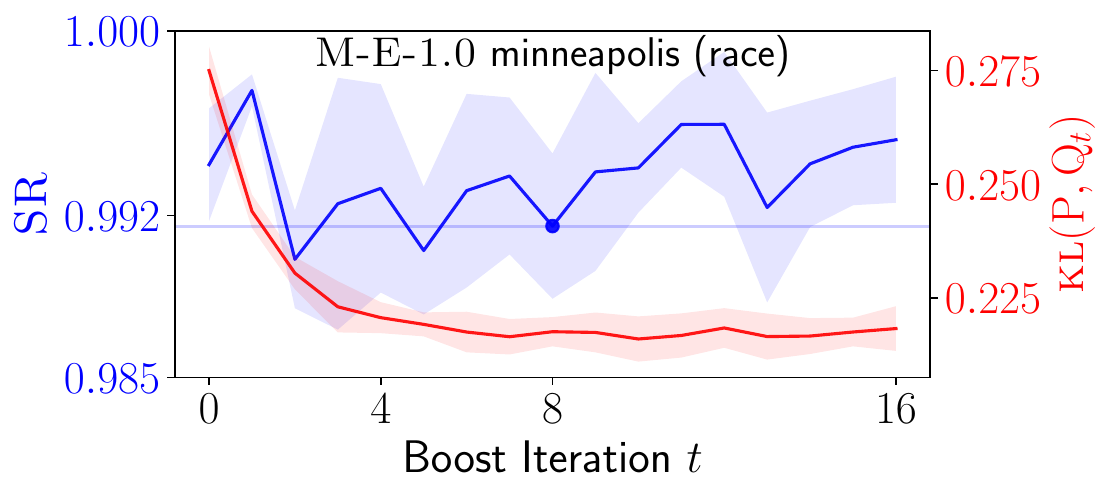}%
    \hfill
    \negativespace
    \caption{Evaluations of \fbde over boosting iterations. Left: WL accuracy vs KL of \mollifierd over boosting iterations on the \adult dataset. 
    Right: \(\SR\) vs KL of \mollifiera over boosting iterations on the \textsc{Minneapolis} dataset ---  original data's \( \SR \) is \( 0.684 \pm 0.005 \). 
    The shaded region depicts the 1 std. range. The horizontal line / point indicates the \( T = 8 \) (left) \( \kl(\meas{P}, \meas{Q}_{8})\) value and (right) \( \SR(\meas{Q}_{8}) \) value.%
    }
    \label{fig:boost_iters_joint}%
    \negativespace
\end{figure*}

\cref{lem:relative_fairness_guarantee} shows a significant difference from the privacy case: the distribution within a carefully constructed mollifier are fair; not just the sampler~\citep{hbcnLD}. This difference is significant in fairness as knowledge of selected elements in the mollifier can elucidate sources of (un)fairness learned by a model. Furthermore, this detail is crucial in applications where we need a set of fair distributions~\citep{kang2022certifying}.
%Given that having a slightly fair element ensures the fairness of a mollifier's elements, we adapt the relative mollifier construction \citep{hbcnLD}.
Given \cref{lem:relative_fairness_guarantee}, a natural question to ask is, \emph{what kinds of fair distributions are included in the relative mollifier?} The following lemma answers this question.

\begin{lemma} \label{lem:relative_fair_distribution}
    Suppose that \( \mathcal{M}_{\varepsilon, \meas{Q}_{0}} \) is a relative mollifier
    and \( \meas{Q} \in \mathcal{D}(\mathcal{Z}) \) has \( \SR(\meas{Q}) \geq \exp(-\varepsilon) \). If \( \forall s, s^{\prime} \in \mathcal{S} \) we have either \( \SR(\meas{Q}, s, s^{\prime}; y) > 1 \) or \( \SR(\meas{Q}, s, s^{\prime}; y) \leq \SR(\meas{Q}_{0}, s, s^{\prime}; y) \leq 1 \), then \( \meas{Q} \in \mathcal{M}_{\varepsilon, \meas{Q}_{0}} \).
\end{lemma}

Intuitively, the condition of membership in \cref{lem:relative_fair_distribution} can be interpreted as \( \meas{Q} \) and \( \meas{Q}_{0} \) having a shared `type' of (un)fairness: letting \( \mathcal{S} = \{ \rm{male}, \rm{female} \} \), if the probability of positive outcomes in both \( \meas{Q}, \meas{Q}_{0} \) is higher given \( \SSS = \rm{male} \) than \( \SSS = \rm{female} \) but \( \meas{Q}_{0} \) is more fair, then \( \meas{Q} \in \mathcal{M}_{\varepsilon, \meas{Q}_{0}} \).
A significant instantiation of \cref{lem:relative_fair_distribution} is when \( \sr_{0} = 1 \), which gives us that \( \SR(\meas{Q}) \geq \exp(-\varepsilon) \implies \meas{Q} \in \mathcal{M}_{\varepsilon, \meas{Q}_{0}}\). 
In other-words, we have a \emph{complete} mollifier: by combining \cref{lem:relative_fairness_guarantee}, \( \meas{Q} \in \mathcal{M}_{\varepsilon, \meas{Q}_{0}} \iff \SR(\meas{Q}) \geq \exp(\varepsilon) \) which notably including all perfectly fair distributions.

We can now express \fbde as the process of mollification \( \argmin_{\meas{Q} \in \mathcal{M}} \kl(\meas{P}, \meas{Q}) \)~\citep{hbcnLD}, \ie, finding the closest element of a mollifier \( \mathcal{M}_{\varepsilon, \meas{Q}_{0}} \) \wrt to an input \( \meas{P} \). In particular, by taking \( \varepsilon = - \log (\tau / \sr_{0} ) \) the statistical rate conditions in \cref{lem:relative_fairness_guarantee,lem:relative_fair_distribution} correspond exactly to the fairness budget \( \tau \) of \fbde. 
Taking \( \sr_{0} = 1 \), the minimization of KL (as per \cref{thm:kl_drop}) can be interpreted as an iterative boosting procedure for mollification. Even in the case when \( \mathcal{M}_{\varepsilon, \meas{Q}_{0}} \) is \emph{incomplete} (\( \sr \neq 1 \)), it is still useful as careful selection of \( \meas{Q}_{0} \) can practically yield all distributions we are interested in, \ie, those close to \( \meas{P} \).
Furthermore, it is fine to further focus on a subset of a mollifier.
Especially so in our case, where we only consider a `boost'-able exponential family subset of \( \mathcal{M}_{\varepsilon, \meas{Q}_{0}} \); we note the strong approximation capabilities of exponential families~\citep[Section 1.5.4]{ngSE}.
%In our case, we are only considering the `boost'-able exponential family subset of \( \mathcal{M}_{\varepsilon, \meas{Q}_{0}} \).

One interesting perspective of \fbde and its different fairness settings is to examine the mollifiers they are boosting within. In particular, we can consider the \( \varepsilon_{t} \) parameter in \cref{def:relmol}, or the `\emph{size}', of the relative mollifiers \wrt iteration value \( t \) --- summarized in \cref{tab:boosting_rates}. Assuming \( \sr_{0} = 1 \), using the exact leverage \( \leverage_{t}^{\rm E} \) implies that we are boosting within a constant sized mollifier as per \cref{thm:exact_fairness}, with \( \varepsilon_{t} = -\log (\tau / \sr_{0})\). For the relative leverage \( \leverage_{t}^{\rm R} \), as the fairness degrades in this case, the corresponding mollifier grows, for which \( \varepsilon_{t} = - (1 + \log t) \log (\tau / \sr_{0}) \). Notice that size of the mollifier also coincides (up to constant multiplication) with the upper bounds of \( \Delta(\meas{Q}_{T}) \) as per \cref{thm:statistical_difference}.
\begin{table*}[th]
    \negativespace
    \caption{\fbde~(\(T = 32\)) and baselines evaluated for \compas and \adult datasets. The table reports the mean and s.t.d.}
    \label{tab:summary_experiments_table}
    \centering
    \scriptsize
    \begin{sc}%
    \begin{tabularx}{\textwidth}{LLPDDDDDDDD} %{LLMXXXXXXXX}
        \toprule
        && & Data & \mollifiera & \mollifierb & \mollifierc & \mollifierd & \celisa & \tabfairgan & \fairkmeans \\
        \midrule
{\multirow{9}{*}{\rotatebox[origin=c]{90}{\compas (\(\mathcal{S}\) = race)}}} &
{\multirow{3}{*}{\rotatebox[origin=c]{90}{\underline{data}}}} 
& $ \textrm{RR} $                 & \(.662 \pm .007\) & \(.966 \pm .008\) & \(.977 \pm .008\) & \(.944 \pm .010\) & \(.964 \pm .009\) & \(.992 \pm .004\) & \(.632 \pm .052\) &  \( - \) \\
&& $ \textrm{SR} $                & \(.747 \pm .013\) & \(.988 \pm .006\) & \(.899 \pm .011\) & \(.978 \pm .011\) & \(.896 \pm .011\) & \(.992 \pm .006\) & \(.727 \pm .100\) &  \( - \) \\
&& $ \textrm{KL} $                & \( - \) & \(.135 \pm .020\) & \(.129 \pm .020\) & \(.132 \pm .020\) & \(.127 \pm .020\) & \(.164 \pm .018\) & \(2.71 \pm .735\) &  \( - \) \\
        \cmidrule(l{5pt}){2-11}

&{\multirow{3}{*}{\rotatebox[origin=c]{90}{\underline{pred}}}} 
& $ \textrm{SR}_{\textrm{c}}$     & \(.747 \pm .020\) & \(.959 \pm .025\) & \(.875 \pm .025\) & \(.945 \pm .027\) & \(.872 \pm .024\) & \(.939 \pm .018\) & \(.802 \pm .074\) &  \( - \) \\
&& $ \textrm{EO} $                & \(.781 \pm .034\) & \(.960 \pm .026\) & \(.900 \pm .041\) & \(.950 \pm .028\) & \(.895 \pm .039\) & \(.944 \pm .020\) & \(.815 \pm .079\) &  \( - \) \\
&& $ \textrm{Acc} $               & \(.660 \pm .004\) & \(.641 \pm .014\) & \(.653 \pm .012\) & \(.642 \pm .010\) & \(.656 \pm .012\) & \(.648 \pm .015\) & \(.591 \pm .043\) &  \( - \) \\
        \cmidrule(l{5pt}){2-11}
        
&{\multirow{3}{*}{\rotatebox[origin=c]{90}{\underline{clus}}}} 
& $ \cPR $            & \(.334 \pm .019\) & \(.332 \pm .022\) & \(.334 \pm .019\) & \(.334 \pm .019\) & \(.319 \pm .036\) & \(.333 \pm .020\) & \(.286 \pm .027\) & \(.259 \pm .013\) \\
&& $ \cSR $           & \(.288 \pm .055\) & \(.235 \pm .082\) & \(.284 \pm .055\) & \(.284 \pm .055\) & \(.256 \pm .084\) & \(.288 \pm .059\) & \(.268 \pm .063\) & \(.284 \pm .036\) \\
&& $ \textrm{dist} $  & \(.065 \pm .000\) & \(.078 \pm .000\) & \(.078 \pm .000\) & \(.078 \pm .000\) & \(.078 \pm .000\) & \(.078 \pm .000\) & \(.078 \pm .001\) & \(.069 \pm .001\) \\
        \midrule

{\multirow{9}{*}{\rotatebox[origin=c]{90}{\adult (\(\mathcal{S}\) = sex)}}} &
{\multirow{3}{*}{\rotatebox[origin=c]{90}{\underline{data}}}} 
& $ \textrm{RR} $                 & \(.496 \pm .002\) & \(.958 \pm .003\) & \(.979 \pm .003\) & \(.919 \pm .004\) & \(.957 \pm .003\) & \(.995 \pm .002\) & \(.516 \pm .023\) &  \( - \) \\
&& $ \textrm{SR} $                & \(.360 \pm .005\) & \(.961 \pm .005\) & \(.883 \pm .006\) & \(.924 \pm .006\) & \(.865 \pm .006\) & \(.979 \pm .004\) & \(.862 \pm .101\) &  \( - \) \\
&& $ \textrm{KL} $                & \( - \) & \(.122 \pm .006\) & \(.119 \pm .006\) & \(.113 \pm .006\) & \(.114 \pm .006\) & \(.182 \pm .005\) & \(1.68 \pm .538\) &  \( - \) \\

        \cmidrule(l{5pt}){2-11}
&{\multirow{3}{*}{\rotatebox[origin=c]{90}{\underline{pred}}}} 
& $ \textrm{SR}_{\textrm{c}}$     & \(.360 \pm .003\) & \(.818 \pm .010\) & \(.766 \pm .010\) & \(.793 \pm .011\) & \(.753 \pm .008\) & \(.919 \pm .011\) & \(.823 \pm .118\) &  \( - \) \\
&& $ \textrm{EO} $                & \(.471 \pm .008\) & \(.959 \pm .016\) & \(.908 \pm .018\) & \(.935 \pm .016\) & \(.895 \pm .016\) & \(.981 \pm .010\) & \(.867 \pm .105\) &  \( - \) \\
&& $ \textrm{Acc} $               & \(.803 \pm .003\) & \(.785 \pm .002\) & \(.788 \pm .002\) & \(.787 \pm .002\) & \(.788 \pm .002\) & \(.773 \pm .005\) & \(.781 \pm .006\) &  \( - \) \\
        \cmidrule(l{5pt}){2-11}
        
&{\multirow{3}{*}{\rotatebox[origin=c]{90}{\underline{clus}}}} 
& $ \cPR $           & \(.125 \pm .067\) & \(.093 \pm .019\) & \(.101 \pm .025\) & \(.116 \pm .027\) & \(.111 \pm .033\) & \(.130 \pm .061\) & \(.137 \pm .038\) & \(.117 \pm .039\) \\
&& $ \cSR $           & \(.426 \pm .232\) & \(.252 \pm .082\) & \(.302 \pm .076\) & \(.190 \pm .061\) & \(.254 \pm .086\) & \(.468 \pm .237\) & \(.437 \pm .179\) & \(.296 \pm .151\) \\
&& $ \textrm{dist} $  & \(.034 \pm .000\) & \(.037 \pm .000\) & \(.037 \pm .000\) & \(.037 \pm .000\) & \(.037 \pm .000\) & \(.037 \pm .000\) & \(.037 \pm .000\) & \(.035 \pm .000\) \\
        \bottomrule
    \end{tabularx}
    \end{sc}
    \negativespace
    \negativespace
\end{table*}
\snegativespace
\section{Experiments}
\label{sec:experiments}
In this section, we (1) verify that \fbde debiased densities and ad-hears to specified data fairness measures; (2) inspect the implications of utilizing samples produced by a \fbde debiased density in downstream tasks, specifically, prediction and clustering tasks; (3) explore the interpretability of \fbde when utilizing decision tree (DT) weak learners (WLs); and (4) present an experimental on a dataset with continuous \( \mathcal{X} \).

To analyze these points, we evaluate \fbde over pre-processed \compas (binary \(\mathcal{S} = \rm{race}\)) and \adult (binary \( \mathcal{S} = \rm{sex}\)) datasets provided by \texttt{AIF360}\footnote{Public at: \url{www.github.com/Trusted-AI/AIF360}} \citep{aif360}. This consists of a discrete binary domain.

We consider 4 configurations of \fbde, boosted for \( T = 32 \) iterations. We consider a fixed fairness budget \( \tau = 0.8\) throughout and take a combination of exact vs relative leverage; and \( \sr_{0} = 1 \) vs \( \sr_{0} = 0.9\). We designate each configuration by \config, where \configleverage encodes the leverage and \configbase encodes the base rate, \ie, \mollifiera uses exact leverage with \( \sr_{0} = 1 \). DT WLs are calibrated using Platt's method \citep{platt1999probabilistic}.
For baselines, we consider two different data pre-processing approaches. Firstly, we consider the max entropy approach proposed by \citet{ckvDP} (\celisa) with default parameters. Secondly, we compare against TabFairGAN proposed by \citet{rgTF} (\tabfairgan), which includes a separate training phase for fairness.
%uses two separate training phases: one for accuracy; and a second for fairness and accuracy. 
%
%For baselines, we utilize two configuration of max entropy approaches for data preprocessing presented in \citet{ckvDP} -- with statistical rate setting \( \rho = 1 \) and prior interpolation parameter \( C = 0.5 \) with a marginal vector constraint determined by: (\celisa) the weighted mean from \citet[Algorithm 1]{ckvDP}; and (\celisb) the empirical expectation vector with the marginal set to ensure equal representation of the sensitive attribute.
%For baselines, we utilize two configuration of max entropy approaches for data preprocessing presented in \citet{ckvDP} -- with statistical rate setting \( \rho = 1 \) and prior interpolation parameter \( C = 0.5 \) with a marginal vector constraint determined by: (\celisa) the weighted mean from \citet[Algorithm 1]{ckvDP}; and (\celisb) the empirical expectation vector with the marginal set to ensure equal representation of the sensitive attribute.
In clustering, we consider a R{\'e}nyi Fair K-means (\( K = 4 \)) approach as per \citet{bnbrRF} (\fairkmeans).

In evaluating all approaches, we utilize 5-fold cross validation and evaluate all measurements using the test set (whenever appropriate). All training was on a MacBook Pro (16 GB memory, M1, 2020).

Additional experiments and discussion is presented in the Appendix, including, additional dataset comparisons, in-processing algorithm comparison, and the runtime of approaches.
Code for \fbde is available at \url{www.github.com/alexandersoen/fbde}.

%\begin{figure}[t]
%    \centering
%    \includegraphics[width=0.95\columnwidth,trim={0pt, 10pt, 0pt, 0pt},clip]{figures/wl_kl/M-R-0.9_adult_sex.pdf}%
%    \caption{\WL~accuracy vs KL of \mollifierd over boosting iterations on the \adult dataset. The shaded region depicts the 1 std. range. The horizontal line and point indicates the \( \kl(\meas{P}, \meas{Q}_{8})\) value. }%
%    \label{fig:boost_iters}%
%    \negativespace
%    \negativespace
%\end{figure}%

\paragraph{Data Fairness}
\cref{tab:summary_experiments_table} summarizes all results for \compas and \adult.
To evaluate the data fairness of our approach and baselines \celisa and \tabfairgan we compare the representation rate \( \RR \) and statistical rate \( \SR \). We first note that \tabfairgan does not target \( \RR \) directly.
To evaluate the information maintained after debiasing, we measure the KL between the original dataset and debiased samples.

All approaches apart from \tabfairgan provides an increase in fairness (for both \( \RR \) and \( \SR \)). Across \fbde variants, we notice a trade-off between fairness and KL, where exact leveraging provides stronger fairness than relative leveraging at the cost of higher KL.
\fbde performs better for both fairness and KL when compared to \tabfairgan.
We notice that \tabfairgan struggles in the smaller \compas dataset and can even harm fairness (notice the s.t.d.).
Notably, the KL is significantly worse than other approaches as a result of miss-matching of input distribution \( \meas{P} \) support (with small \( \delta = 10^{-9} \) added for unsupported regions) --- possibly as a result of the transformation from discrete to continuous domains \citep{rgTF}; in addition to possible mode collapse in training \citep{ttvIG}.
\celisa has the best \( \SR \), although it comes with a slight cost in KL when compared to \fbde.

It should be noted that \fbde's fairness target / budget is set to \( \tau = 0.8 \). So if a higher \( \SR \) is required, \( \tau \) can be increased. Although we take \( T = 32 \), as the leveraging coefficients decrease rapidly, we would expect a majority of the learning to occur in the initial iterations of the algorithm. This is indeed the case, as shown in \cref{fig:boost_iters_joint} (left). 
We also note that the decrease in KL is proportional to the accuracy of WLs. This follows the intuition of \cref{thm:kl_drop} --- more accurate weak learners implies larger boosting constants \( \gamma_{\cdot} \), which leads to larger drops in KL.

\paragraph{Prediction Fairness}
To evaluate the prediction fairness, we evaluate a decision tree classifier (\clf) (from \sklearn with max depth of 32) trained on debiased samples. \clf~is evaluated on \clf's statistical rate (with \( \hat{Y} = 1 \)) (\( \SRclf \)) and equality of opportunity ratio / true positive rate ratio (\EO).
Accuracy (\Acc) is evaluated to measure the degradation from utilizing debiased samples.

When comparing for prediction, all approaches compared to data provide an increase for both \( \SR_{\textrm{c}} \) and \( \EO \) with a slight decrease in \( \Acc \). For \compas, the \fbde variants and \mollifiera are comparable across fairness and utility; with \tabfairgan slightly lower in fairness and accuracy scores. In \adult, the \fbde variants and \tabfairgan are similar in performance, with \celisa having the best  \( \SR_{\textrm{c}} \) and \( \EO \) at the cost of the worst \( \Acc \).

\paragraph{Clustering Fairness}
To evaluate the clustering fairness, a K(\(=4\))-Means classifier (from \sklearn) is trained using the debiased samples.
The fairness considerations in clustering is two fold.
First we measure the difference in fairness across clusters (lower is better): we take the difference between the min and max ratio of privileged data points (\( \SSS=1 \)) in each cluster (as per \citet{cklvFC}) (\( \cPR \)); and between ratios of statistical rate (\( \cSR \)).
%Secondly, we measure the worst-case statistical rate across clusters (\( \cSRm \)) (higher is better).
Secondly, to evaluate the quality, we measure the average Manholobis distance to designated cluster centers (\textsc{dist}).
%To evaluate the quality, we measure the average Manholobis distance to designated cluster centers (\textsc{dist}).

%\cref{tab:clustering_experiments_table} summarizes the clustering fairness performance.
\fairkmeans has the best utility \textsc{dist} across approaches. Interestingly, pre-processing approaches can still be competitive to \fairkmeans across the fairness measures. In particular, \( \cSR \) of the pre-processing approaches, particularly \fbde algorithms, can actually beat \fairkmeans --- which is perhaps expected as data \( \SR \) is specifically targeted. In \adult, \( \cPR \) was also significantly improved using pre-processing.

%
%\begin{figure}[t]
%    \centering
%    \includegraphics[width=0.95\columnwidth,trim={0pt, 10pt, 0pt, 0pt},clip]{figures/sr_kl/M-E-1.0_minneapolis_race.pdf}%
%    \caption{\(\SR\) vs KL of \mollifiera over boosting iterations on the \textsc{Minneapolis} dataset. The shaded region depicts the 1 s.t.d. range. The original data's \( \SR \) is \( 0.684 \pm 0.005 \).
%    }%
%    \label{fig:cts_main}%
%    \negativespace
%    \negativespace
%\end{figure}

\paragraph{Interpretability}
Uniquely, \fbde allows for the learned distribution to be examined by analyzing the WLs learned. For instance, \cref{fig:wl_teaser} depicts a DT WL learned for \compas by \mollifiera in one of its folds. Importantly, the `realness' and `fakeness' equate to modifications of the initial perfectly fair distribution \( \meas{Q}_{0} \) to be come closer to \( \meas{P} \): by \cref{eq:joint_fbde} real predictions are up-weighted while fake predictions are down-weighted. This can be used to examine the underlying bias within input \( \meas{P}\). For example, taking `True' (twice) on the ``race''-``two\_year\_recid'' path ends on a majority ``Fake'' node, which indicates that to get closer to \( \meas{P} \), the probability of non-Caucasians re-offenders was increased.

From a proactive point of view, one can instead restrict the hypothesis class in which the WL outputs. For instance, one can restrict the input domain to not include sensitive attributes --- however, this may a cause harm in KL utility, and may still potentially learn spurious correlations in the form of proxies~\citep{dfkmsUP}. Alternatively, the interpretability of the WLs allows for a human-in-the-loop algorithm: an auditor can become the stopping criteria of \fbde by comparing the gain in utility of an additional WL against the possible unfairness its reweighting could causes.

\paragraph{Continuous Domain}
In addition to the discrete datasets of \compas and \adult, we consider the \textsc{Minneapolis} police stop dataset\footnote{Public at: \url{www.opendata.minneapolismn.gov/datasets/police-stop-data}} (binary \( \mathcal{S} = \rm{race} \))
which contains numeric / continuous features. In particular, we consider only a subset of features, taking position, race, and `person searched'. We evaluate \mollifiera
%
%and \mollifierc
with 1 hidden layer (20 neuron) neural networks as WLs, with \( T = 16 \) iterations.

\cref{fig:boost_iters_joint} (right) plots the change in \( \SR \) versus KL across boosted updates. Here, a small \( \delta = 10^{-9} \) probability value is added to a binned domain to estimate the KL. We verify that the \( \SR \) stays above the budget in this continuous domain setting, where \( \SR(\meas{P}) = 0.684 \pm 0.005 \). Similarly, the \( \RR \) is improved (see Appendix). However, we note that the learned \( \RR \) is lower than what theory may suggest, \ie, \mollifiera gives \( \RR(\meas{Q}_{T}) = 0.847 \pm 0.002 \) which is slightly lower than the expected \( \sqrt{\tau} \approx 0.894 \) given by \cref{thm:exact_fairness} (original \( \RR(\meas{P}) = 0.183 \pm 0.001 \)). We attribute this miss-match in theory due to numerical approximation error in the estimation of samples via LMC and the MCMC estimation of \cref{eq:marginal_by_sampling_init}, where better estimates could improve performance. For instance, in our experiments we only use the Unadjusted Langevin algorithm for LMC; more sophisticated methods can be used for better samples \citep{rtEC}.
One differing aspect readers may notice is the non-monotonicity of the KL curve. In addition to numerical approximation error, the non-monotonicity may be occurring due to the binned (\( \delta \)) estimation of the KL.

\snegativespace
\section{Limitations and Conclusion}\label{sec-conc}

In this paper, we introduce a new boosting algorithm, \fbde, which learns exponential family distributions which are both representation rate and statistical rate fair.
To conclude, we highlight a few limitations of \fbde.

Firstly, \fbde is a boosting algorithm and thus relies on the WLA holding, \ie, our fairness and convergence guarantees rely on this. Thus the performance of the WLs should be interrogated to ensure that \fbde does not harm fairness and cause subsequent social harm, \eg, \cref{fig:boost_iters_joint} (left).

Secondly, we reiterate that \fbde is not an unilateral replacement of other (types of) fairness algorithms, where specialized algorithms can provide stronger guarantees for specific criteria. Instead, \fbde targets an upstream source of unfairness in the ML pipeline, which can eventually allow for other forms of fairness.
Nevertheless, when attempting to achieve these downstream fairness metrics one should be careful by examining \fbde's performance per boosting iterations. Indeed, some data settings can cause \fbde to have critical failure in, \eg, \( \SRclf \), see \cref{sec:extra_discrete_experiments}.

Lastly, we note that the fairness guarantees (\cref{thm:exact_fairness,thm:relative_fairness}) can break in practice as a result of numerical approximation error --- especially in continuous domain datasets. We leave improvements of, \eg, sampling for \fbde in continuous domains for future work.
%, which is particularly useful when the data in question will be used in a variety of tasks (with each their own fairness considerations).
%We leave the adaptions of mollifiers to specific different fairness constraints to future work.
\section*{Acknowledgements}

AS thanks members of the ANU Humanising Machine Intelligence program for discussions on fairness and ethical concerns in AI, and the
NeCTAR Research Cloud for providing computational resources, an Australian research platform
supported by the National Collaborative Research Infrastructure Strategy.
We thank Lydia Lucchesi for discussion regarding experiment datasets.

\bibliography{main}
\bibliographystyle{cust}

\newpage
\appendix
\onecolumn

\counterwithin{theorem}{section}
\counterwithin{figure}{section}
\counterwithin{table}{section}

\renewcommand\thesection{\Alph{section}}
\renewcommand\thesubsection{\thesection.\Roman{subsection}}
\renewcommand\thesubsubsection{\thesection.\Roman{subsection}.\arabic{subsubsection}}

\renewcommand{\thetable}{\Roman{table}}
\renewcommand{\thefigure}{\Roman{figure}}

\begin{center}
\Huge{Appendix}
\end{center}

\vspace{20pt}

\begin{abstract}
This is the Appendix to Paper "\papertitle". To
differentiate with the numbering in the main file, the sectioning is letter-based (A, B, ..., \etc). Thus, Theorems in the Appendix will be displayed as (A.1, B.3, ..., \etc). Additionally, figure and tables in the Appendix are numbered with Roman numerals (I, II, ..., \etc).
\end{abstract}

\vspace{10pt}

\section*{Table of contents}

\newcommand{\tocrow}[2]{%
\noindent $\hookrightarrow$ \cref{#1}: #2\hrulefill Pg \pageref{#1}\\
}% For some reason \nameref does not work with the current .sty

\noindent \textbf{Proofs} \\
\tocrow{sec:pf_bounded_boosted_sr}{Proof of \cref{lem:bounded_boosted_sr}}
\tocrow{sec:pf_exact_fairness}{Proof of \cref{thm:exact_fairness}}
\tocrow{sec:pf_relative_fairness}{Proof of \cref{thm:relative_fairness}}
\tocrow{sec:pf_kl_drop}{Proof of \cref{thm:kl_drop}}
\tocrow{sec:pf_statistical_difference}{Proof of \cref{thm:statistical_difference}}
\tocrow{sec:pf_suff_boost_steps}{Proof of \cref{cor:suff_boost_steps}}
\tocrow{sec:pf_relative_fairness_guarantee}{Proof of \cref{lem:relative_fairness_guarantee}}
\tocrow{sec:pf_relative_fair_distribution}{Proof of \cref{lem:relative_fair_distribution}}

\noindent \textbf{Additional Technical Discussion} \\
\tocrow{sec:analyzing_kl_drop_bound}{Analyzing KL Drop Bound}
\tocrow{sec:general_fair_mollifiers}{General Fair Mollifiers}
\tocrow{sec:mixing_prior}{Improvement in the Low / Sparse Data Regime via Prior Mixing}

\noindent \textbf{Additional Experiments} \\
\tocrow{sec:additional_dataset_descriptions}{Additional Dataset Descriptions}
\tocrow{sec:additional_discrete_experiments}{Extended Discrete Experiments}
\tocrow{sec:additional_continuous_experiments}{Additional Continuous Experiments}
\tocrow{sec:extra_discrete_experiments}{Extra Discrete Experiments}
\tocrow{sec:dutch_german_without_mixing}{\dutch and \german Without Mixing Priors}
\tocrow{sec:in_processing_experiments}{In-Processing Experiments}

\newpage

%%%% Proofs
\section{Proof of \texorpdfstring{\cref{lem:bounded_boosted_sr}}{3.1}} \label{sec:pf_bounded_boosted_sr}

\begin{myproof}
    To lower bound the pairwise statistical rate, we utilize the `unrolled' computation of the required marginal measures. For arbitrary \( y, s \) we have
    \begin{align*}
        \meas{Q}_{t}(y | s)
        &= \frac{\meas{Q}_{t}(y, s)}{\meas{Q}_{t}(s)} 
        = \frac{\meas{Q}_{t-1}(y, s)}{\meas{Q}_{t-1}(s)} \cdot \frac{Z_{t}(y, s)}{Z_{t}(s)} 
         \cdots 
        =  \frac{\meas{Q}_{0}(y, s)}{\meas{Q}_{0}(s)} \cdot \prod_{i=1}^{t}  \frac{Z_{i}(y, s)}{Z_{i}(s)} \\
        &=  \meas{Q}_{0}(y | s) \cdot \prod_{i=1}^{t}  \frac{Z_{i}(y, s)}{Z_{i}(s)}.
    \end{align*}
    Thus, the pairwise statistical rate for \( s, s' \) is recovered immediately by taking a lower bound on the statistical rate of \( \meas{Q}_{0} \).
\end{myproof}
\section{Proof of \texorpdfstring{\cref{thm:exact_fairness}}{3.2}}
\label{sec:pf_exact_fairness}

To prove \cref{thm:exact_fairness}, we split the proof into the statistical rate component of the theorem and the representation rate component of the theorem.

\subsection{Statistical Rate}

To prove the exact fairness guarantee for statistical rate, we first use the following lemmas.

\begin{lemma}\label{lem:bounded_general_sr}
    For \( T \geq 1 \), we have
    \begin{equation}
        \SR(\meas{Q}_{T}) \geq \sr_{0} \cdot \exp\left( -4C \sum_{i=1}^{T} \leverage_{t}\right).
    \end{equation}
\end{lemma}
\begin{myproof}
    We first consider bounds on the difference between log-normalizer (log-partition) functions.
    First by taking the smallest and largest values of the classifier \( c_{t}(.) \) (which assumed to be bounded in \( [-C, C] \)), we have
    \begin{align*}
        - C \cdot \leverage_{t} \leq \leverage_{t} c_{t}(x, y, s) \leq C \cdot \leverage_{t}
    \end{align*}
    Then by taking the exponential, integrand (w.r.t., measure \( \dmeas{Q}_{t-1}(x | y, s) \)), and logarithm, we get
    \begin{eqnarray}
        - C \leverage_{t} \leq &\log \left( \int_{\mathcal{X}} \exp(\leverage_{t} c_{t}(x, y, s)) \dmeas{Q}_{t-1}(x | y, s) \right) & \leq C \leverage_{t} \nonumber \\
        - C \leverage_{t} \leq & \log Z_{t}(y, s) & \leq C \leverage_{t}. \label{eq:log_partition_bound}
    \end{eqnarray}
    Then by taking the largest difference between \( \log Z_{t}(y, s) \) and \( \log Z_{t}(y, s') \), we have that for any \( s, s' \)
    \begin{equation}\label{eq:joint_normalizer_bound}
        - 2C \leverage_{t} \leq \log Z_{t}(y, s) - \log Z_{t}(y, s') \leq 2C \leverage_{t}.
    \end{equation}
    Identically, we can bound the difference for \( \log Z_{t}(s) \) by replacing the measure in the integrand step with \( \dmeas{Q}_{t-1}(x,y | s) \), giving:
    \begin{equation}\label{eq:marginal_normalizer_bound}
        - 2C \leverage_{t} \leq \log Z_{t}(s) - \log Z_{t}(s') \leq 2C \leverage_{t}.
    \end{equation}

    \begin{remark}\label{remark:log_partition_bound}
        We can also bound \( \log Z_{t} \) (\ie, \cref{eq:log_partition_bound}) using the same method by, again, replacing the measure to be \( \dmeas{Q}_{t-1} \)
    \end{remark}
    
    Together, \eqref{eq:joint_normalizer_bound} and \eqref{eq:marginal_normalizer_bound} allows us to simplify \cref{lem:bounded_boosted_sr}:
    \begin{align*}
        \SR(\meas{Q}_{t}, s, s'; y) 
        &\geq \sr_{0} \cdot \prod_{i=1}^{t} \left( \frac{Z_{i}(s')}{Z_{i}(s)} \cdot \frac{Z_{i}(y, s)}{Z_{i}(y, s')} \right) 
        = \sr_{0} \cdot \exp\left(\sum_{i=1}^{t} \log\left( \frac{Z_{i}(s')}{Z_{i}(s)} \cdot \frac{Z_{i}(y, s)}{Z_{i}(y, s')} \right)\right) \\
        &\geq \sr_{0} \cdot \exp\left(-4C \sum_{i=1}^{t} \leverage_i \right).
    \end{align*}
    As this holds for any \( t, s, s' \), the bound holds for \( \SR(\meas{Q}_{T}) \) as required.
\end{myproof}

We can now prove the statistical rate part of \cref{thm:exact_fairness}.

\begin{myproof}
    Taking \( \leverage_{t} = - ({C2^{t+1}})^{-1} \log (\tau / \sr_{0}) \), from \cref{lem:bounded_general_sr} we get:
    \begin{align*}
        \SR(\meas{Q}_{t})
        &\geq \sr_{0} \cdot \exp\left(-4C \sum_{i=1}^{t} \leverage_i \right)
        = \sr_{0} \cdot \exp\left(2 \log\left(\frac{\tau}{\sr_{0}}\right) \sum_{i=1}^{t} \frac{1}{2^{i}} \right) \\
        &= \sr_{0} \cdot \exp\left(2 \log\left(\frac{\tau}{\sr_{0}}\right) \left( 1 - \frac{1}{2^t} \right) \right) \\
        &\geq \sr_{0} \cdot \exp\left(\log\left(\frac{\tau}{\sr_{0}}\right)\right)
        = \tau.
    \end{align*}
    The last inequality follows that we are only considering \( t \geq 1 \).
    Thus \( \SR(\meas{Q}_{T}) > \tau \) as required.
\end{myproof}

We now move to proving the representation rate component of \cref{thm:exact_fairness}.

\subsection{Representation Rate}

%\section{Proof of \cref{thm:exact_rr}}
%\label{sec:pf_exact_rr}

To lower bound the corresponding representation rate, we present a similar lemma to \cref{lem:bounded_general_sr}.

\begin{lemma} \label{lem:bounded_general_rr}
    For \( T \geq 1 \), we have
    \begin{equation}
        \RR(\meas{Q}_{T}) \geq \rr_{0} \cdot \exp\left( -2C \sum_{i=1}^{T} \leverage_{t} \right).
    \end{equation}
\end{lemma}
\begin{myproof}
    Proving \cref{lem:bounded_general_rr} is similar to the proof of \cref{lem:bounded_boosted_sr}. Specifically, we use \cref{eq:marginal_normalizer_bound} to bound the representation rate.
    
    First, we calculate the marginal distribution via a recursive relation similar to \cref{lem:bounded_boosted_sr}:
    \begin{align*}
        \meas{Q}_{t}(s) = \meas{Q}_{t-1}(s) \cdot \frac{Z_{t}(s)}{Z_{t}} = \ldots = \meas{Q}_{0}(s) \cdot \prod_{i=1} \frac{Z_{i}(s)}{Z_{t}}.
    \end{align*}

    Now, we \cref{eq:marginal_normalizer_bound} to bound the resulting representation rate:
    \begin{align*}
        \RR(\meas{Q}_{t}, s, s^{\prime})
        &\geq \rr_{0} \cdot \prod_{i=1}^{t} \frac{Z_{i}(s)}{Z_{i}(s^{\prime})}
        = \rr_{0} \cdot \exp\left( \sum_{i=1}^{t} \log \left(\frac{Z_{i}(s)}{Z_{i}(s^{\prime})}\right) \right) \\
        &\geq \rr_{0} \cdot \exp\left( -2C \sum_{i=1}^{t} \leverage_{t} \right).
    \end{align*}
    As the bound holds for all \( s, s^{\prime} \), the bound holds for \( \RR(\meas{Q}_{t})\) required.
\end{myproof}

\begin{remark}
    Notice that \cref{lem:bounded_general_rr} holds even when we restrict \( \meas{Q}_{t} \) to only have non-sensitive features \( x \) and sensitive features \( s \) (\ie, no labels being separately considered). This allows for leveraging schemes \( \leverage_{t} \) to be designed for representation rate specifically. For instance, for a target representation rate \( \tau_{\rr} \) we can establish exact and relative representation rate fairness guarantees by replacing ``\( \tau / \sr_{0} \)'' to ``\( \tau_{\rr} / \rr_{0} \)''.
\end{remark}

We can now prove the representation rate component of \cref{thm:exact_fairness}.

\begin{myproof}
    Taking \( \leverage_{t} = - ({C2^{t+1}})^{-1} \log (\tau / \sr_{0}) \), from \cref{lem:bounded_general_rr} we get:
    \begin{align*}
        \RR(\meas{Q}_{t})
        &\geq \rr_{0} \cdot \exp\left( -2C \sum_{i=1}^{t} \leverage_{t} \right)
        = \rr_{0} \cdot \exp\left( \log \left(\frac{\tau}{\sr_{0}}\right) \sum_{i=1}^{t} \frac{1}{2^{i}} \right) \\
        &= \rr_{0} \cdot \exp\left( \log \left(\frac{\tau}{\sr_{0}}\right) \left( 1 - \frac{1}{2^t} \right) \right) \\
        &\leq  \rr_{0} \cdot \exp\left(\frac{1}{2} \log \left(\frac{\tau}{\sr_{0}}\right) \right)
        %= \rr_{0} \cdot \exp\left( \frac{1}{2} \cdot \log \left(\frac{\tau}{\sr_{0}}\right) \right)
        = \rr_{0} \cdot \sqrt{\frac{\tau}{\sr_{0}}}.
    \end{align*}
    As required.
\end{myproof}
\section{Proof of \texorpdfstring{\cref{thm:relative_fairness}}{3.3}}
\label{sec:pf_relative_fairness}

The relative fairness guarantee is proven similarly to the exact fairness guarantee in \cref{sec:pf_exact_fairness}. We also split the proof into statistical rate and representation rate components of the theorem.

\subsection{Statistical Rate}

\begin{myproof}
    Taking \( \leverage_{t} = - ({4Ct})^{-1} \log (\tau / \sr_{0}) \), from \cref{lem:bounded_general_sr} we get:
    \begin{align*}
        \SR(\meas{Q}_{t})
        &\geq \sr_{0} \cdot \exp\left(-4C \sum_{i=1}^{t} \leverage_i \right) \\
        &= \sr_{0} \cdot \exp\left(\log\left(\frac{\tau}{\sr_{0}}\right) \sum_{i=1}^{t} \frac{1}{i} \right) \\
        &> \sr_{0} \cdot \exp\left(\log\left(\frac{\tau}{\sr_{0}}\right) \cdot \left( 1 + \int_{1}^{t} \frac{1}{i} \dmeas{}i \right) \right) \\
        &= \sr_{0} \cdot \exp\left(\log\left(\frac{\tau}{\sr_{0}}\right) \cdot \left( 1 + \log t \right) \right) \\
        &= \sr_{0} \cdot \exp\left(\log\left(\frac{\tau}{\sr_{0}}\right)\right) \cdot \exp\left(\log\left(\frac{\tau}{\sr_{0}}\right) \cdot \log t \right) \\
        &= \tau \cdot \left( \frac{\tau}{\sr_{0}} \right)^{\log t} \\
        &= \sr_{0}^{- \log t} \cdot \tau^{1 + \log t}.
    \end{align*}
    Here we note that \( \log (\tau / \sr_{0}) < 0 \) as \( \tau < \sr_{0} \). Furthermore, note that \( \sr_{0}^{-\log t} \) is an increasing function of \( t \geq 1 \), Thus taking \( t = 1 \) for this term:
    \begin{align*}
        \SR(\meas{Q}_{t}) > \sr_{0}^{- \log 1} \cdot \tau^{1 + \log t} = \tau^{1 + \log t}.
    \end{align*}
    Thus \( \SR(\meas{Q}_{T}) > \tau^{1 + \log T} \) as required.
\end{myproof}

\subsection{Representation Rate}

%\section{Proof of \cref{thm:relative_rr}}
%\label{sec:pf_relative_rr}

Using \cref{lem:bounded_general_rr}, we prove the representation rate component of \cref{thm:relative_fairness}.

\begin{myproof}
    Taking \( \leverage_{t} = - ({4Ct})^{-1} \log (\tau / \sr_{0}) \), from \cref{lem:bounded_general_rr} we get:
    \begin{align*}
        \RR(\meas{Q}_{t})
        &\geq \rr_{0} \cdot \exp\left( -2C \sum_{i=1}^{t} \leverage_{t} \right) \\
        &= \rr_{0} \cdot \exp\left( \frac{1}{2} \cdot \log \left(\frac{\tau}{\sr_{0}}\right) \sum_{i=1}^{t} \frac{1}{i} \right) \\
        &> \rr_{0} \cdot \exp\left( \frac{1}{2} \cdot \log \left(\frac{\tau}{\sr_{0}}\right) \left( 1 + \int_{1}^{t} \frac{1}{i} \dmeas{}i \right) \right) \\
        &= \rr_{0} \cdot \exp\left( \frac{1}{2} \cdot \log \left(\frac{\tau}{\sr_{0}}\right) \left( 1 + \log t \right) \right) \\
        &= \rr_{0} \cdot \left(\sqrt{\frac{\tau}{\sr_{0}}}\right)^{1 + \log t}.
    \end{align*}
    As required.
\end{myproof}
\section{Proof of \texorpdfstring{\cref{thm:kl_drop}}{3.5}}
\label{sec:pf_kl_drop}

\begin{lemma}\label{lem:kl_drop}
    The \kl-drop is given by:
    \begin{equation}
        \kl(\meas{P}, \meas{Q}_{t-1}) - 
        \kl(\meas{P}, \meas{Q}_{t})
        = \leverage_{t} \cdot \expect_{\meas{P}}[c_{t}] - \log \expect_{\meas{Q}_{t-1}}[\exp(\leverage_{t} c_{t})].
    \end{equation}
\end{lemma}
\begin{myproof}
    The drop is given by the following:
    \begin{align*}
        \kl(\meas{P}, \meas{Q}_{t-1}) -  \kl(\meas{P}, \meas{Q}_{t})
        &= \int \log\left( \frac{\meas{P}}{\meas{Q}_{t-1}} \right) \dmeas{P} - \int \log\left( \frac{\meas{P}}{\meas{Q}_{t}} \right) \dmeas{P} \\
        &= \int \log\left( \frac{\meas{P}}{\meas{Q}_{t-1}} \right) \dmeas{P} - \int \log\left( \frac{\meas{P}}{\exp\left(\leverage_{t} c_{t} - \log Z_{t} \right) \cdot \meas{Q}_{t-1}} \right) \dmeas{P} \\
        &= \int_{\mathcal{X} \times \mathcal{Y} \times \mathcal{S}} \log\left( \frac{\exp\left(\leverage_{t} c_{t}(x, y, s) - \log Z_{t} \right) \cdot \meas{Q}_{t-1}(x, y, s)}{\meas{Q}_{t-1}(x, y, s)} \right) \dmeas{P}(x, y, s) \\
        &= \int_{\mathcal{X} \times \mathcal{Y} \times \mathcal{S}} \left(\leverage_{t} c_{t}(x, y, s) - \log Z_{t} \right) \dmeas{P}(x, y, s) \\
        &= \leverage_{t} \cdot \expect_{\meas{P}} [c_{t}] - \log Z_{t} \\
        &= \leverage_{t} \cdot \expect_{\meas{P}} [c_{t}] - \log \expect_{\meas{Q}_{t-1}}[\exp(\leverage_{t} c_{t})],
    \end{align*}
    where the last line follows from the definition of the normalizing term (noting it can be expressed as an expectation).
\end{myproof}

\begin{lemma}\label{lem:expect_exp_bound}
    Given the WLA and WLs with bounding constant \( C > 0 \), given coefficients:
    \begin{align*}
        A = \frac{1}{4C^2} \left[ \exp(C) - \exp(-C) \right];
        \quad\quad
        B = \frac{1}{2C} \left[ \exp(C) - \exp(-C) \right];
        \quad\quad \\
        K = \frac{\exp(C) + \exp(-C)}{2} - AC^2,
    \end{align*}
    then for WL \( c_{t}(.) \),
    \begin{equation}
        \expect_{\meas{Q}_{t-1}}[\exp(c_{t})] < \exp(- \Gamma(\var_{\meas{Q}_{t-1}}[c_{t}], \gamma_{\meas{Q}}^{t}))
    \end{equation}
    where
    \begin{equation}
        %\Gamma(\var_{\meas{Q}_{t-1}}[c_{t}], \gamma_{\meas{Q}}^{t}) = \log \left( (A \cdot \var_{\meas{Q}_{t-1}}[c_{t}] - C \cdot (B - C \cdot A) \cdot \gamma_{Q}^{t} + K)^{-1} \right).
        \Gamma(\var_{\meas{Q}_{t-1}}[c_{t}], \gamma_{\meas{Q}}^{t}) = \log \left( (A \cdot \var_{\meas{Q}_{t-1}}[c_{t}] - C^{2} A \cdot \gamma_{Q}^{t} + K)^{-1} \right).
    \end{equation}
\end{lemma}
\begin{myproof}
    First notice that \( f(x) = Ax^{2} + Bx + K \) is an upper bound of \( \exp(x) \) on \( [-C, C] \), \ie, \( \exp(x) \leq f(x) \). The coefficients are derived from fitting \( f(x) \) to satisfy points \( (-C, \exp(-C)) \) and \( (C, \exp(C)) \); and ensuring monotonicity on the interval. \cref{fig:upperbound} depicts the bound (and a comparison to a prior bound used in \citet{hbcnLD}).

    From this, we have:
    \begin{align*}
        \expect_{\meas{Q}_{t-1}}[\exp(c_{t})] 
        &\leq \expect_{\meas{Q}_{t-1}}[A c_{t}^{2} + B c_{t} + K] \\
        &= A \cdot \expect_{\meas{Q}_{t-1}}[c_{t}^{2}] + B \cdot \expect_{\meas{Q}_{t-1}}[c_{t}] + K \\
        &= A \cdot (\var_{\meas{Q}_{t-1}}[c_{t}] + \expect_{\meas{Q}_{t-1}}[c_{t}]^{2}) + B \cdot \expect_{\meas{Q}_{t-1}}[c_{t}] + K \\
        &= A \cdot \var_{\meas{Q}_{t-1}}[c_{t}] + \expect_{\meas{Q}_{t-1}}[c_{t}] \cdot (A \cdot \expect_{\meas{Q}_{t-1}}[c_{t}] + B) + K.
    \end{align*}
    Note from the WLA we have \( \expect_{\meas{Q}_{t-1}}[-c_{t}] \geq C \cdot \gamma_{\meas{Q}}^{t} \iff \expect_{\meas{Q}_{t-1}}[c_{t}] \leq -C \cdot \gamma_{\meas{Q}}^{t} < 0 \). Furthermore, notice that \( Az + B > 0 \) and increasing for \( z \in [-C, C] \):
    \begin{align*}
        Az + B = \frac{\exp(C) - \exp(-C)}{2C}\left[ \frac{z}{2C} + 1 \right] > 0 \quad\quad (\textrm{for } z \in [-C, C]).
    \end{align*}
    Thus,
    \begin{align*}
        \expect_{\meas{Q}_{t-1}}[\exp(c_{t})] 
        &\leq A \cdot \var_{\meas{Q}_{t-1}}[c_{t}] + \expect_{\meas{Q}_{t-1}}[c_{t}] \cdot (A \cdot \expect_{\meas{Q}_{t-1}}[c_{t}] + B) + K \\
        &\leq A \cdot \var_{\meas{Q}_{t-1}}[c_{t}] - \gamma_{\meas{Q}} \cdot C \cdot (A \cdot \expect_{\meas{Q}_{t-1}}[c_{t}] + B) + K \\
        &\leq A \cdot \var_{\meas{Q}_{t-1}}[c_{t}] - \gamma_{\meas{Q}} \cdot C \cdot (B - C \cdot A) + K.
    \end{align*}
    Which gives the Lemma, as per:
    \begin{align*}
        \expect_{\meas{Q}_{t-1}}[\exp(c_{t})] 
        &\leq A \cdot \var_{\meas{Q}_{t-1}}[c_{t}] - \gamma_{\meas{Q}} \cdot C \cdot (B - C \cdot A) + K \\
        &= \exp \left( \log \left( A \cdot \var_{\meas{Q}_{t-1}}[c_{t}] - \gamma_{\meas{Q}} \cdot C \cdot (B - C \cdot A) + K \right) \right) \\
        &= \exp \left( (-1) \cdot -\log \left( A \cdot \var_{\meas{Q}_{t-1}}[c_{t}] - \gamma_{\meas{Q}} \cdot C \cdot (B - C \cdot A) + K \right) \right) \\
        &= \exp \left( -\log \left( A \cdot \var_{\meas{Q}_{t-1}}[c_{t}] - \gamma_{\meas{Q}} \cdot C \cdot (B - C \cdot A) + K \right)^{-1} \right).
    \end{align*}
    Finally, we conclude the proof by noting that \( 2CA = B \).
\end{myproof}

\begin{remark}
    We note that the quadratic approximation \( f(x) = Ax^2 + Bx + K \) gets worse as \( C \) increases. Intuitively, as the interval approximation increases, the quadratic finds it more difficult to ``catch-up'' to the exponential growth of \( x \mapsto \exp(x) \). See \cref{fig:upperbound} for a pictorial view of this.
\end{remark}

\begin{figure}[ht]
    \centering
    \includegraphics[width=0.4\textwidth]{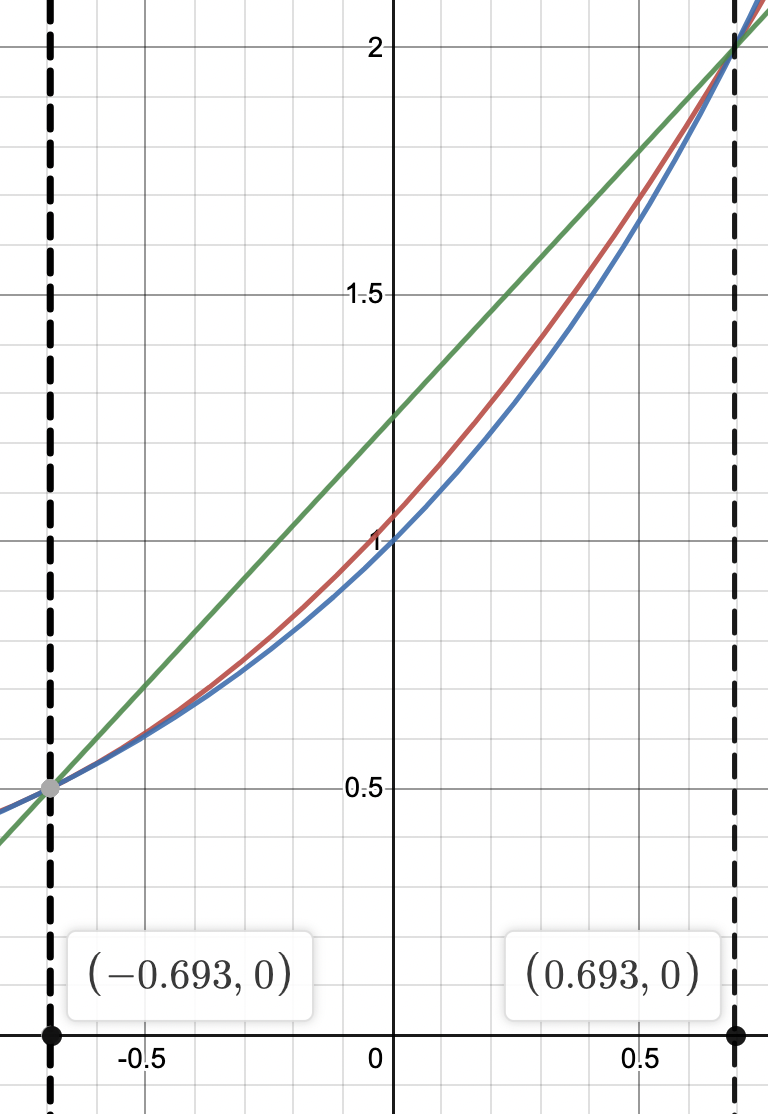}
    \caption{Upper bounds for \( \exp(x) \). The curves are: original \( \exp(x) \) function (blue); our quadratic function (red); and linear function used in \citet{hbcnLD} (green).}
    \label{fig:upperbound}
\end{figure}

Now we can prove \cref{thm:kl_drop}.

\begin{myproof}
    Given that \( \tau / \sr_{0} > \exp(-4C) \), we have that for each leveraging scheme in \cref{thm:exact_fairness,thm:relative_fairness} the leverage is bounded by \( \leverage_{t} \leq 1 \) for all \( t \).

    Thus with \cref{lem:kl_drop,lem:expect_exp_bound}, we bound the KL:
    \begin{align*}
        \kl(\meas{P}, \meas{Q}_{t-1}) - \kl(\meas{P}, \meas{Q}_{t})
        &= \leverage_{t} \cdot \expect_{\meas{P}}[c_{t}] - \log \expect_{\meas{Q}_{t-1}}[\exp(\leverage_{t} c_{t})] \tag{\cref{lem:kl_drop}} \\
        &\geq \leverage_{t} \cdot \expect_{\meas{P}}[c_{t}] - \leverage_{t} \cdot \log \expect_{\meas{Q}_{t-1}}[\exp(c_{t})] \tag{a} \\
        &= \leverage_{t} \cdot (\expect_{\meas{P}}[c_{t}] - \log \expect_{\meas{Q}_{t-1}}[\exp(c_{t})]) \\
        &\geq \leverage_{t} \cdot (\expect_{\meas{P}}[c_{t}] - \log \exp(- \Gamma(\var_{\meas{Q}_{t-1}}[c_{t}], \gamma_{\meas{Q}}^{t})) \tag{\cref{lem:expect_exp_bound}} \\
        &= \leverage_{t} \cdot (\expect_{\meas{P}}[c_{t}] + \Gamma(\var_{\meas{Q}_{t-1}}[c_{t}], \gamma_{\meas{Q}}^{t}) \\
        &\geq \leverage_{t} \cdot (C \cdot \gamma_{\meas{Q}}^{t} + \Gamma(\var_{\meas{Q}_{t-1}}[c_{t}], \gamma_{\meas{Q}}^{t}) \tag{WLA cond. on \( \meas{P} \)},
    \end{align*}
    where (a) is given by Jensen's inequality and noting that \( x \mapsto x^{1 / \leverage_{t}} \) is convex as \( \leverage_{t} \leq 1 \).
\end{myproof}
\section{Proof of \texorpdfstring{\cref{thm:statistical_difference}}{3.6}}
\label{sec:pf_statistical_difference}

To prove the theorem, we breakdown the proof into the corresponding upper and lower bounds. In particular, we first prove general Lemmas first and then specialize for the specific leveraging coefficients \( \leverage_{t} \) to give the Theorem.

\subsection{Upper Bound}

\begin{lemma}\label{lem:staticial_upper}
    Suppose \( C > 0 \), then
    \begin{equation}
        \kl(\meas{P}, \meas{Q}_{0}) -  \kl(\meas{P}, \meas{Q}_{T}) \leq 2C \cdot \sum_{k=1}^{T} \leverage_{k}.
    \end{equation} 
\end{lemma}
\begin{myproof}
    We first note that by unrolling the definition of \( \meas{Q}_{T} \) (\cref{eq:joint_fbde}), the density can be stated directly in terms of the initial distribution:
    \begin{equation}\label{eq:fbde_as_init}
        \meas{Q}_{T} = \meas{Q}_{0} \exp(\langle \ve{\leverage}, \ve{c}\rangle - \varphi(\ve{\leverage})) ,
    \end{equation}
    where \( \ve{\leverage} = (\leverage_{1}, \ldots, \leverage_{T}) \), \( \ve{c} = (c_{1}, \ldots, c_{T}) \), and \( \varphi(\ve{\leverage}) = \sum_{i=1}^{T} \log Z_{i} \).
    
    Thus similarly to \cref{lem:kl_drop}, calculate the total drop:
    \begin{align*}
        \kl(\meas{P}, \meas{Q}_{0}) -  \kl(\meas{P}, \meas{Q}_{T})
        &= \int \log\left( \frac{\meas{P}}{\meas{Q}_{0}} \right) \dmeas{P} - \int \log\left( \frac{\meas{P}}{\meas{Q}_{T}} \right) \dmeas{P} \\
        &= \int \log\left( \frac{\meas{P}}{\meas{Q}_{0}} \right) \dmeas{P} - \int \log\left( \frac{\meas{P}}{\meas{Q}_{0} \exp(\langle \ve{\leverage}, \ve{c}\rangle - \varphi(\ve{\leverage}))} \right) \dmeas{P} \\
        &= \int \log\left( \frac{\meas{Q}_{0} \exp\left(\langle \ve{\leverage}, \ve{c}\rangle - \varphi(\ve{\leverage}) \right)}{\meas{Q}_{0}} \right) \dmeas{P} \\
        &= \int \left(\langle \ve{\leverage}, \ve{c}\rangle - \varphi(\ve{\leverage}) \right) \dmeas{P} \\
        &= \langle \ve{\leverage}, \int \ve{c} \, \dmeas{P} \rangle - \varphi(\ve{\leverage}).
    \end{align*}
    We now upper bound each of these terms. Firstly, we note that as each weak learner \( c_{t} \) is upper bounded by \( C \), we can easily bound the first term:
    \begin{align*}
        \langle \ve{\leverage}, \int \ve{c} \, \dmeas{P} \rangle
        \leq 
        \langle \ve{\leverage}, \int C \cdot \ve{1} \, \dmeas{P} \rangle
        = C \cdot \sum_{k=1}^{T} \leverage_{k}.
    \end{align*}
    
    Similarly, we can bound the second term by noting \cref{remark:log_partition_bound}. Thus using an identical bound to \cref{eq:log_partition_bound} we get:
    \begin{align*}
        \varphi(\ve{\leverage})
        = \sum_{i=1}^{T} \log Z_{i}   
        \geq \sum_{k=1}^{T} (-C \leverage_{k})
        = - C \cdot \sum_{k=1}^{T} \leverage_{k}
    \end{align*}
    
    Thus it follows that 
    \begin{align*}
        \kl(\meas{P}, \meas{Q}_{0}) -  \kl(\meas{P}, \meas{Q}_{T}) \leq 2C \cdot \sum_{k=1}^{T} \leverage_{k}.
    \end{align*}
    As required.
\end{myproof}

\subsection{Lower Bound}

\begin{lemma}\label{lem:staticial_lower}
    Suppose \( \alpha(\gamma) \defeq \min_{t} \Gamma( \var_{\meas{Q}_{t-1}}[c_t], \gamma) / (\gamma C) \), \( C > 0 \), \( \leverage \leq 1\) for all \( t \), and \( T > 0 \), then:
    \begin{align*}
         \kl(\meas{P}, \meas{Q}_{0}) - \kl(\meas{P}, \meas{Q}_{T})
         &\geq \left( \gamma_{\meas{P}} + \gamma_{\meas{Q}} \cdot \alpha(\gamma_{\meas{Q}}) \right) \cdot C \cdot \sum_{k=1}^{T} \leverage_{k}.
    \end{align*}
\end{lemma}
\begin{myproof}
    To establish the lower bound, we repeatedly apply \cref{thm:kl_drop}:
    \begin{align*}
         \kl(\meas{P}, \meas{Q}_{T})
         \leq \kl(\meas{P}, \meas{Q}_{T-1}) -  \leverage_{T} \cdot \Lambda_{T}
         \leq \quad \ldots \quad
         \leq \kl(\meas{P}, \meas{Q}_{0}) -  \sum_{k=1}^{T} \leverage_{k} \cdot \Lambda_{k}.
    \end{align*}
    
    Let \( \alpha(\gamma) \defeq \min_{t} \Gamma( \var_{\meas{Q}_{t-1}}[c_t], \gamma) / (\gamma C) \). Note that the variance is bounded by \( 0 \leq \var_{\meas{Q}_{t-1}}[c_{t}] \leq \expect_{\meas{Q}}[c_{t}^{2}] \leq C^{2} \) and that \( \Gamma(\cdot, \gamma) \) is a decreasing function for all \( \gamma \in [0, 1] \).
    Thus we have
    \begin{align*}
         \kl(\meas{P}, \meas{Q}_{0}) - \kl(\meas{P}, \meas{Q}_{T})
         &\geq \sum_{k=1}^{T} \leverage_{k} \cdot \Lambda_{k} \\
         &= \sum_{k=1}^{T} \leverage_{k} \cdot \left( \gamma_{\meas{P}}^{t} \cdot C + \Gamma(\var_{\meas{Q}_{t-1}}[c_{t}], \gamma_{\meas{Q}}^t) \right) \\
         &= \sum_{k=1}^{T} \leverage_{k} \cdot \left( \gamma_{\meas{P}} \cdot C + \Gamma(\var_{\meas{Q}_{t-1}}[c_{t}], \gamma_{\meas{Q}}) \right) \tag{Fixed const. assum.} \\
         &\geq \sum_{k=1}^{T} \leverage_{k} \cdot \left( \gamma_{\meas{P}} \cdot C + \gamma_{\meas{Q}} \cdot C \cdot \alpha(\gamma_{\meas{Q}}) \right) \tag{Definition of \( \alpha \)} \\
         &= \left( \gamma_{\meas{P}} + \gamma_{\meas{Q}} \cdot \alpha(\gamma_{\meas{Q}}) \right) \cdot C \cdot \sum_{k=1}^{T} \leverage_{k}.
    \end{align*}
    As required.
\end{myproof}

\subsection{For Fairness Leverage Coefficients}

From \cref{lem:staticial_lower,lem:staticial_upper} we have that:
\begin{theorem}\label{thm:delta_kl_gen_bounds}
    Suppose \( \alpha(\gamma) \defeq \min_{t} \Gamma( \var_{\meas{Q}_{t-1}}[c_t], \gamma) / (\gamma C) \), \( C > 0 \), \( \leverage \leq 1\) for all \( t \), and \( T > 0 \), then:
\begin{align*}
    \left( \gamma_{\meas{P}} + \gamma_{\meas{Q}} \cdot \alpha(\gamma_{\meas{Q}}) \right) \cdot C \cdot \sum_{k=1}^{T} \leverage_{k}
    \leq 
    \kl(\meas{P}, \meas{Q}_{0}) -  \kl(\meas{P}, \meas{Q}_{T})
    \leq
    2C \cdot \sum_{k=1}^{T} \leverage_{k}.
\end{align*}
\end{theorem}

To upper and lower bound the statistical difference \( \Delta(\meas{Q}_{T}) \), we upper and lower bound the sum of leveraging coefficients \( \leverage_{t} \) for each of the fairness guarantees.

\subsubsection{Exact Fairness Leverage}

\begin{myproof}
Taking \( \leverage_{t} = - ({C2^{t+1}})^{-1} \log (\tau / \sr_{0}) \) from geometric series, we have
\begin{align*}
    \sum_{k=1}^{T} \leverage_{k}
    = - \frac{\log (\tau / \sr_{0})}{2C} \sum_{k=1}^{T} \frac{1}{2^{k}}
    = - \frac{\log (\tau / \sr_{0})}{2C} \left( 1 - \frac{1}{2^{T}} \right)
\end{align*}

Thus we have
\begin{align*}
    &- \log( \tau / \sr_{0}) \cdot
    \frac{\gamma_{\meas{P}} + \gamma_{\meas{Q}} \cdot \alpha(\gamma_{\meas{Q}}) }{2} \cdot \left( 1 - \frac{1}{2^T} \right) \\
    &\quad\leq
    \kl(\meas{P}, \meas{Q}_{0}) -  \kl(\meas{P}, \meas{Q}_{T}) \leq - \log( \tau / \sr_{0}) \left( 1 - \frac{1}{2^T} \right)
\end{align*}

As required.
\end{myproof}

\subsubsection{Relative Fairness Leverage}

\begin{myproof}
From harmonic series, we have that
\begin{align*}
    \frac{1}{T} + \log T < \sum_{k=1}^{T} \frac{1}{k} < 1 + \log T.
\end{align*}
Thus taking \( \leverage_{t} = - ({4Ct})^{-1} \log (\tau / \sr_{0}) \) we have that
\begin{align*}
    -\log(\tau / \sr_{0}) \cdot \frac{1}{4C} \cdot \left(\frac{1}{T} + \log T\right) < \sum_{k=1}^{T} \leverage_{t} < -\log(\tau / \sr_{0}) \cdot \frac{1}{4C} \cdot \left(1 + \log T\right).
\end{align*}
This gives the following bounds:
\begin{align*}
    &-\log(\tau / \sr_{0}) \cdot \frac{\gamma_{\meas{P}} + \gamma_{\meas{Q}} \cdot \alpha(\gamma_{\meas{Q}})}{4} \cdot \left(\frac{1}{T} + \log T \right)
    \leq 
    \kl(\meas{P}, \meas{Q}_{0}) -  \kl(\meas{P}, \meas{Q}_{T})  \\
    &\quad
    \leq - \log(\tau / \sr_{0}) \cdot \frac{1}{2} \cdot \left( 1 + \log T \right).
\end{align*}
As required.
\end{myproof}
\section{Proof of \texorpdfstring{\cref{cor:suff_boost_steps}}{3.7}}
\label{sec:pf_suff_boost_steps}

\begin{myproof}
    The statement follows directly from utilizing \cref{thm:delta_kl_gen_bounds} and utilizing the series lower bounds used in the proof of \cref{thm:statistical_difference}.

    Indeed, we have
    \begin{align*}
    & 
    \kl(\meas{P}, \meas{Q}_{0}) -  \kl(\meas{P}, \meas{Q}_{T}) \geq
    \left( \gamma_{\meas{P}} + \gamma_{\meas{Q}} \cdot \alpha(\gamma_{\meas{Q}}) \right) \cdot C \cdot \sum_{k=1}^{T} \leverage_{k}
    \\
    \iff &
    \kl(\meas{P}, \meas{Q}_{T}) \leq
    \kl(\meas{P}, \meas{Q}_{0})
    -
    \left( \gamma_{\meas{P}} + \gamma_{\meas{Q}} \cdot \alpha(\gamma_{\meas{Q}}) \right) \cdot C \cdot \sum_{k=1}^{T} \leverage_{k}.
    \end{align*}
    Thus, a sufficient conditions for \( \kl(\meas{P}, \meas{Q}_{T}) < \varepsilon \) is that the RHS is \( < \varepsilon \).

    Thus we are looking to find conditions where the following holds:
    \begin{equation}
        \label{eq:suff_kl_cond_gen}
        \frac{\kl(\meas{P}, \meas{Q}_{0}) - \varepsilon}{\left( \gamma_{\meas{P}} + \gamma_{\meas{Q}} \cdot \alpha(\gamma_{\meas{Q}}) \right) \cdot C} < \sum_{k=1}^{T} \leverage_{k}.
    \end{equation}

    We will take \( \lambda = -\log (\tau / \sr_{0}) \).

    \subsection{Exact Leverage}
    We can directly rewrite the condition as
    \begin{align*}
        \frac{\kl(\meas{P}, \meas{Q}_{0}) - \varepsilon}{\left( \gamma_{\meas{P}} + \gamma_{\meas{Q}} \cdot \alpha(\gamma_{\meas{Q}}) \right) \cdot C} &< \sum_{k=1}^{T} \leverage_{k} =  \frac{\lambda}{2C} \left( 1 - \frac{1}{2^{T}} \right) \\
        2^{T} &> \left( 1 - 2\cdot \frac{\kl(\meas{P}, \meas{Q}_{0}) - \varepsilon}{\lambda \cdot \left( \gamma_{\meas{P}} + \gamma_{\meas{Q}} \cdot \alpha(\gamma_{\meas{Q}}) \right)} \right) \\
        {T} &> \frac{1}{\log 2}\log \cdot \left( 1 - 2\cdot \frac{\kl(\meas{P}, \meas{Q}_{0}) - \varepsilon}{\lambda \cdot \left( \gamma_{\meas{P}} + \gamma_{\meas{Q}} \cdot \alpha(\gamma_{\meas{Q}}) \right)} \right).
    \end{align*}
    As required.

    \subsection{Relative Leverage}
    For this leverage, we need a weaker bound. We use the fact that 
    \begin{align*}
        \frac{\kl(\meas{P}, \meas{Q}_{0}) - \varepsilon}{\left( \gamma_{\meas{P}} + \gamma_{\meas{Q}} \cdot \alpha(\gamma_{\meas{Q}}) \right) \cdot C} < \sum_{k=1}^{T}
        \impliedby
        \frac{\kl(\meas{P}, \meas{Q}_{0}) - \varepsilon}{\left( \gamma_{\meas{P}} + \gamma_{\meas{Q}} \cdot \alpha(\gamma_{\meas{Q}}) \right) \cdot C} < \frac{\lambda}{4C} \cdot \log T.
    \end{align*}

    Thus, we immediately have that
    \begin{equation*}
        T > \exp \left( 4 \cdot \frac{\kl(\meas{P}, \meas{Q}_{0}) - \varepsilon}{\lambda \cdot \left( \gamma_{\meas{P}} + \gamma_{\meas{Q}} \cdot \alpha(\gamma_{\meas{Q}}) \right)} \right)
    \end{equation*}
    implies the statement holds.
\end{myproof}

\begin{remark}
    Tighter bounds can be made for relative leverage \wrt the ``omega function''. More generally, by leaving the leverage series as is (without simplification) gives sufficient condition for a bounded KL. However, the condition may not be \wrt \( T \) being large.
\end{remark}
\section{Proof of \texorpdfstring{\cref{lem:relative_fairness_guarantee}}{4.2}}
\label{sec:pf_relative_fairness_guarantee}

\begin{myproof}
We prove the following by contradiction. 
Suppose that there exists distribution \( \meas{Q} \in \mathcal{M}_{\varepsilon, \meas{Q}_{0}} \) with \( \SR(\meas{Q}) < \sr_{0} \cdot \exp(- \varepsilon) \).
Thus for there exists a \( s, s^{\prime} \in \mathcal{S} \) such that
\begin{equation*}
    \SR(\meas{Q}, s^{\prime}, s; y) <  \sr_{0} \cdot \exp(- \varepsilon)
    \implies 
    \SR(\meas{Q}, s, s^{\prime}; y) >  \frac{1}{\sr_{0}} \cdot \exp(\varepsilon).
\end{equation*}
But given that \( \SR(\meas{Q}_{0}) = \sr_{0} \), by definition, we have that 
\begin{equation*}
    \forall a, b \in \mathcal{S}: \; \SR(\meas{Q}_{0}, a, b; y) \geq \sr_{0}
\end{equation*}

Thus we have that,
\begin{align*}
    \max\left\{ \frac{\SR(\meas{Q}, s, s^{\prime}; y)}{\SR(\meas{Q}_{0}, s, s^{\prime}; y)}, \frac{\SR(\meas{Q}_{0}, s, s^{\prime}; y)}{\SR(\meas{Q}, s, s^{\prime}; y)}\right\}
    &\geq \frac{\SR(\meas{Q}, s, s^{\prime}; y)}{\SR(\meas{Q}_{0}, s, s^{\prime}; y)} \\
    &\geq \SR(\meas{Q}, s, s^{\prime}; y) \cdot \SR(\meas{Q}_{0}, s^{\prime}, s; y) \\
    &\geq \SR(\meas{Q}, s, s^{\prime}; y) \cdot sr_{0} \\
    &> \frac{1}{\sr_{0}} \cdot \exp(\varepsilon) \cdot \sr_{0} = \exp(\varepsilon).
\end{align*}
However, \( \meas{Q} \in \mathcal{M}_{\varepsilon, \meas{Q}_{0}} \) and thus by \cref{def:relmol} we have a contradiction.

Thus \( \meas{Q} \in \mathcal{M}_{\varepsilon, \meas{Q}_{0}} \implies \SR(\meas{Q}) \geq \sr_{0} \cdot \exp(-\varepsilon) \).

\end{myproof}
\section{Proof of \texorpdfstring{\cref{lem:relative_fair_distribution}}{4.3}}
\label{sec:pf_relative_fair_distribution}

\begin{myproof}
    We first note that the condition of ``\( \forall s, s^{\prime} \in \mathcal{S} \) we have either \( \SR({\meas{Q}}, s, s^{\prime}; y) > 1 \) or \( \SR({\meas{Q}}, s, s^{\prime}; y) \leq \SR({\meas{Q}}_{0}, s, s^{\prime}; y) \leq 1 \)'' is equivalent to ``\( \forall s, s^{\prime} \in \mathcal{S} \) we have \( \SR({\meas{Q}}, s, s^{\prime}; y) \leq 1 \) implies \( \SR({\meas{Q}}, s, s^{\prime}; y) \leq \SR({\meas{Q}}_{0}, s, s^{\prime}; y) \leq 1 \)''.

    We suppose that \( {\meas{Q}} \) is an arbitrary distribution that satisfies this condition.
    Let \( s, s^{\prime} \in \mathcal{S} \) be arbitrary such that \( \SR(\meas{Q}, s, s^{\prime}; y) \leq 1 \), and thus \( \SR(\meas{Q}, s, s^{\prime}; y) \leq \SR(\meas{Q}_{0}, s, s^{\prime}; y) \leq 1 \).
    
    We have that
    \begin{equation*}
        \frac{\SR(\meas{Q}, s, s^{\prime}; y)}{\SR(\meas{Q}_{0}, s, s^{\prime}; y)} \geq \SR(\meas{Q}, s, s^{\prime}; y)
    \end{equation*}
    and
    \begin{equation*}
        \frac{\SR(\meas{Q}, s, s^{\prime}; y)}{\SR(\meas{Q}_{0}, s, s^{\prime}; y)} \leq 1 \leq \frac{\SR(\meas{Q}_{0}, s, s^{\prime}; y)}{\SR(\meas{Q}, s, s^{\prime}; y)}.
    \end{equation*}
    
    As \( \SR(\meas{Q}) \geq \exp(-\varepsilon) \), we have that
    \begin{align*}
        \SR(\meas{Q}, s, s^{\prime}; y) \geq \exp(-\varepsilon)
        &\implies \frac{\SR(\meas{Q}, s, s^{\prime}; y)}{\SR(\meas{Q}_{0}, s, s^{\prime}; y)} \geq \exp(-\varepsilon) \\
        &\implies \min\left\{ \frac{\SR(\meas{Q}, s, s^{\prime}; y)}{\SR(\meas{Q}_{0}, s, s^{\prime}; y)}, \frac{\SR(\meas{Q}_{0}, s, s^{\prime}; y)}{\SR(\meas{Q}, s, s^{\prime}; y)} \right\} \geq \exp(-\varepsilon).
    \end{align*}
    
    Furthermore, if \( \SR(\meas{Q}, s, s^{\prime}; y) > 1 \), then necessarily \( \SR(\meas{Q}, s^{\prime}, s; y) \leq 1 \). However, we also have
    \begin{equation*}
        \min\left\{ \frac{\SR(\meas{Q}, s, s^{\prime}; y)}{\SR(\meas{Q}_{0}, s, s^{\prime}; y)}, \frac{\SR(\meas{Q}_{0}, s, s^{\prime}; y)}{\SR(\meas{Q}, s, s^{\prime}; y)} \right\}
        =
        \min\left\{ \frac{\SR(\meas{Q}_{0}, s^{\prime}, s; y)}{\SR(\meas{Q}, s^{\prime}, s; y)}, \frac{\SR(\meas{Q}, s^{\prime}, s; y)}{\SR(\meas{Q}_{0}, s^{\prime}, s; y)} \right\}.
    \end{equation*}
    
    Thus for all \( s, s^{\prime} \in \mathcal{S} \)
    \begin{align*}
        &\min\left\{ \frac{\SR(\meas{Q}, s, s^{\prime}; y)}{\SR(\meas{Q}_{0}, s, s^{\prime}; y)}, \frac{\SR(\meas{Q}_{0}, s, s^{\prime}; y)}{\SR(\meas{Q}, s, s^{\prime}; y)} \right\} \geq \exp(-\varepsilon)\\
        &\quad\iff \max\left\{ \frac{\SR(\meas{Q}, s, s^{\prime}; y)}{\SR(\meas{Q}_{0}, s, s^{\prime}; y)}, \frac{\SR(\meas{Q}_{0}, s, s^{\prime}; y)}{\SR(\meas{Q}, s, s^{\prime}; y)} \right\} \leq \exp(\varepsilon).
    \end{align*}
    That is \( \meas{Q} \in \mathcal{M}_{\varepsilon, \meas{Q}_{0}} \).
\end{myproof}

%%%% Technical Material
\section{Analyzing KL Drop Bound}
\label{sec:analyzing_kl_drop_bound}

In the following section, we analysis the bound given in \cref{thm:kl_drop}. In particular, we examine the non-negativity of the bound and compare it against \citet[Theorem 5]{hbcnLD}.

\begin{lemma}\label{lem:gamma_calculations}
    The following holds for \( \Gamma(\cdot, \cdot) \):
    \begin{align*}
        \Gamma(C^{2}, \gamma) &= \log \left( \frac{4}{(2 + \gamma)\exp(-C)) + (2 - \gamma) \exp(C)} \right); \\
        \Gamma(0, \gamma) &= \log \left(\frac{4}{(3 - \gamma)\exp(-C)) + (1 - \gamma) \exp(C)}\right); \\
        \Gamma(v, 1) &= \log \left(\frac{4C^2}{v(\exp(C) - \exp(-C) +4C^2 \exp(-C)}\right); \\
        \Gamma(0, 1) &= C.
    \end{align*}
\end{lemma}
\begin{myproof}
    Follows directly from the definition of \( \Gamma \), as per \cref{lem:expect_exp_bound}.
\end{myproof}

To compare against \citet[Theorem 5]{hbcnLD}, we consider the maximum variance which is required to make \( \Gamma(\cdot, \cdot) > 0\).

\begin{theorem}\label{thm:positive_gamma}
    For the function \( \Gamma(v, \gamma) \), we have that
    \begin{align}
        \Gamma(v, \gamma) > 0 &\iff \gamma > \frac{v}{C^2} + 1 - \frac{4}{\exp(C) + 1}; \label{eq:gamma_gamma_cond} \\
        \Gamma(v, \gamma) > 0 &\iff v < C^2 \cdot \left(\gamma + \frac{4}{\exp(C) + 1} - 1 \right) \label{eq:gamma_var_cond}.
    \end{align}
\end{theorem}
\begin{myproof}
    We consider the conditions for \( \Gamma(v, \gamma) > 0 \). This occurs when:
    \begin{align*}
        &\Gamma(v, \gamma) > 0 \\
        \iff &\log \left( (A \cdot v - C^{2} A \cdot \gamma + K)^{-1} \right) > 0 \\
        \iff &A \cdot v - C^{2} A \cdot \gamma + K < 1 \\
        \iff &\gamma > \frac{Av + K - 1}{C^2A} \\
        \iff &\gamma > \frac{v}{C^2} + \frac{K - 1}{C^2A}.
    \end{align*}
    Through computation, one can note that \( (K-1) / (C^2A) = 1 - 4 / (\exp(C) + 1) \).
    Thus we have the following:%
    \begin{align*}
        \Gamma(v, \gamma) > 0 &\iff \gamma > \frac{v}{C^2} + 1 - \frac{4}{\exp(C) + 1}; \\
        \Gamma(v, \gamma) > 0 &\iff v < C^2 \cdot \left(\gamma + \frac{4}{\exp(C) + 1} - 1 \right).
    \end{align*}
    As required.
\end{myproof}

\begin{corollary}[Comparing to \citet{hbcnLD}]\label{cor:compare_hbcnLD}
    Let \( C = \log 2 \). If \( \var_{\meas{Q}_{t-1}}[c_{t}] < C^2 \cdot \frac{2}{3} \), when \citet[Theorem 5]{hbcnLD}'s KL drop is positive, so is ours.

    Furthermore, lower \( \var_{\meas{Q}_{t-1}}[c_{t}] \) values allows for positive KL drops even when \citet[Theorem 5]{hbcnLD} has a non-positive drop.
\end{corollary}
\begin{myproof}
    To guarantee a positive drop, \citet[Theorem 5]{hbcnLD} requires the WLA to be within a ``high boosting regime''. This is equivalent to having \( \gamma_{\meas{Q}}^t > 1/3 \).
    
    It follows from \cref{thm:positive_gamma} that we get a positive drop when:
    \begin{equation}
        \gamma > \frac{2}{3} + 1 - \frac{4}{3} = \frac{1}{3}. \label{eq:gamma_var_cond_hbcnLD}
    \end{equation}
    As required.

    The last part of the corollary follows from noting that \cref{eq:gamma_var_cond} is linearly increasing function of \( v \). Thus lower \( v \) will relax the constraint in \cref{eq:gamma_var_cond_hbcnLD}.
\end{myproof}
\section{General Fair Mollifiers}
\label{sec:general_fair_mollifiers}

The relative mollifier construction present in \cref{def:relmol} was not the original definition of a mollifier, but only a convenient method for construction.
Originally in \citet{hbcnLD}, elements of private mollifiers satisfy a pairwise privacy constraint. We introduce the \emph{fair mollifier} which is equivalent to this original definition.

\begin{definition}\label{def-fairmol}
    Let \( \mathcal{M} \subset \mathcal{D}(\mathcal{Z}) \) and \( \varepsilon > 0 \). \( \mathcal{M} \) is an {\bf \( \varepsilon \)-fair mollifier} iff\footnote{Absolute continuity also required for all $\{\meas{P}\} \cup \mathcal{M}$ \wrt all  $\mathcal{M}$.} \( \forall \meas{Q}, \meas{Q}^{\prime} \in \mathcal{M}\), %\forall a_{i}, a_{j} \in \mathcal{A} \),
    \begin{eqnarray}\label{eq:fairmol_cond}
        \SR(\meas{Q}, s, s^{\prime}; y) \leq e^{\varepsilon} \cdot \SR(\meas{Q}^{\prime}, s, s^{\prime}; y), \quad \forall s, s^{\prime} \in \mathcal{S}.
    \end{eqnarray} 
\end{definition}

\cref{def-fairmol} is inspired by the local notion of differential privacy \citep{klnrsWC}. Likewise to the main-text, all statistical rate definitions can be interchanged with representation rate.

\begin{lemma} \label{lem:contain_fair_distribution}
    Suppose that \( \mathcal{M} \) is an \( \varepsilon \)-fair mollifier. If there exists a \( \meas{F} \in \mathcal{M} \) with \( \SR(\meas{F}) = \sr_{f} \), then all \( \meas{Q} \in \mathcal{M} \) have \( \SR(\meas{Q}) \geq \sr_{f} \cdot \exp(- \varepsilon) \).
\end{lemma}
\begin{myproof}
    As \( \mathcal{M} \) is a \( \varepsilon \)-fair mollifier, the following constraint holds for all \( \meas{Q} \in \mathcal{M} \) and for all \( s, s^{\prime} \in \mathcal{S} \) with the distribution \( \meas{F} \in \mathcal{M} \), where \( \SR(\meas{F}) = \sr_{f} \):
    \begin{align*}
        \SR(\meas{F}, s, s^{\prime}; y) \leq \exp(\varepsilon) \cdot \SR(\meas{Q}, s, s^{\prime}; y)
        &\implies 
        \sr_{0} \leq \exp(\varepsilon) \cdot \SR(\meas{Q}, s, s^{\prime}; y) \\
        &\implies 
        \sr_{0} \cdot \exp(-\varepsilon) \leq \SR(\meas{Q}, s, s^{\prime}; y).
    \end{align*}
    Thus all \( \meas{Q} \in \mathcal{M} \) has statistical rate \( \SR(\meas{Q}) \geq \sr_{f} \cdot \exp(-\varepsilon) \).
\end{myproof}

It should be noted that \cref{lem:contain_fair_distribution} is not tight, a mollifier might have a strong fairness guarantee.

\begin{lemma}
    \(\mathcal{M}_{\varepsilon, \meas{Q}_{0}} \) is a \( 2 \varepsilon \)-mollifier.
\end{lemma}
\begin{myproof}
    Follows directly from \cref{eq:relative_mollifier} applies to two distributions.
\end{myproof}

% This isn't correct
%\begin{lemma}
%    When \( \sr_{0} = 1 \), then  \(\mathcal{M}_{\varepsilon, \meas{Q}_{0}} \) is a \( \varepsilon \)-mollifier. More precisely it is the complete mollifier.
%\end{lemma}
%\begin{myproof}
%    From \cref{lem:relative_fair_distribution,lem:relative_fairness_guarantee}.
%\end{myproof}
\section{Improvement in the Low / Sparse Data Regime via Prior Mixing}
\label{sec:mixing_prior}

When utilizing the \fbde algorithm, picking a good initial distribution can be pivotal to ensure that our final distribution is not only fair, but also close to the true data distribution. However, when we only have access to a small volume of data; or when a dataset has a large number of features causing sparsity over input combinations, it may be difficult to find good approximations of the input distribution. These scenarios can be further problematic when we utilize an empirical distribution to create our initial distribution as per \cref{eq:init_sensitive,eq:init_conditional}: inputs in the training set may not appear in the test set. In this case, the final distribution will have zero support for inputs which occur in the test set.

To remedy this issue, we introduce the \emph{mixed} initial distribution. The idea is to take a mixture distribution between the initial distribution introduced \cref{eq:init_sensitive,eq:init_conditional} with a \emph{prior distribution}. That is, we an initial distribution as per,
\begin{equation}
    \meas{Q}_{\init} \leftarrow (1 - \alpha) \cdot \meas{Q}_{\init} + \alpha \cdot \meas{Q}_{\prior},
\end{equation}
where \( \alpha \in [0, 1] \) and \( \meas{Q}_{\prior} \in \mathcal{D}(\mathcal{Z}) \) is a `prior' distribution.

By taking \( \alpha \neq 1 \) and a \( \meas{Q}_{\prior} \) with non-zero support over all of \( \mathcal{Z} \), we can ensure that the final boosted distribution also has non-zero support regardless of the initial distribution.

The one caveat to utilizing this mixture is that, of course, we still want the final \( \meas{Q}_{\init} \) to satisfy the initial fairness constraints, \ie, the dependence on \( \sr_{0} \) and \( \rr_{0} \) in \cref{thm:exact_fairness,thm:relative_fairness}.

Nevertheless, as long as the final \( \meas{Q}_{\init} \) satisfies the initial fairness constraints, the mixing of distribution provides a convenient mechanism for adding prior knowledge into \fbde. Indeed, beyond dealing with miss-matching distribution support issues, \( \meas{Q}_{\prior} \) can be used to add any user prior knowledge to \fbde.

We present experiment where mixing priors are not used for \dutch and \german in \cref{sec:dutch_german_without_mixing}.

%%%% Additional Experiments
\section{Additional Dataset Descriptions}
\label{sec:additional_dataset_descriptions}

We present the following additional dataset descriptions for \compas and \adult not included in the main-text. We note that \compas and \adult are standard public benchmark datasets; and \texttt{AIF360} uses a Apache License 2.0.

(a) {\bf \compas~\citep{almkMB}} 
The pre-processed dataset contains data
{where each defendant has the sex, race, age, number of priors, charge degree, and recidivism within two years. These attributes are given as binary features, where counts and continuous values are discretized into bins. The resulting dataset has a domain size of 144 and with 5,278 data points. We consider sensitive attributes of sex (only in appendix) and race.}

(b) {\bf \adult~\citep{dkUM}} 
The pre-processed dataset contains data where for each person the dataset contains the race, sex, age, years of educations, and binary label whether they earn more than \( 50 \)K a year. The counts and continuous values are discretized similarly to \compas. The pre-processed dataset has a domain size of 504 and with 48,842 data points. We consider sex and race (only in appendix) as the sensitive attribute.

For the \textsc{Minneapolis} dataset, we also include additional dataset descriptions. We note that it uses a CC0 1.0 License.

(c) {\bf \textsc{Minneapolis}}
The dataset consists of stop data for the Minneapolis Police Department. We only consider the features consisting of latitude, longitude, race, and `search' features, where the last feature is a binary variable that reports an individual was searched at a stop. We remove any rows with missing values. We remove outliers (those consisting of longitude values less than \( -13 \)). Furthermore, we binarize the race features to `Black' vs all other categories. In total, the resulting dataset has 121,965 data points.

In addition the datasets examined in the main-text, we also include the \dutch and \german standard datasets. The primary results for these extra datasets are presented in \cref{sec:extra_discrete_experiments}. We also include an additional discussion in \cref{sec:mixing_prior}, where we use a ``trick'' to improve the performance of mollifiers in \german. We further note that these are standard public benchmark datasets.

(d) {\bf \textsc{Dutch}~\citep{van20002001}} 
The Dutch Census dataset is from a 2001 Netherlands census, where data represents aggregated groups of people. The prediction task is to determine whether an individual has a prestigious job or not. The dataset consists of 60,420 examples with 12 different attributes (many with large number of categorical features). The pre-processed dataset has a domain size of 5,971,968, thus, the dataset is extremely sparse. We consider sex and age as sensitive attributes.

(e) {\bf \textsc{German}~\citep{kamiran2009classifying}} The German Credit dataset consists of individual bank holders, with the prediction task being to determine whether or not the grant credit to someone. We utilize the pre-processed version provided by \aif. Notably, the dataset is small with only 1,000 samples with a domain size of 216.
\section{Extended Discrete Experiments}
\label{sec:additional_discrete_experiments}

We present additional experiments and figures corresponding to the discrete datasets of \compas and \adult. In particular, we include all combinations of `sex' and `race' sensitive attributes settings for the datasets. An extended \cref{tab:summary_experiments_table} is presented in \cref{tab:full_experiments_table}. Furthermore, in \cref{fig:wl_all_plots} we plot the complete set of weak learner accuracy vs KL over boosting iterations plots that was given by \cref{fig:boost_iters_joint} (left). We further include similar plots for both representation rate and statistical rate in \cref{fig:rr_all_plots,fig:sr_all_plots}, respectively. We also present plots corresponding to \( \SR \) vs \( \SR_{\textrm{c}} \) in \cref{fig:srclf_all_plots}, which suggests that the classifier statistical rate \( \SR_{\textrm{c}} \) tends to follow a similar trend to that of the data statistical rate \( \SR \).

\begin{table*}[t]
    \caption{\fbde~(\(T = 32\)) and baselines evaluated for \compas and \adult datasets. The table reports the mean and s.t.d.}
    \label{tab:full_experiments_table}
    \scriptsize
    \begin{center}
    \begin{sc}
    \begin{tabularx}{\textwidth}{LLPDDDDDDDD} %{LLMXXXXXXXX}
        \toprule
        && & Data & \mollifiera & \mollifierb & \mollifierc & \mollifierd & \celisa & \tabfairgan & \fairkmeans \\
        \midrule

{\multirow{9}{*}{\rotatebox[origin=c]{90}{\compas (\(\mathcal{S}\) = sex)}}} &
{\multirow{3}{*}{\rotatebox[origin=c]{90}{\underline{data}}}} 
& $ \textrm{RR} $                 &  \(.443 \pm .004\) & \(.916 \pm .011\) & \(.944 \pm .007\) & \(.852 \pm .010\) & \(.907 \pm .007\) & \(.987 \pm .007\) & \(.231 \pm .044\) & \(-\) \\
&& $ \textrm{SR} $                &  \(.728 \pm .044\) & \(.993 \pm .003\) & \(.904 \pm .008\) & \(.986 \pm .010\) & \(.898 \pm .008\) & \(.989 \pm .011\) & \(.823 \pm .138\) & \(-\) \\
&& $ \textrm{KL} $                & 
\(-\) & \(.288 \pm .020\) & \(.295 \pm .022\) & \(.268 \pm .020\) & \(.283 \pm .021\) & \(.335 \pm .021\) & \(3.14 \pm .831\) & \(-\) \\
        \cmidrule(l{5pt}){2-11}

&{\multirow{3}{*}{\rotatebox[origin=c]{90}{\underline{pred}}}} 
& $ \textrm{SR}_{\textrm{c}}$     &  \(.726 \pm .025\) & \(.952 \pm .006\) & \(.874 \pm .009\) & \(.938 \pm .019\) & \(.874 \pm .016\) & \(.966 \pm .012\) & \(.844 \pm .113\) & \(-\) \\
&& $ \textrm{EO} $                &  \(.764 \pm .031\) & \(.966 \pm .021\) & \(.905 \pm .027\) & \(.963 \pm .027\) & \(.907 \pm .030\) & \(.978 \pm .015\) & \(.797 \pm .098\) & \(-\) \\
&& $ \textrm{Acc} $               &  \(.660 \pm .004\) & \(.651 \pm .008\) & \(.654 \pm .005\) & \(.657 \pm .009\) & \(.655 \pm .008\) & \(.650 \pm .007\) & \(.595 \pm .036\) & \(-\) \\
        \cmidrule(l{5pt}){2-11}
        
&{\multirow{3}{*}{\rotatebox[origin=c]{90}{\underline{clus}}}} 
& $ \cPR $                        & 
\(.122 \pm .014\) & \(.122 \pm .014\) & \(.122 \pm .014\) & \(.122 \pm .014\) & \(.119 \pm .017\) & \(.119 \pm .017\) & \(.106 \pm .032\) & \(.143 \pm .019\) \\
&& $ \cSR $                       & 
\(.380 \pm .123\) & \(.380 \pm .123\) & \(.380 \pm .123\) & \(.380 \pm .123\) & \(.438 \pm .097\) & \(.385 \pm .122\) & \(.364 \pm .065\) & \(.295 \pm .089\) \\
&& $ \textrm{dist} $              & 
\(.064 \pm .000\) & \(.078 \pm .000\) & \(.078 \pm .000\) & \(.078 \pm .000\) & \(.078 \pm .000\) & \(.078 \pm .000\) & \(.078 \pm .000\) & \(.069 \pm .001\) \\
        \midrule
        
{\multirow{9}{*}{\rotatebox[origin=c]{90}{\compas (\(\mathcal{S}\) = race)}}} &
{\multirow{3}{*}{\rotatebox[origin=c]{90}{\underline{data}}}} 
& $ \textrm{RR} $                 & \(.662 \pm .007\) & \(.966 \pm .008\) & \(.977 \pm .008\) & \(.944 \pm .010\) & \(.964 \pm .009\) & \(.992 \pm .004\) & \(.632 \pm .052\) &  \( - \) \\
&& $ \textrm{SR} $                & \(.747 \pm .013\) & \(.988 \pm .006\) & \(.899 \pm .011\) & \(.978 \pm .011\) & \(.896 \pm .011\) & \(.992 \pm .006\) & \(.727 \pm .100\) &  \( - \) \\
&& $ \textrm{KL} $                & \( - \) & \(.135 \pm .020\) & \(.129 \pm .020\) & \(.132 \pm .020\) & \(.127 \pm .020\) & \(.164 \pm .018\) & \(2.71 \pm .735\) &  \( - \) \\
        \cmidrule(l{5pt}){2-11}

&{\multirow{3}{*}{\rotatebox[origin=c]{90}{\underline{pred}}}} 
& $ \textrm{SR}_{\textrm{c}}$     & \(.747 \pm .020\) & \(.959 \pm .025\) & \(.875 \pm .025\) & \(.945 \pm .027\) & \(.872 \pm .044\) & \(.939 \pm .018\) & \(.802 \pm .074\) &  \( - \) \\
&& $ \textrm{EO} $                & \(.781 \pm .034\) & \(.960 \pm .026\) & \(.900 \pm .041\) & \(.950 \pm .028\) & \(.895 \pm .039\) & \(.944 \pm .020\) & \(.815 \pm .079\) &  \( - \) \\
&& $ \textrm{Acc} $               & \(.660 \pm .004\) & \(.641 \pm .014\) & \(.653 \pm .012\) & \(.642 \pm .010\) & \(.656 \pm .012\) & \(.648 \pm .015\) & \(.591 \pm .043\) &  \( - \) \\
        \cmidrule(l{5pt}){2-11}
        
&{\multirow{3}{*}{\rotatebox[origin=c]{90}{\underline{clus}}}} 
& $ \cPR $            & \(.334 \pm .019\) & \(.332 \pm .022\) & \(.334 \pm .019\) & \(.334 \pm .019\) & \(.319 \pm .036\) & \(.333 \pm .020\) & \(.286 \pm .027\) & \(.259 \pm .013\) \\
&& $ \cSR $           & \(.288 \pm .055\) & \(.235 \pm .082\) & \(.284 \pm .055\) & \(.284 \pm .055\) & \(.256 \pm .084\) & \(.288 \pm .059\) & \(.268 \pm .063\) & \(.284 \pm .036\) \\
&& $ \textrm{dist} $  & \(.065 \pm .000\) & \(.078 \pm .000\) & \(.078 \pm .000\) & \(.078 \pm .000\) & \(.078 \pm .000\) & \(.078 \pm .000\) & \(.078 \pm .001\) & \(.069 \pm .001\) \\
        \midrule

{\multirow{9}{*}{\rotatebox[origin=c]{90}{\adult (\(\mathcal{S}\) = sex)}}} &
{\multirow{3}{*}{\rotatebox[origin=c]{90}{\underline{data}}}} 
& $ \textrm{RR} $                 & \(.496 \pm .002\) & \(.958 \pm .003\) & \(.979 \pm .003\) & \(.919 \pm .004\) & \(.957 \pm .003\) & \(.995 \pm .002\) & \(.516 \pm .023\) &  \( - \) \\
&& $ \textrm{SR} $                & \(.360 \pm .005\) & \(.961 \pm .005\) & \(.883 \pm .006\) & \(.944 \pm .006\) & \(.865 \pm .006\) & \(.979 \pm .004\) & \(.862 \pm .101\) &  \( - \) \\
&& $ \textrm{KL} $                & \( - \) & \(.122 \pm .006\) & \(.119 \pm .006\) & \(.113 \pm .006\) & \(.114 \pm .006\) & \(.182 \pm .005\) & \(1.68 \pm .538\) &  \( - \) \\

        \cmidrule(l{5pt}){2-11}
&{\multirow{3}{*}{\rotatebox[origin=c]{90}{\underline{pred}}}} 
& $ \textrm{SR}_{\textrm{c}}$     & \(.360 \pm .003\) & \(.818 \pm .010\) & \(.766 \pm .010\) & \(.793 \pm .011\) & \(.753 \pm .008\) & \(.919 \pm .011\) & \(.823 \pm .118\) &  \( - \) \\
&& $ \textrm{EO} $                & \(.471 \pm .008\) & \(.959 \pm .016\) & \(.908 \pm .018\) & \(.935 \pm .016\) & \(.895 \pm .016\) & \(.981 \pm .010\) & \(.867 \pm .105\) &  \( - \) \\
&& $ \textrm{Acc} $               & \(.803 \pm .003\) & \(.785 \pm .002\) & \(.788 \pm .002\) & \(.787 \pm .002\) & \(.788 \pm .002\) & \(.773 \pm .005\) & \(.781 \pm .006\) &  \( - \) \\
        \cmidrule(l{5pt}){2-11}
        
&{\multirow{3}{*}{\rotatebox[origin=c]{90}{\underline{clus}}}} 
& $ \cPR $           & \(.125 \pm .067\) & \(.093 \pm .019\) & \(.101 \pm .025\) & \(.116 \pm .027\) & \(.111 \pm .033\) & \(.130 \pm .061\) & \(.137 \pm .038\) & \(.117 \pm .039\) \\
&& $ \cSR $           & \(.426 \pm .232\) & \(.252 \pm .082\) & \(.302 \pm .076\) & \(.190 \pm .061\) & \(.254 \pm .086\) & \(.468 \pm .237\) & \(.437 \pm .179\) & \(.296 \pm .151\) \\
&& $ \textrm{dist} $  & \(.034 \pm .000\) & \(.037 \pm .000\) & \(.037 \pm .000\) & \(.037 \pm .000\) & \(.037 \pm .000\) & \(.037 \pm .000\) & \(.037 \pm .000\) & \(.035 \pm .000\) \\

\midrule

{\multirow{9}{*}{\rotatebox[origin=c]{90}{\adult (\(\mathcal{S}\) = race)}}} &
{\multirow{3}{*}{\rotatebox[origin=c]{90}{\underline{data}}}} 
& $ \textrm{RR} $                 &  \(.170 \pm .001\) & \(.914 \pm .002\) & \(.955 \pm .003\) & \(.837 \pm .002\) & \(.910 \pm .002\) & \(.998 \pm .003\) & \(.089 \pm .019\) & \(-\) \\
&& $ \textrm{SR} $                &  \(.601 \pm .011\) & \(.985 \pm .006\) & \(.892 \pm .006\) & \(.974 \pm .007\) & \(.888 \pm .006\) & \(.986 \pm .006\) & \(.284 \pm .032\) & \(-\) \\
&& $ \textrm{KL} $                &  \(-\) & \(.292 \pm .001\) & \(.305 \pm .001\) & \(.264 \pm .001\) & \(.289 \pm .001\) & \(.382 \pm .005\) & \(1.98 \pm .173\) & \(-\) \\
        \cmidrule(l{5pt}){2-11}

&{\multirow{3}{*}{\rotatebox[origin=c]{90}{\underline{pred}}}} 
& $ \textrm{SR}_{\textrm{c}}$     &  \(.600 \pm .031\) & \(.867 \pm .035\) & \(.808 \pm .031\) & \(.860 \pm .034\) & \(.803 \pm .035\) & \(.935 \pm .019\) & \(.480 \pm .114\) & \(-\) \\
&& $ \textrm{EO} $                &  \(.787 \pm .053\) & \(.933 \pm .019\) & \(.960 \pm .037\) & \(.935 \pm .017\) & \(.957 \pm .034\) & \(.929 \pm .038\) & \(.588 \pm .103\) & \(-\) \\
&& $ \textrm{Acc} $               &  \(.803 \pm .003\) & \(.800 \pm .002\) & \(.801 \pm .003\) & \(.800 \pm .002\) & \(.801 \pm .003\) & \(.794 \pm .004\) & \(.745 \pm .019\) & \(-\) \\
        \cmidrule(l{5pt}){2-11}
        
&{\multirow{3}{*}{\rotatebox[origin=c]{90}{\underline{clus}}}} 
& $ \cPR $                        &  \(.042 \pm .019\) & \(.051 \pm .038\) & \(.061 \pm .032\) & \(.051 \pm .029\) & \(.043 \pm .012\) & \(.047 \pm .021\) & \(.071 \pm .039\) & \(.052 \pm .044\) \\
&& $ \cSR $                       &  \(.313 \pm .096\) & \(.422 \pm .111\) & \(.366 \pm .212\) & \(.303 \pm .102\) & \(.327 \pm .115\) & \(.290 \pm .096\) & \(.359 \pm .225\) & \(.400 \pm .218\) \\
&& $ \textrm{dist} $              &  \(.034 \pm .000\) & \(.037 \pm .000\) & \(.037 \pm .000\) & \(.037 \pm .000\) & \(.037 \pm .000\) & \(.037 \pm .000\) & \(.037 \pm .000\) & \(.035 \pm .000\) \\

        \bottomrule
    \end{tabularx}
    \end{sc}
    \end{center}
\end{table*}

\subsection{Runtime Comparison}

In addition, we report the runtime of each approach in \cref{tab:runtime}. We notice that in general, \mollifiera is the quickest approach (by a significant margin) taking only a few seconds on both datasets. By strictly looking at the table, \tabfairgan is the next quickest; and lastly \fbde and \fairkmeans (with \fairkmeans scaling significantly worse on the larger dataset). With this strict analysis, one may come to a conclusion that \fbde has a large computational trade-off when compared to \mollifiera (and even \tabfairgan).
However, if one examines \cref{fig:wl_all_plots} they will notice that the change in the distribution is quite small after \( T = 8 \) iterations. In particular, the KL will still be superior than \( \mollifiera \) if we early stopped at \( T = 8 \). Furthermore, we would have a fairer distribution as well (we would have moved from \( \meas{Q}_{0} \) less). As the boosting algorithms time complexity is linear \wrt \( T \) (\ie, the number of WLs learned) we could reduce the boosting iterations / runtime by 4x; whilst still maintaining superior KL and also having strong fairness. Indeed, the fairness and KL can be verified by examining \cref{fig:rr_all_plots,fig:sr_all_plots,fig:srclf_all_plots}.
Although, this would still result in a runtime many times larger than \mollifiera, we believe that the possible interpretability / human-in-the-loop capabilities of \fbde are out-weight the higher runtime --- we can interpret this comparison as a trade-off between lower runtime vs interpretability properties. 

\begin{table*}[t]
    \caption{Runtime (s) of \fbde~(\(T = 32\)) and baselines evaluated for \compas and \adult datasets. The table reports the mean and s.t.d.}
    \label{tab:runtime}
    \begin{center}
    \begin{small}
    \begin{sc}
    \begin{tabularx}{\textwidth}{DDDDDDDDD} %{LLMXXXXXXXX}
    \toprule
& \mollifiera & \mollifierb & \mollifierc & \mollifierd & \celisa & \tabfairgan & \fairkmeans \\
        \midrule
\compas (sex) & \(88.467 \pm 0.193\) & \(90.362 \pm 0.296\) & \(89.070 \pm 0.443\) & \(89.123 \pm 0.210\) & \(0.740 \pm 1.127\) & \(25.234 \pm 0.131\) & \(62.365 \pm 0.251\) \\
        \midrule
\compas (race) & \(88.026 \pm 0.122\) & \(88.568 \pm 0.220\) & \(89.621 \pm 0.106\) & \(87.935 \pm 0.233\) & \(0.629 \pm 0.894\) & \(25.091 \pm 0.310\) & \(61.705 \pm 0.386\) \\
        \midrule
\adult (sex) & \(311.700 \pm 2.941\) & \(308.990 \pm 0.904\) & \(310.606 \pm 1.777\) & \(309.986 \pm 1.433\) & \(2.068 \pm 3.039\) & \(169.611 \pm 3.659\) & \(5446.451 \pm 54.161\) \\
        \midrule
\adult (sex) & \(313.745 \pm 1.462\) & \(311.175 \pm 1.319\) & \(313.996 \pm 0.755\) & \(315.955 \pm 6.093\) & \(1.735 \pm 2.094\) & \(171.614 \pm 3.297\) & \(5149.282 \pm 43.155\) \\
    \bottomrule
    \end{tabularx}
    \end{sc}
    \end{small}
    \end{center}
\end{table*}

\subsection{A Note about \texorpdfstring{\celisa}{MaxEnt}}

When using \celisa, we utilize take the following parameters: statistical rate setting \( \rho = 1 \), prior interpolation parameter \( C = 0.5 \), and marginal vector constrained by a weighted mean (from \citet[Algorithm 1]{ckvDP}).

\subsection{A Note About \texorpdfstring{\tabfairgan}{TabFair}}
We note the extremely poor performance of \tabfairgan in the \adult dataset with `race' sensitive attribute. We believe that the reason for this performance is the selection of suboptimal (hyper)parameters. Unfortunately, the model parameters presented in \citet[Table A1 in Appendix]{rgTF} are not applicable to our experiments as they utilize a different pre-processing of the \compas and \adult datasets. For \compas, we use a total of 128 epochs with 40 of those being `fair epochs' (\citet{rgTF}'s fairness training phase); in addition the batch sizes are specified to be 128 and the fairness regularizer parameter \( \lambda_{f} = 2.5 \). For \adult, we have 128 epochs with 30 of those being `fair epochs'; in addition the batch sizes are specified to be 128 and the fairness regularizer parameter \( \lambda_{f} = 2 \). Although we did not make an effort to fine-tune these (hyper)parameters for these extra settings in the Appendix, the lack of performance transfer when only the sensitive attribute is switched highlights a significant downside when compared to \fbde and other optimization approaches.

\subsection{Full Weak Learner Plot}
We present the full plot of \cref{fig:wl_teaser} in \cref{fig:full_wl}. We leave the careful examination of such weak learners to future work and rather present these plots as a proof of concept.

\begin{sidewaysfigure}[t]
    \centering
    \includegraphics[width=0.24\columnwidth]{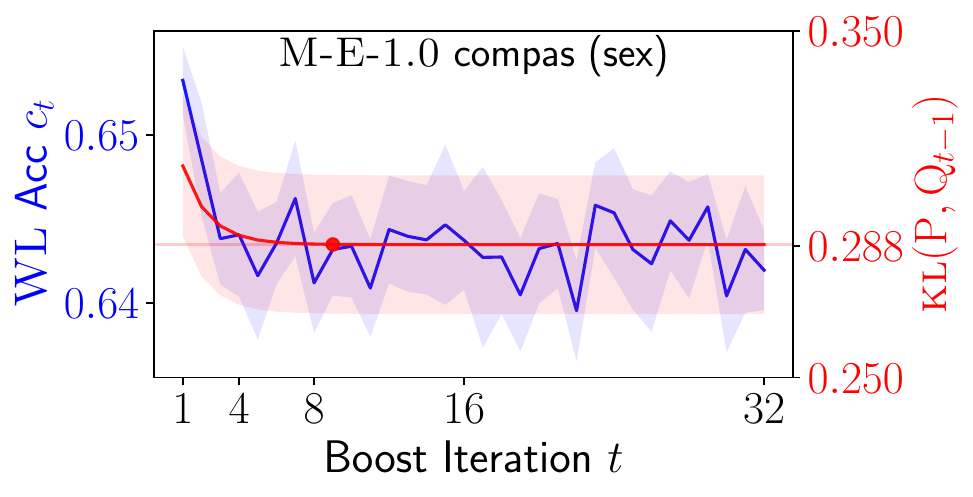}%
    \includegraphics[width=0.24\columnwidth]{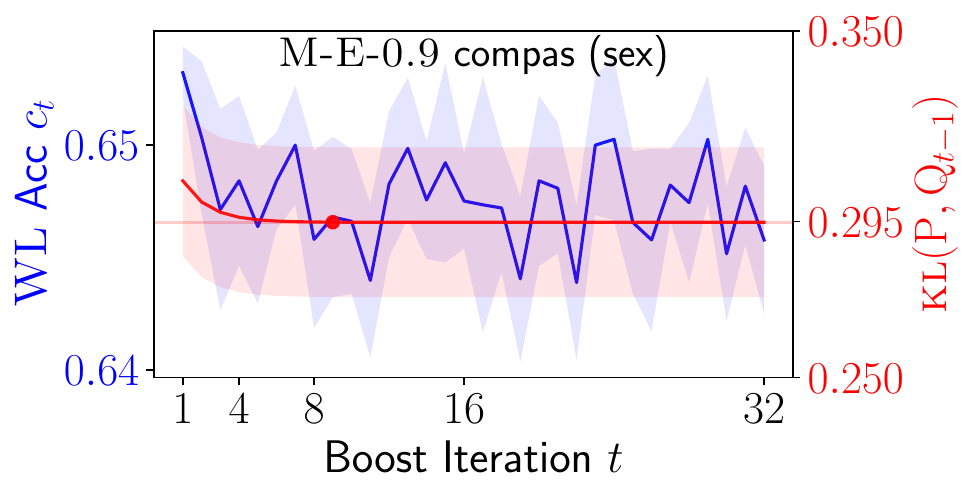}%
    \includegraphics[width=0.24\columnwidth]{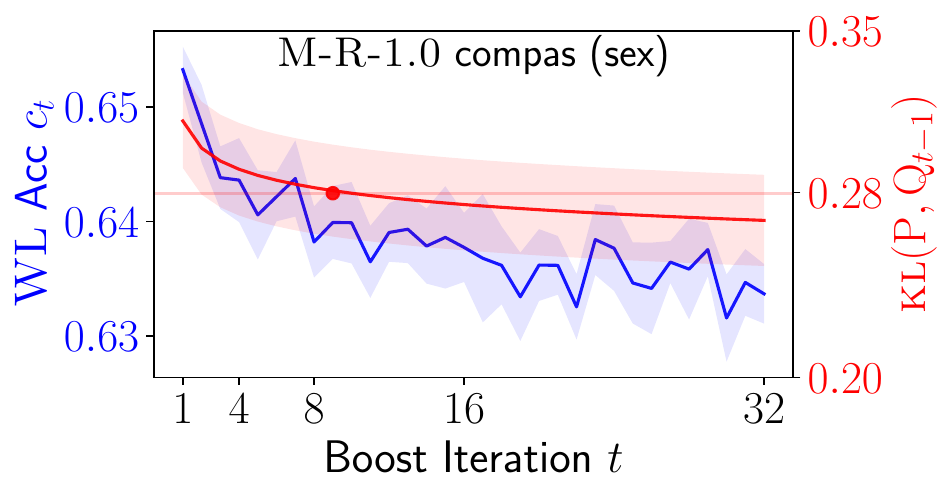}%
    \includegraphics[width=0.24\columnwidth]{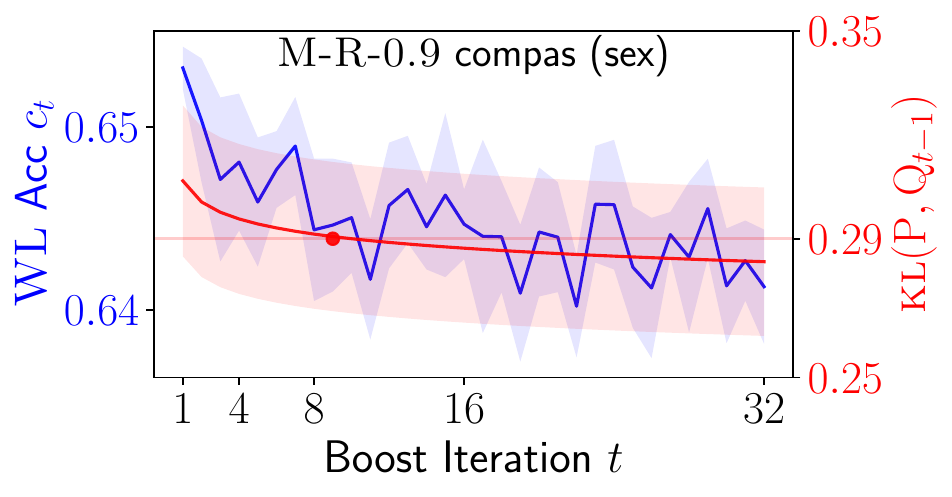}%
    \\
    \includegraphics[width=0.24\columnwidth]{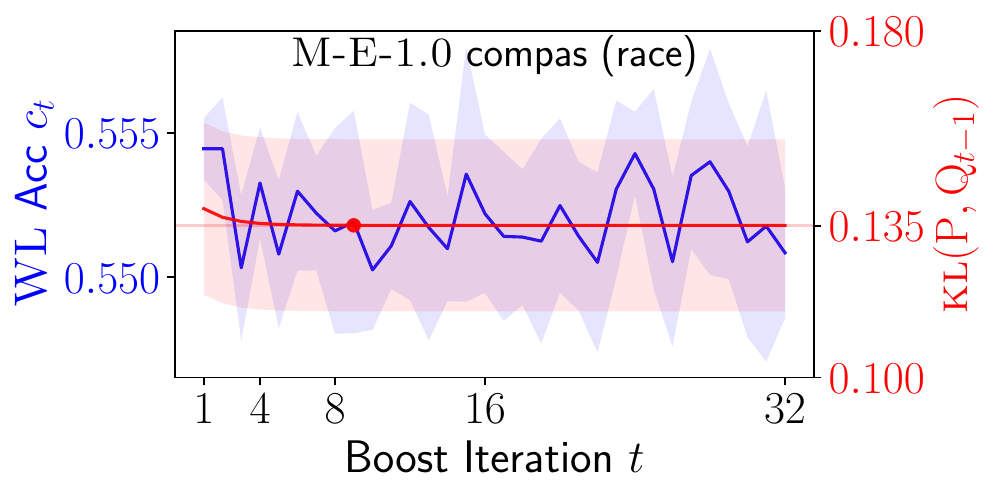}%
    \includegraphics[width=0.24\columnwidth]{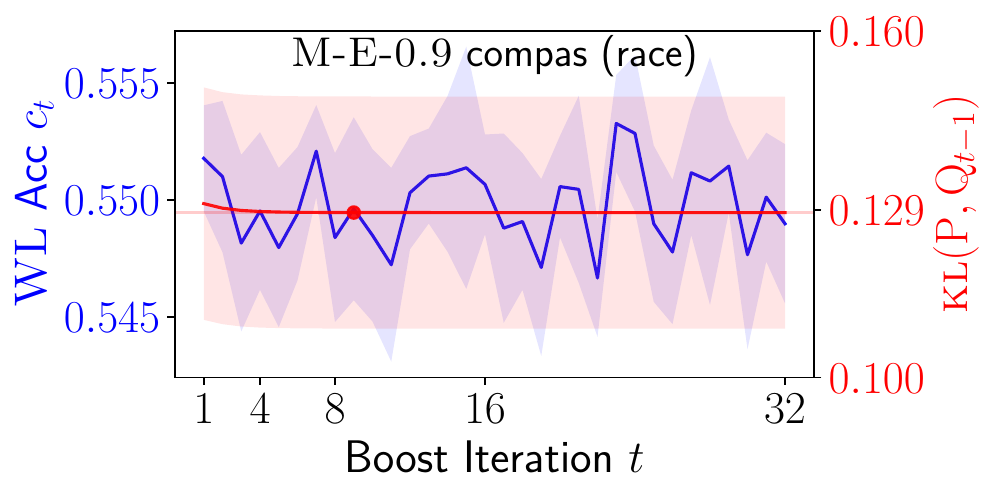}%
    \includegraphics[width=0.24\columnwidth]{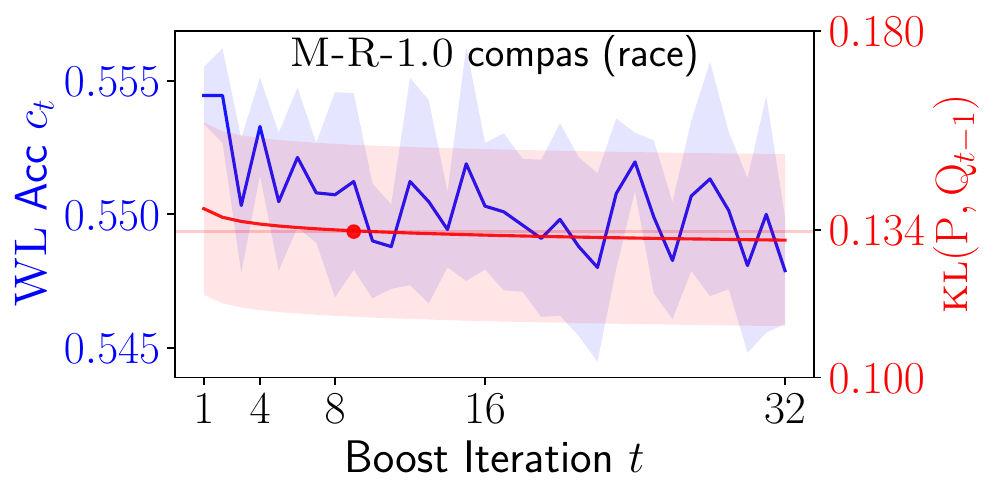}%
    \includegraphics[width=0.24\columnwidth]{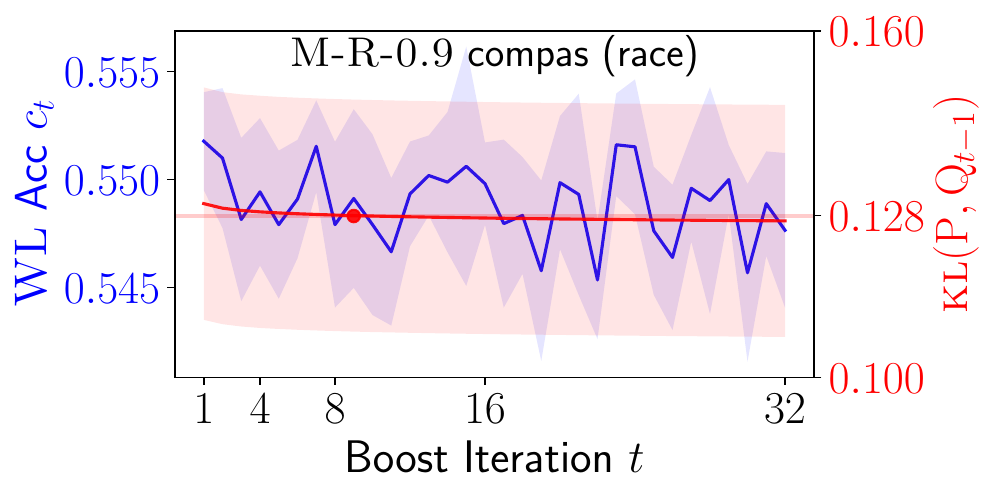}%
    \\
    \includegraphics[width=0.24\columnwidth]{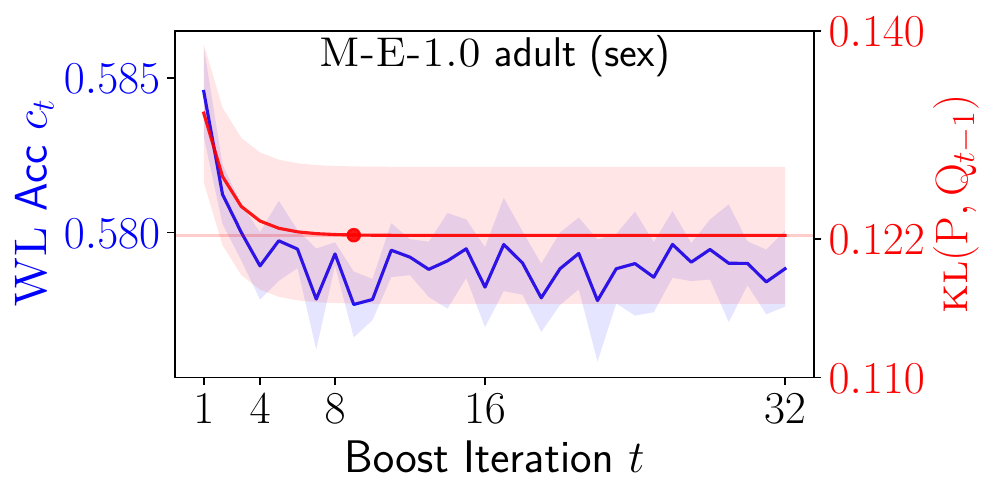}%
    \includegraphics[width=0.24\columnwidth]{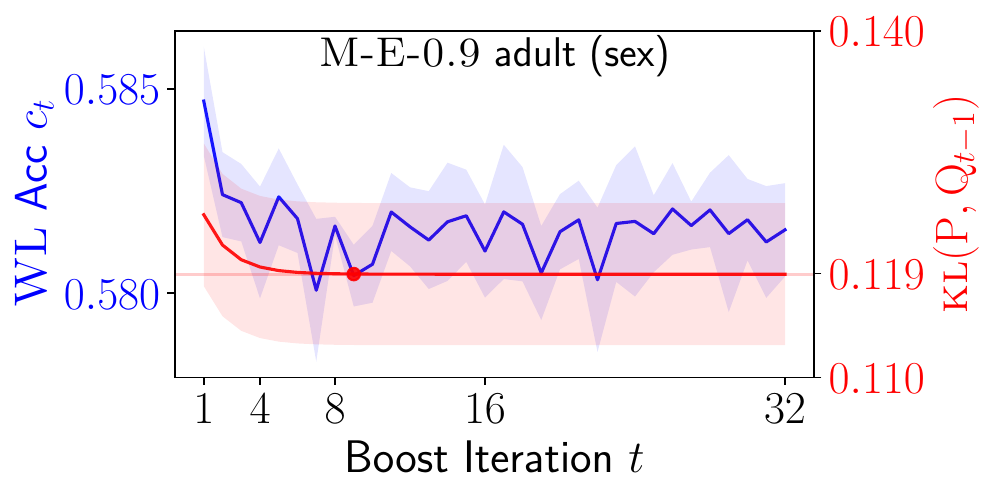}%
    \includegraphics[width=0.24\columnwidth]{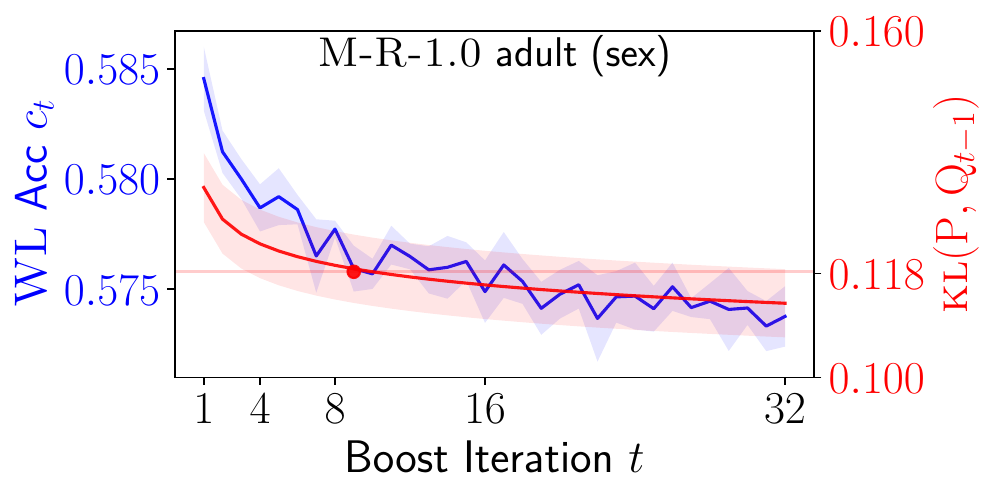}%
    \includegraphics[width=0.24\columnwidth]{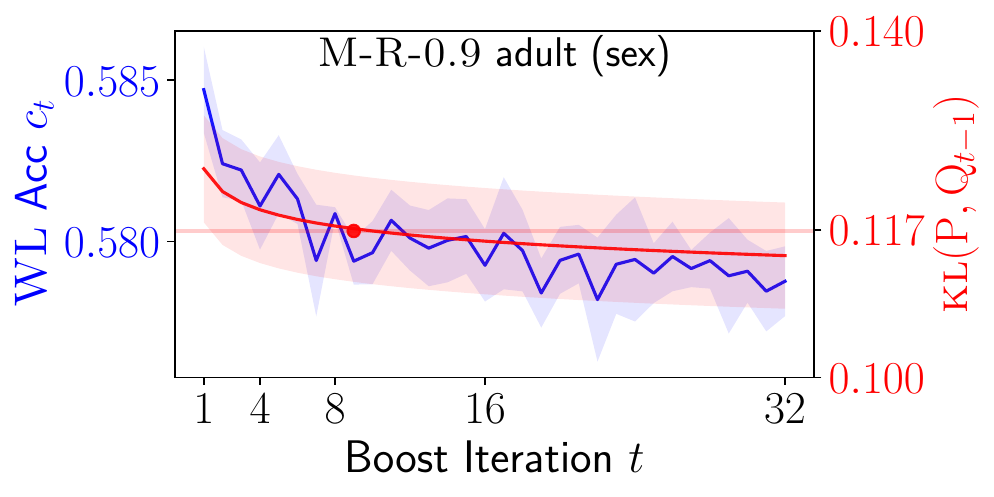}%
    \\
    \includegraphics[width=0.24\columnwidth]{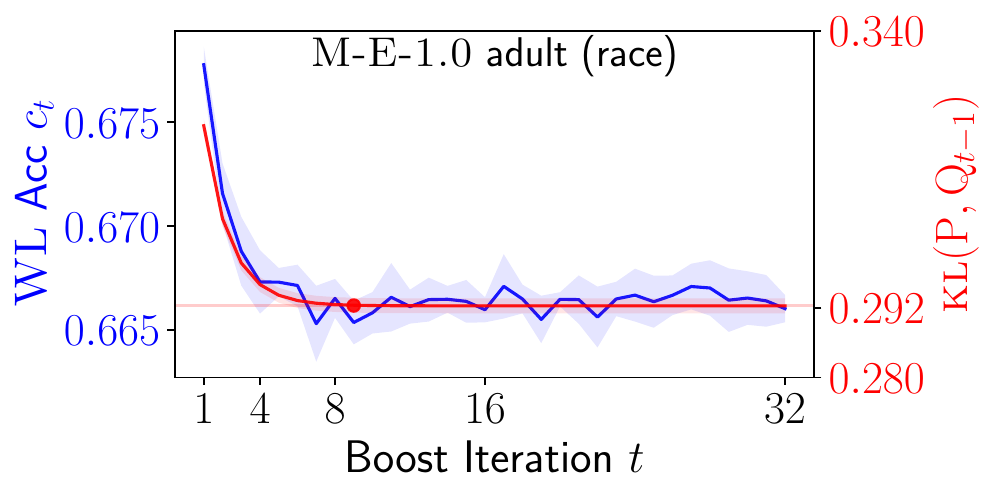}%
    \includegraphics[width=0.24\columnwidth]{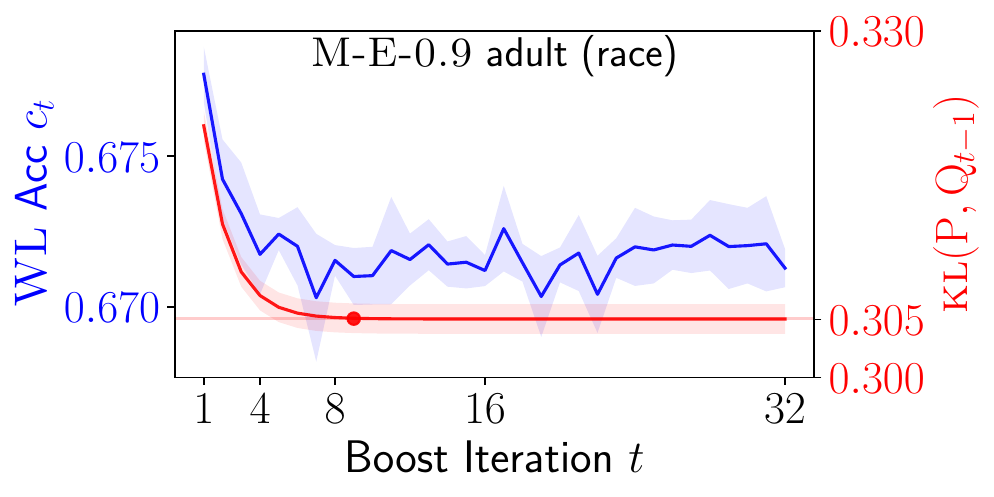}%
    \includegraphics[width=0.24\columnwidth]{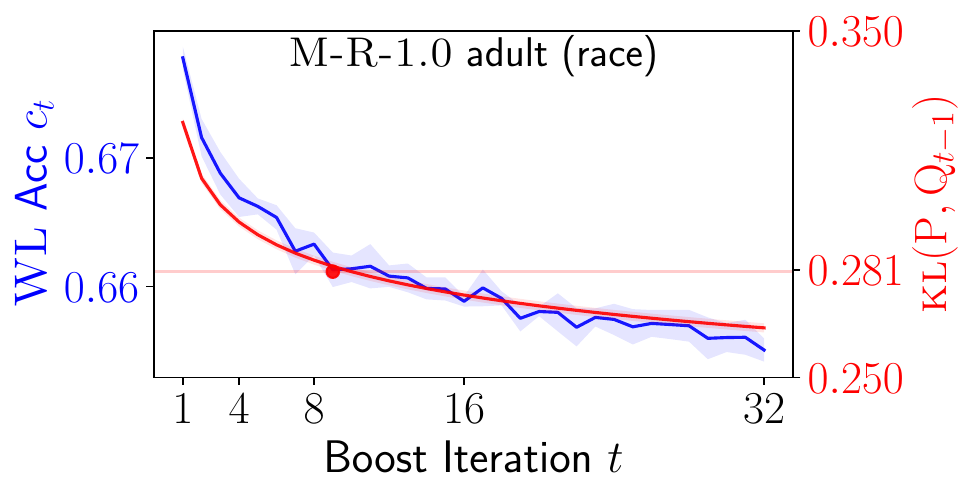}%
    \includegraphics[width=0.24\columnwidth]{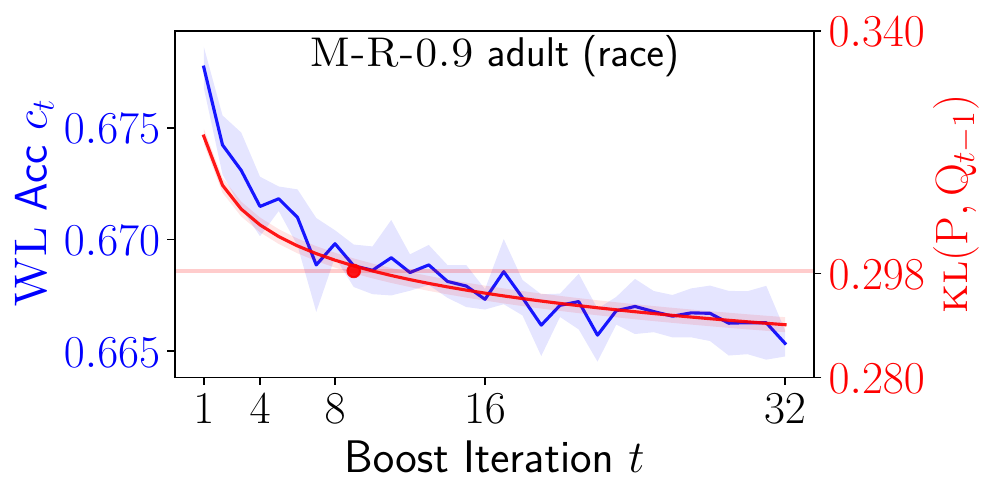}%
    \caption{All WL accuracy vs KL over boosting iterations plots for \compas and \adult. Horizontal line depicts the  \( \kl(\meas{P}, \meas{Q}_{8})\) value.}
    \label{fig:wl_all_plots}
\end{sidewaysfigure}

\begin{sidewaysfigure}[t]
    \centering
    \includegraphics[width=0.24\columnwidth]{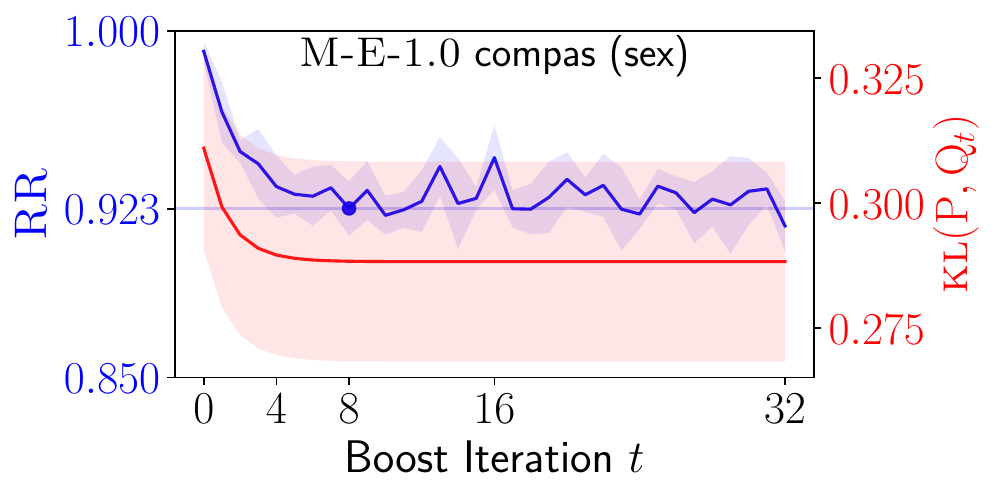}%
    \includegraphics[width=0.24\columnwidth]{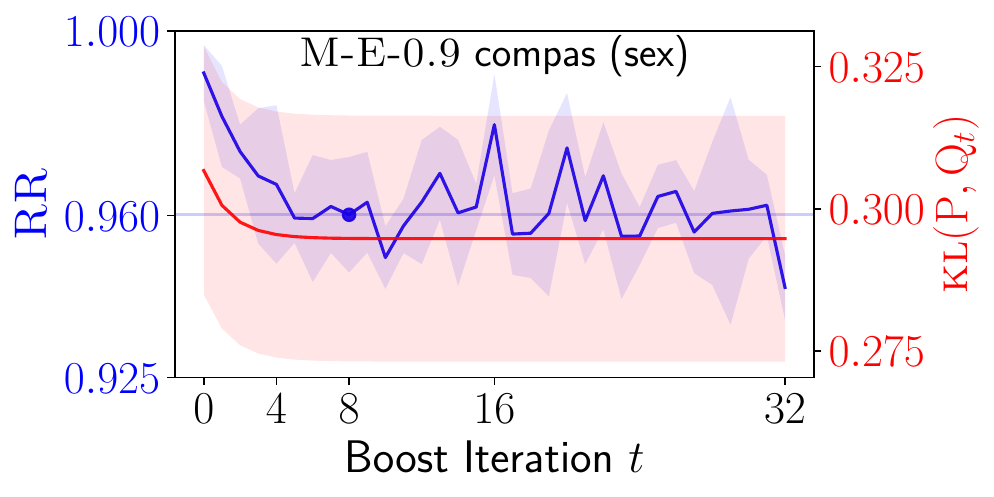}%
    \includegraphics[width=0.24\columnwidth]{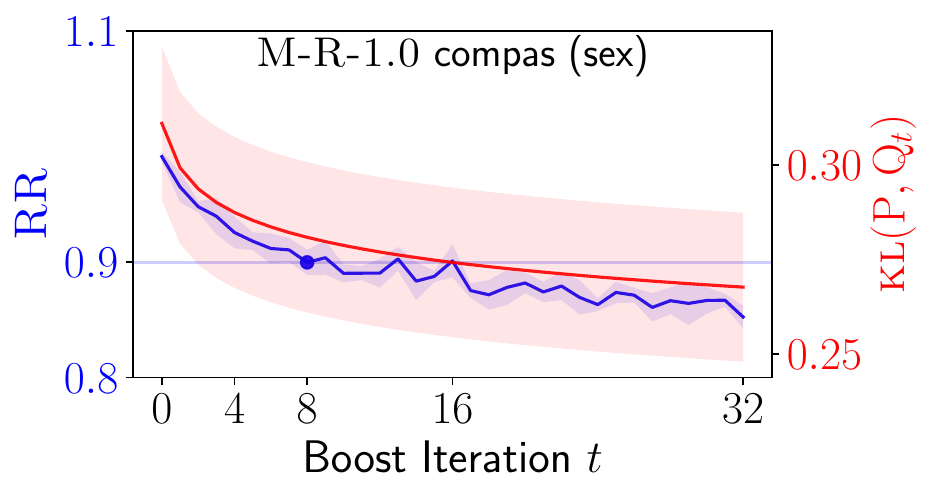}%
    \includegraphics[width=0.24\columnwidth]{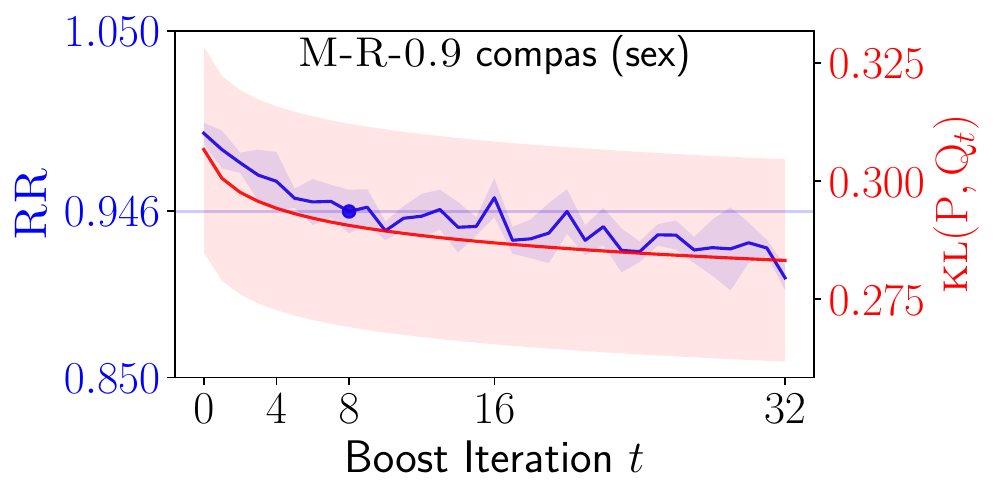}%
    \\
    \includegraphics[width=0.24\columnwidth]{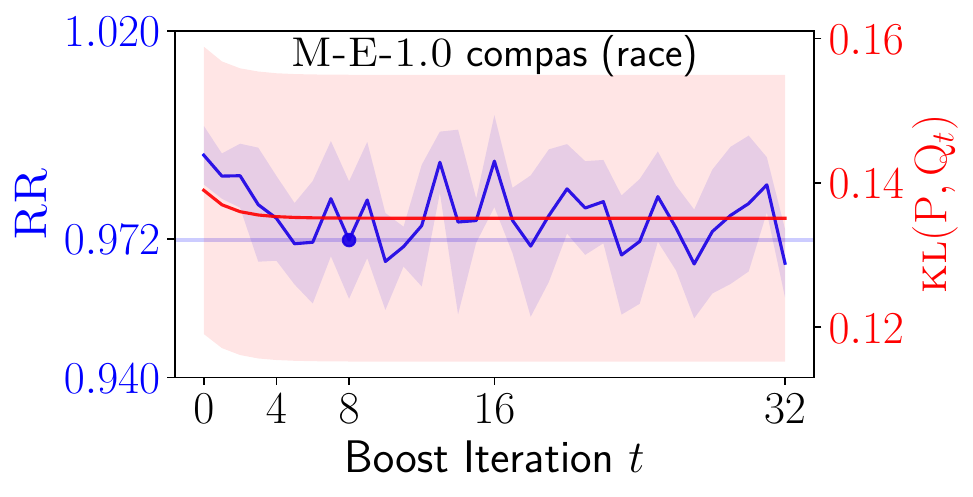}%
    \includegraphics[width=0.24\columnwidth]{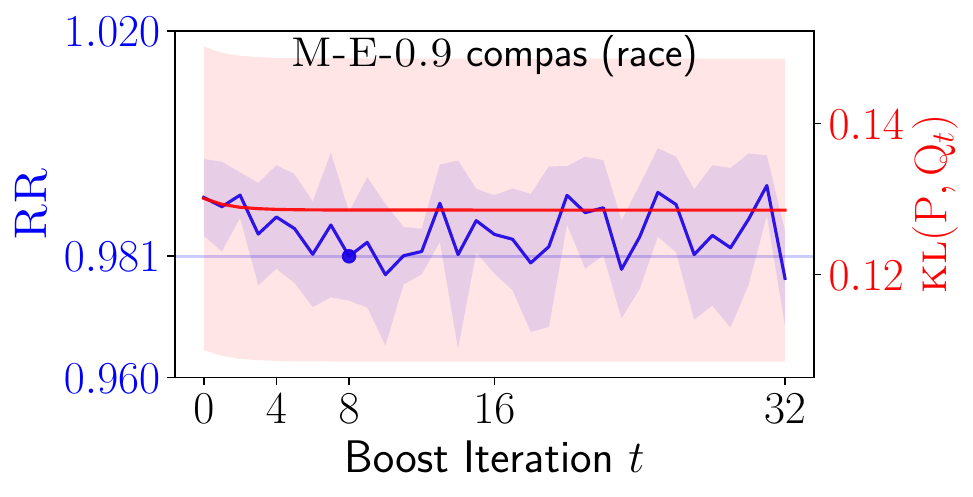}%
    \includegraphics[width=0.24\columnwidth]{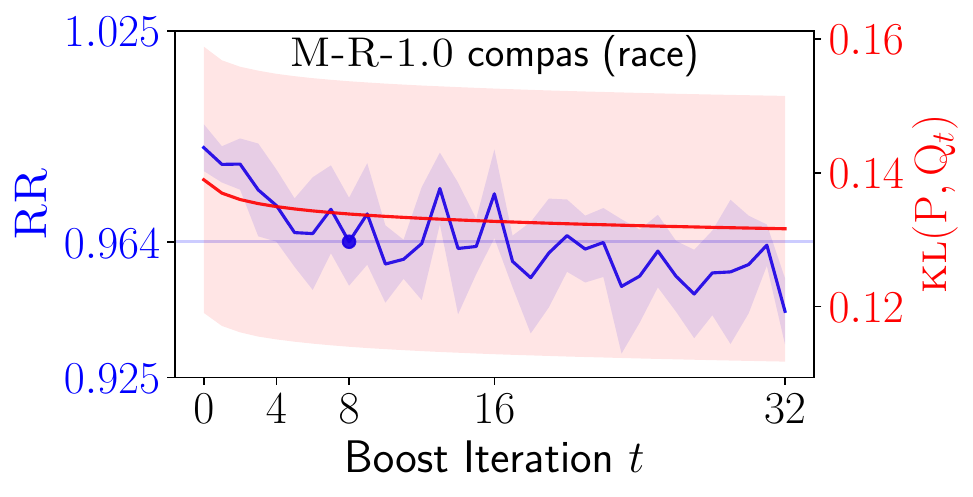}%
    \includegraphics[width=0.24\columnwidth]{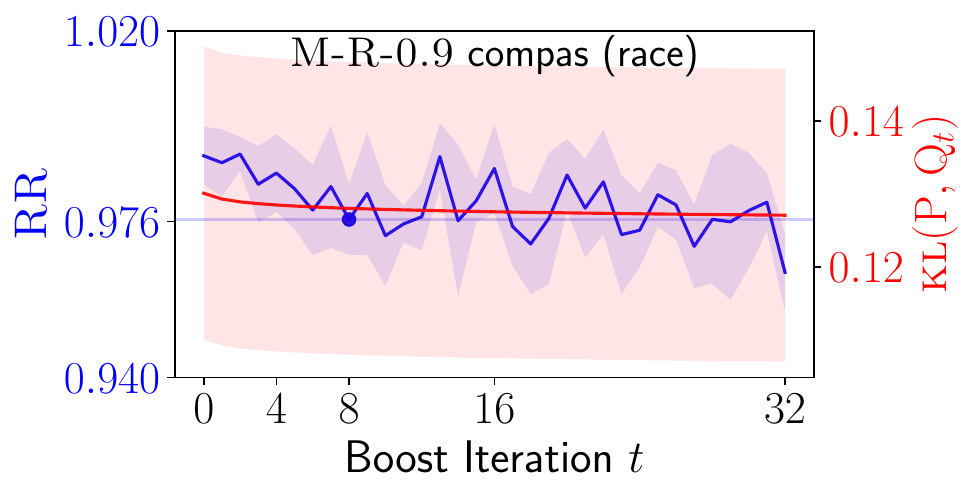}%
    \\
    \includegraphics[width=0.24\columnwidth]{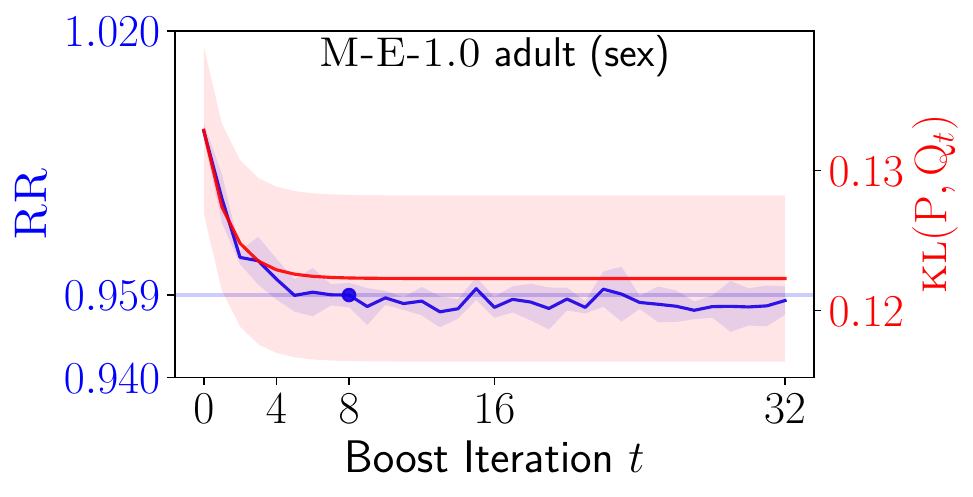}%
    \includegraphics[width=0.24\columnwidth]{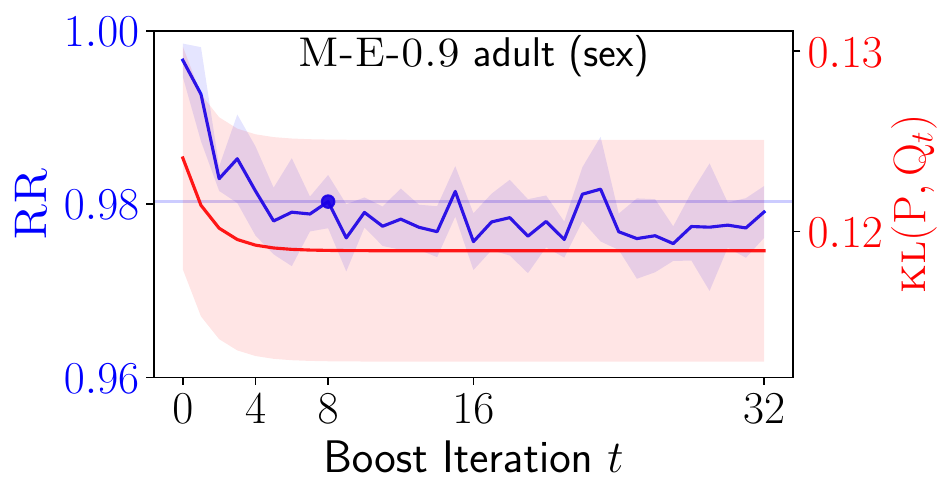}%
    \includegraphics[width=0.24\columnwidth]{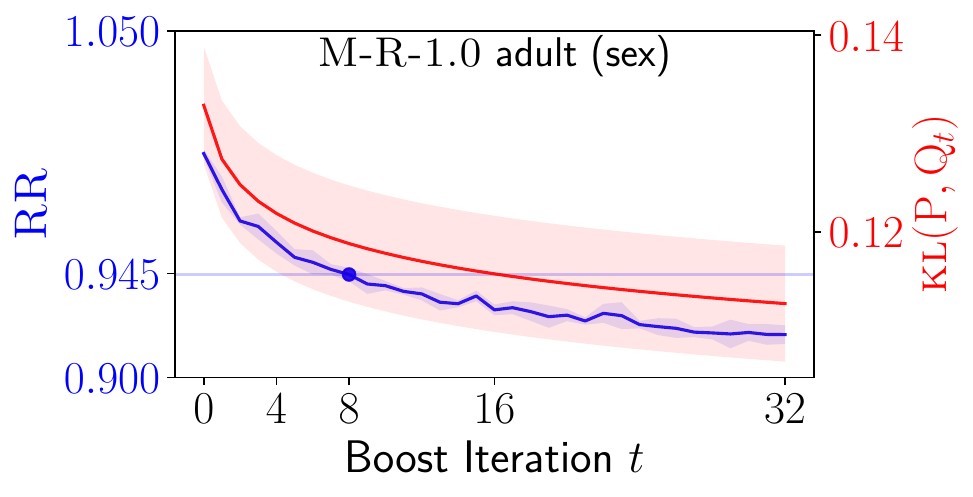}%
    \includegraphics[width=0.24\columnwidth]{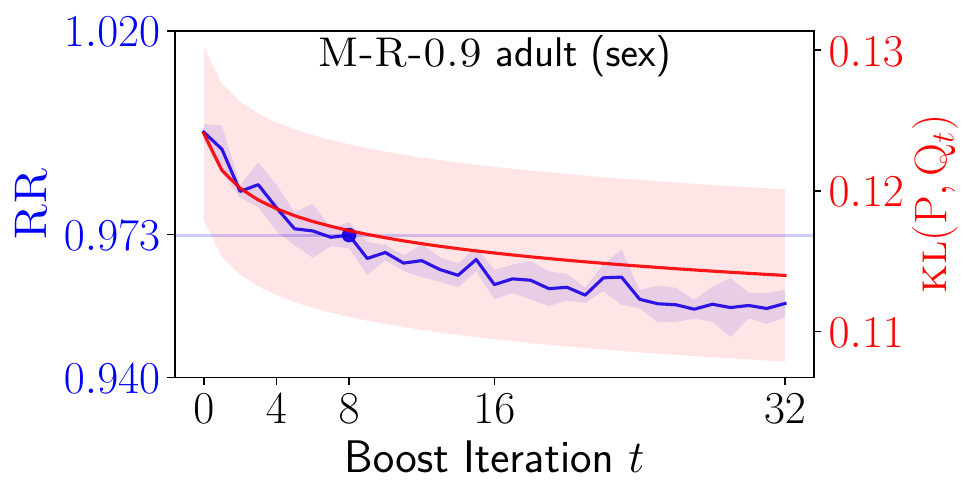}%
    \\
    \includegraphics[width=0.24\columnwidth]{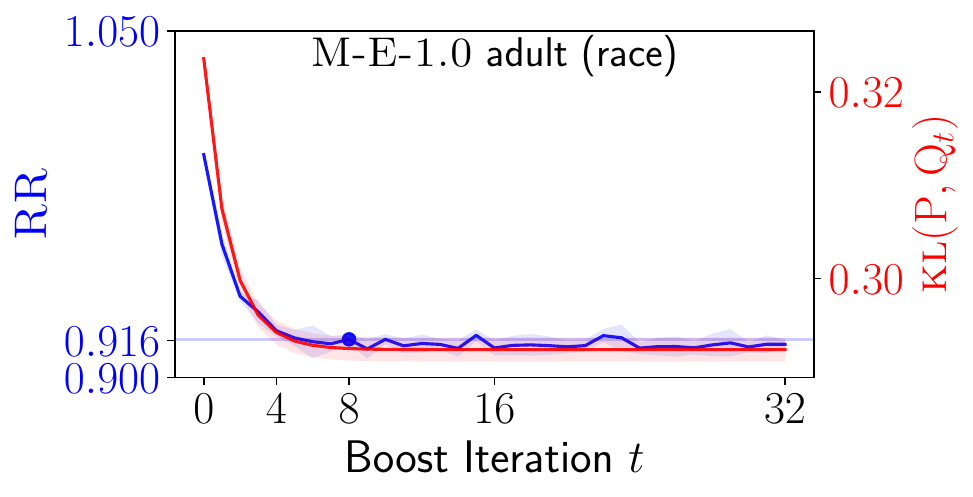}%
    \includegraphics[width=0.24\columnwidth]{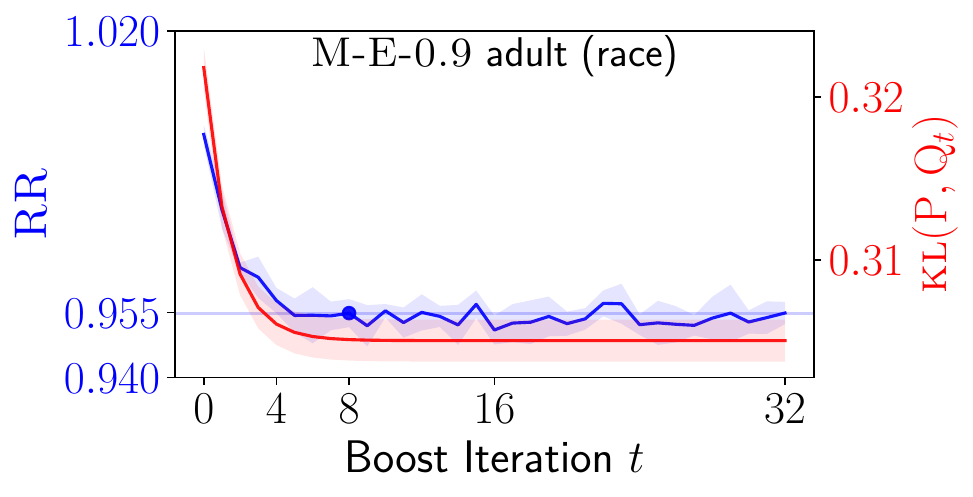}%
    \includegraphics[width=0.24\columnwidth]{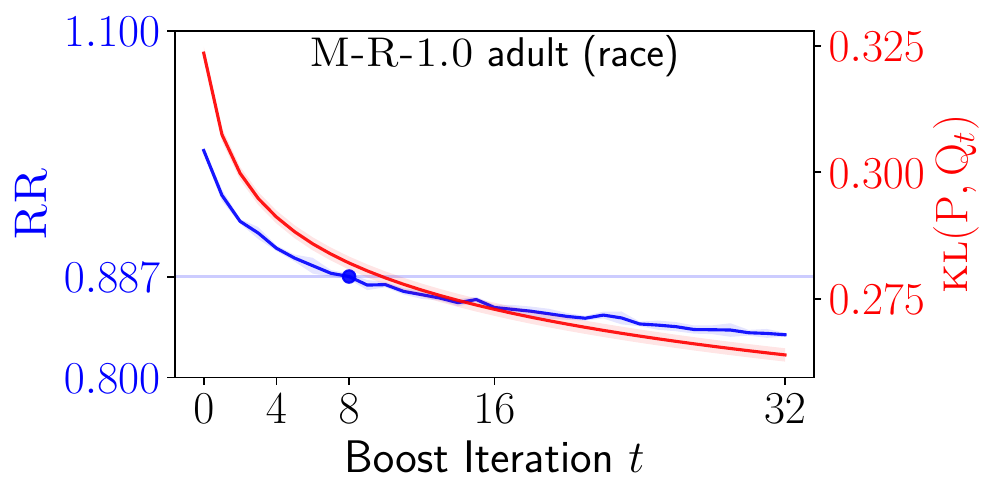}%
    \includegraphics[width=0.24\columnwidth]{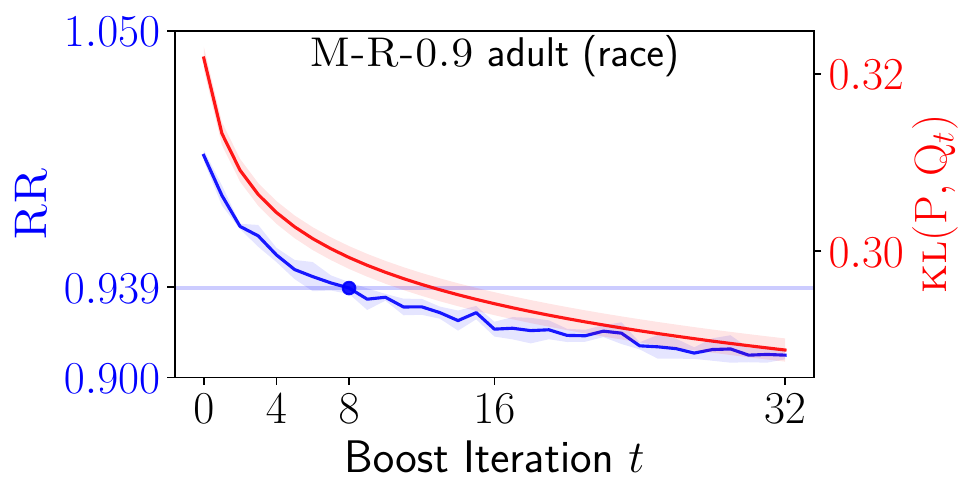}%
    \caption{All RR vs KL over boosting iterations plots for \compas and \adult. Horizontal line depicts the \( T = 8 \) RR value.}
    \label{fig:rr_all_plots}
\end{sidewaysfigure}

\begin{sidewaysfigure}[t]
    \centering
    \includegraphics[width=0.24\columnwidth]{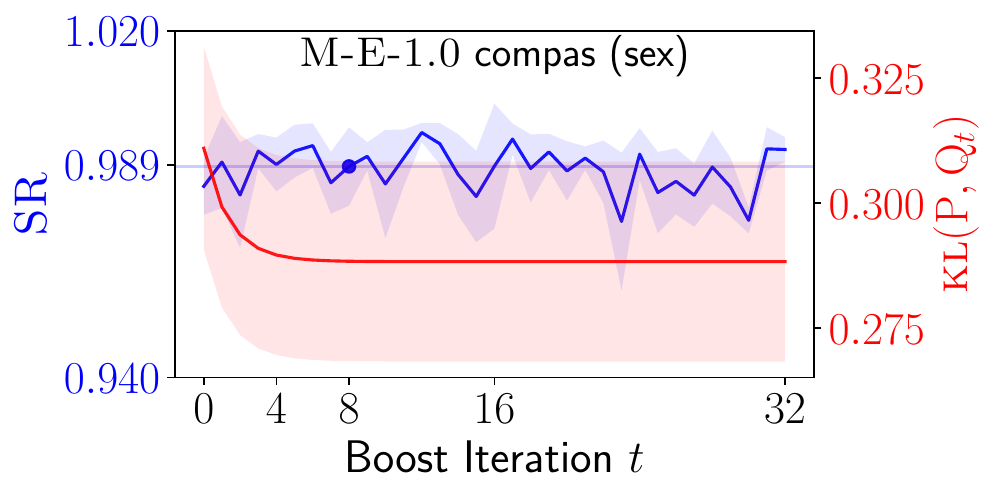}%
    \includegraphics[width=0.24\columnwidth]{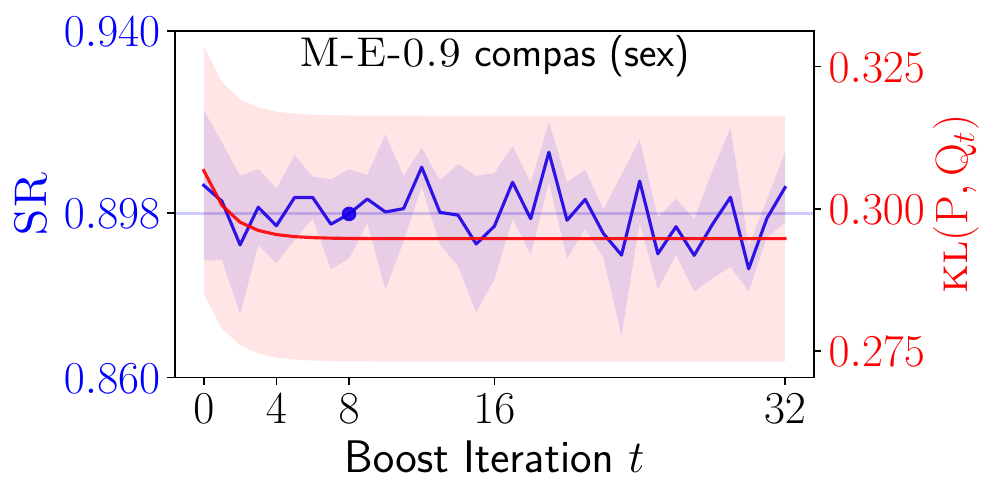}%
    \includegraphics[width=0.24\columnwidth]{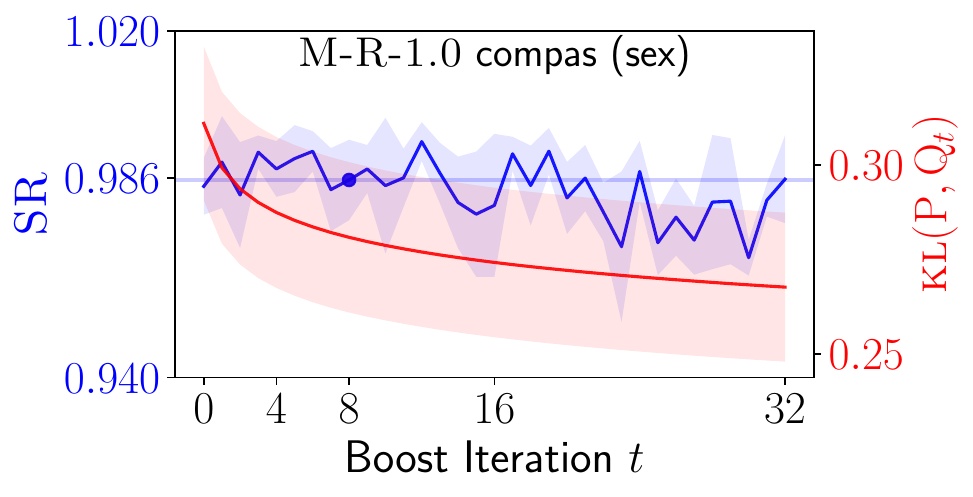}%
    \includegraphics[width=0.24\columnwidth]{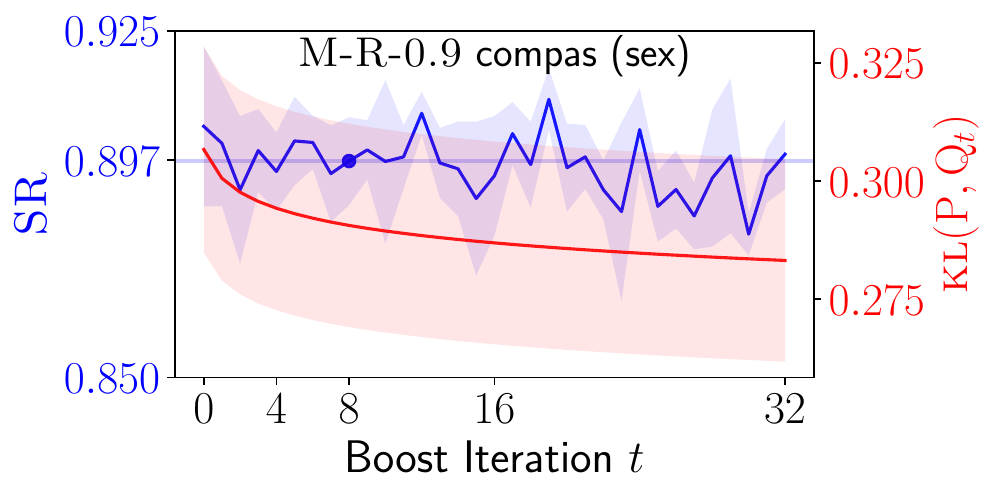}%
    \\
    \includegraphics[width=0.24\columnwidth]{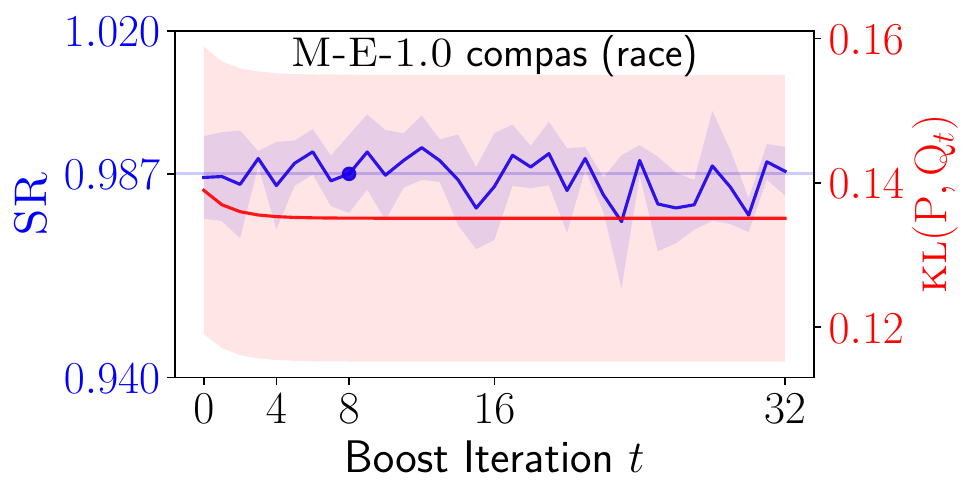}%
    \includegraphics[width=0.24\columnwidth]{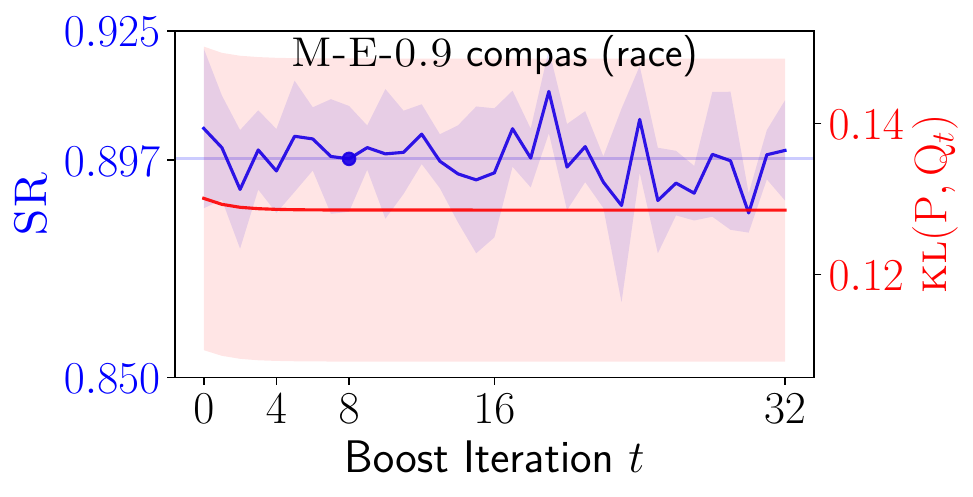}%
    \includegraphics[width=0.24\columnwidth]{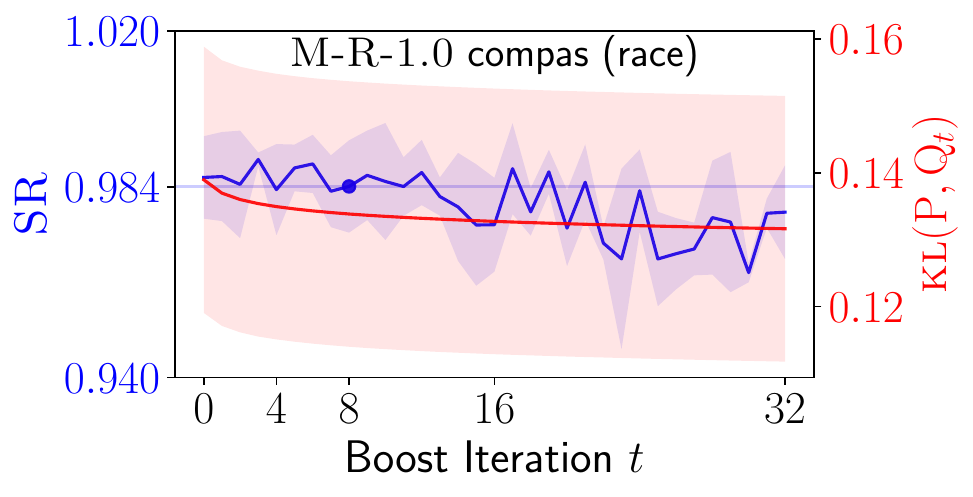}%
    \includegraphics[width=0.24\columnwidth]{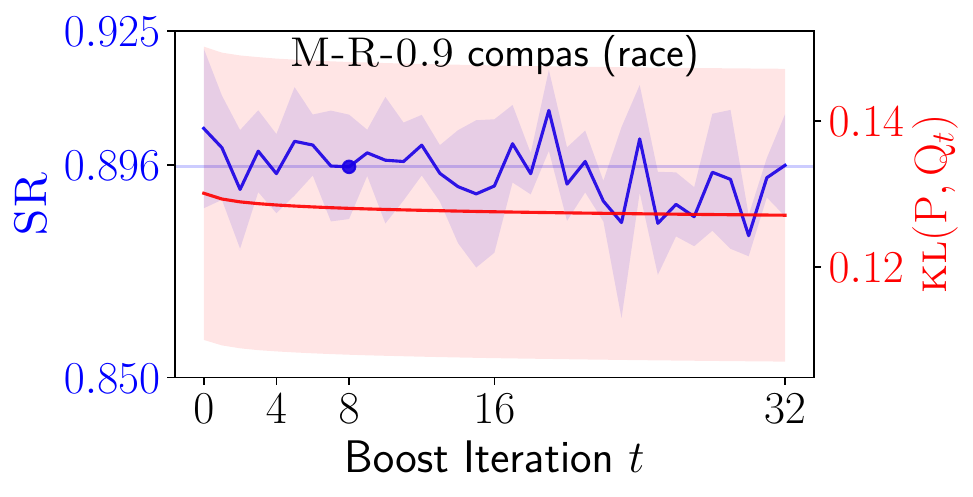}%
    \\
    \includegraphics[width=0.24\columnwidth]{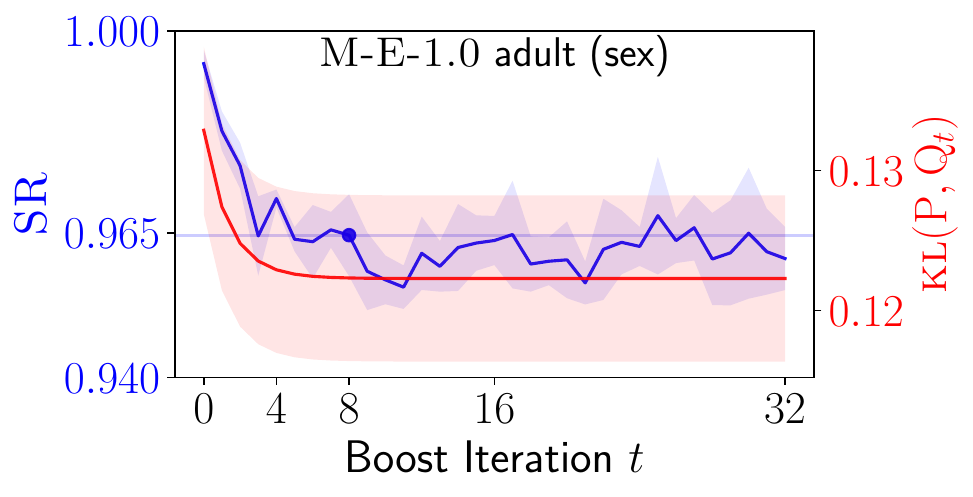}%
    \includegraphics[width=0.24\columnwidth]{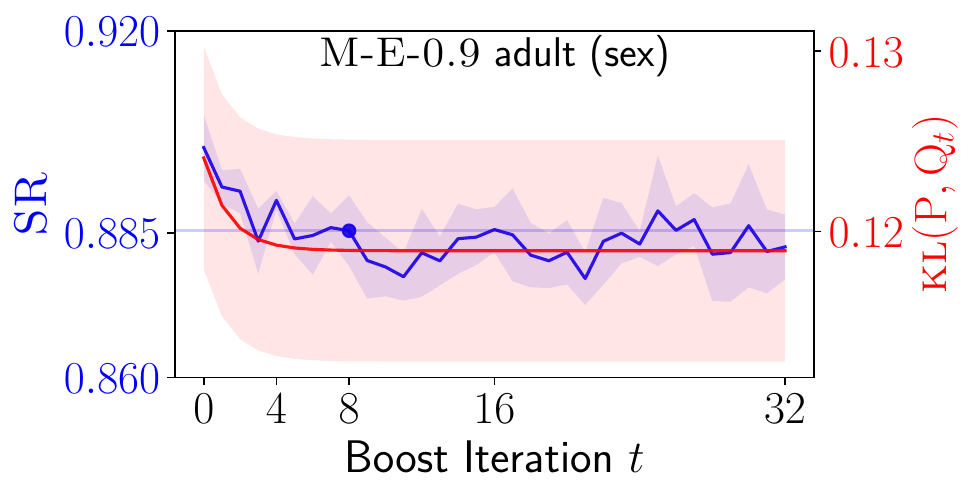}%
    \includegraphics[width=0.24\columnwidth]{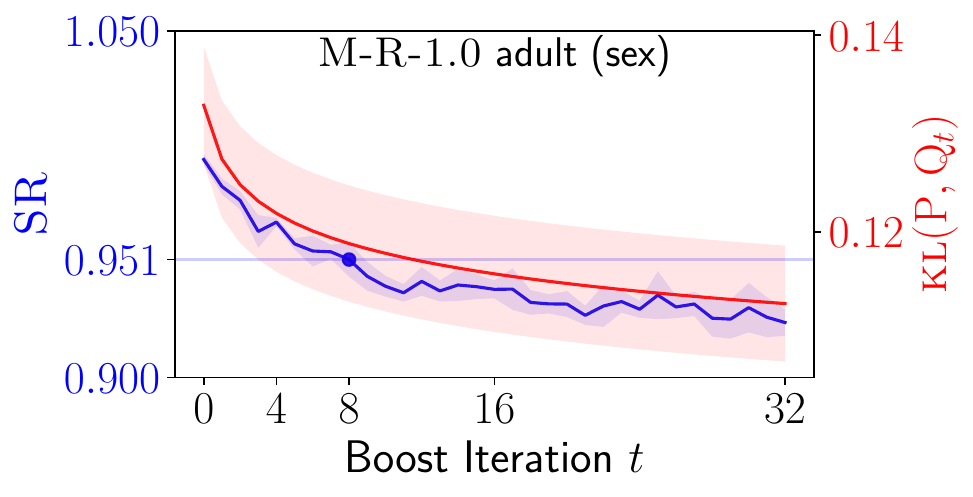}%
    \includegraphics[width=0.24\columnwidth]{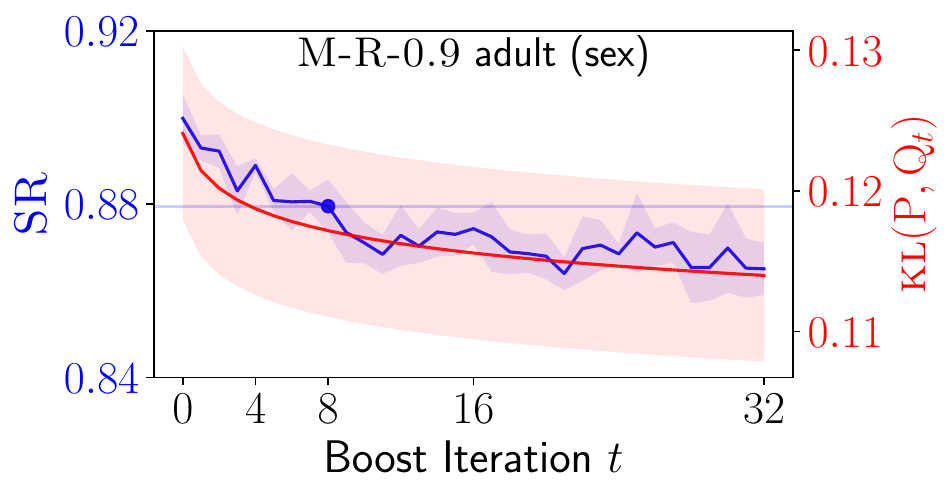}%
    \\
    \includegraphics[width=0.24\columnwidth]{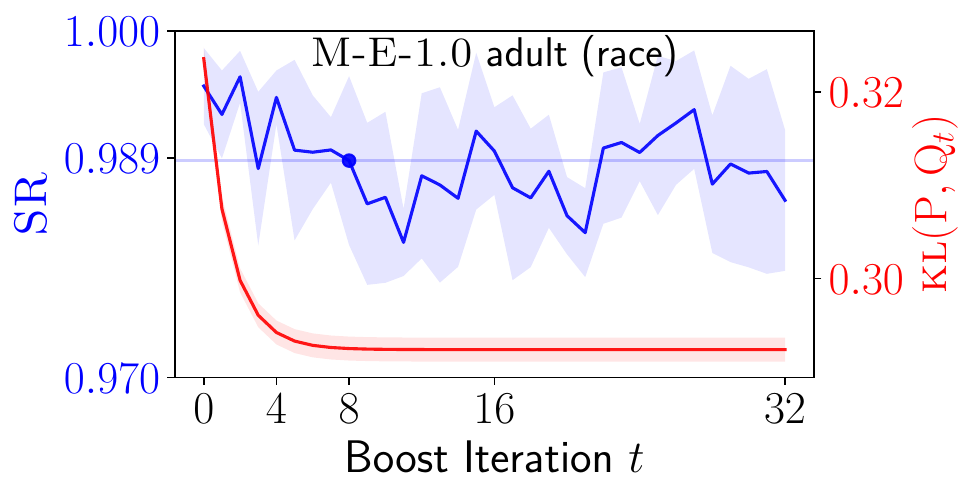}%
    \includegraphics[width=0.24\columnwidth]{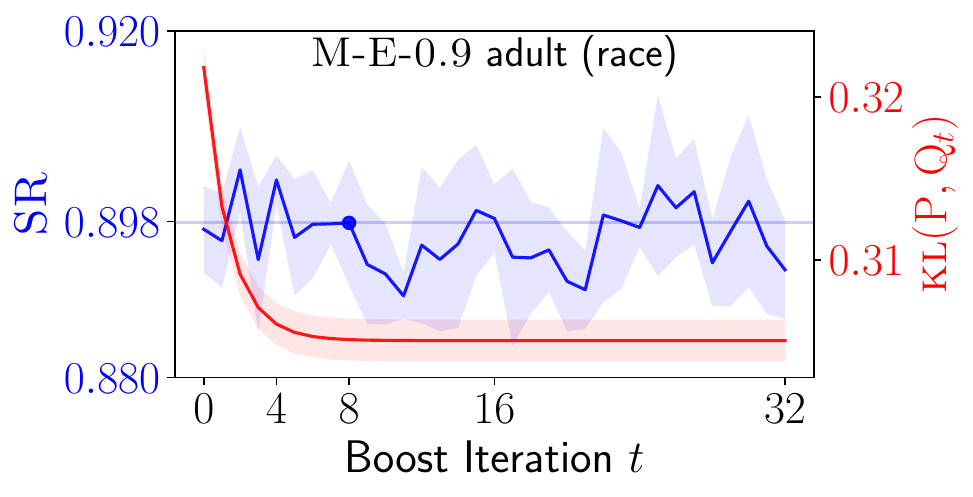}%
    \includegraphics[width=0.24\columnwidth]{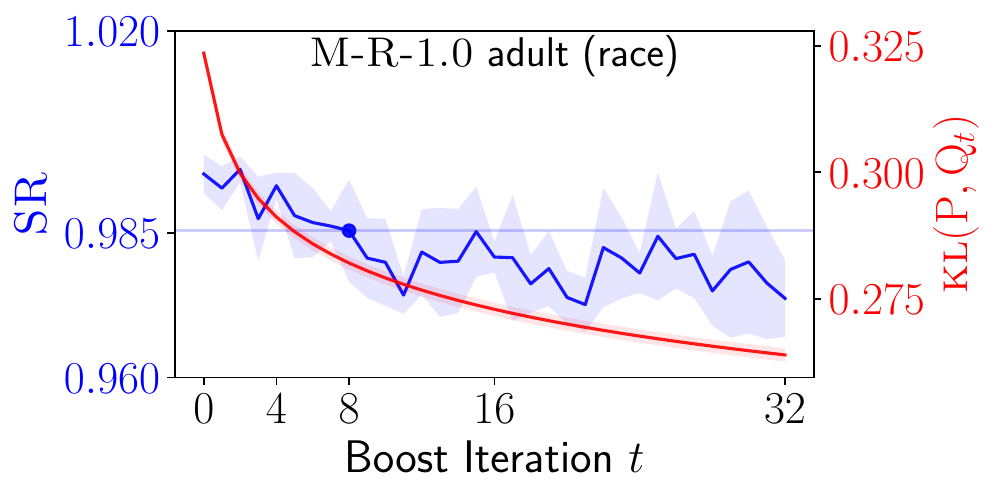}%
    \includegraphics[width=0.24\columnwidth]{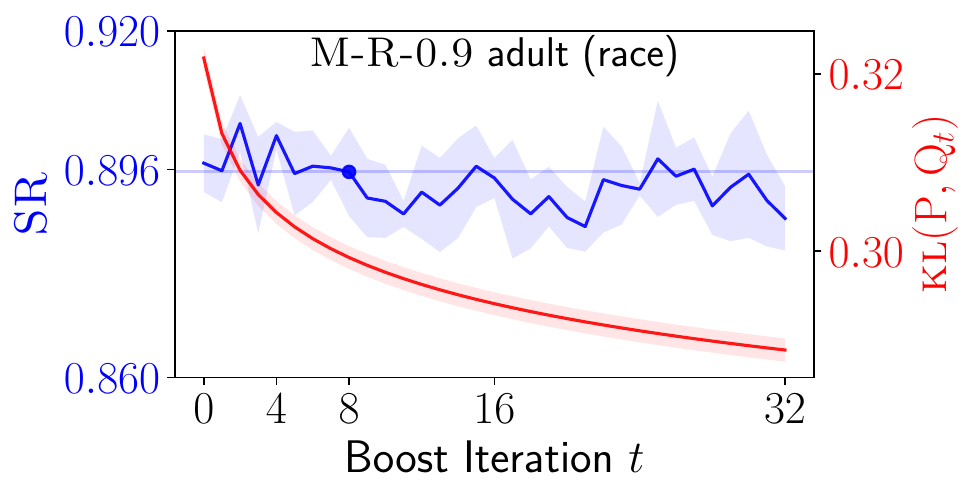}%
    \caption{All SR vs KL over boosting iterations plots for \compas and \adult. Horizontal line depicts the \( T = 8 \) SR value.}
    \label{fig:sr_all_plots}
\end{sidewaysfigure}

\begin{sidewaysfigure}[t]
    \centering
    \includegraphics[width=0.24\columnwidth]{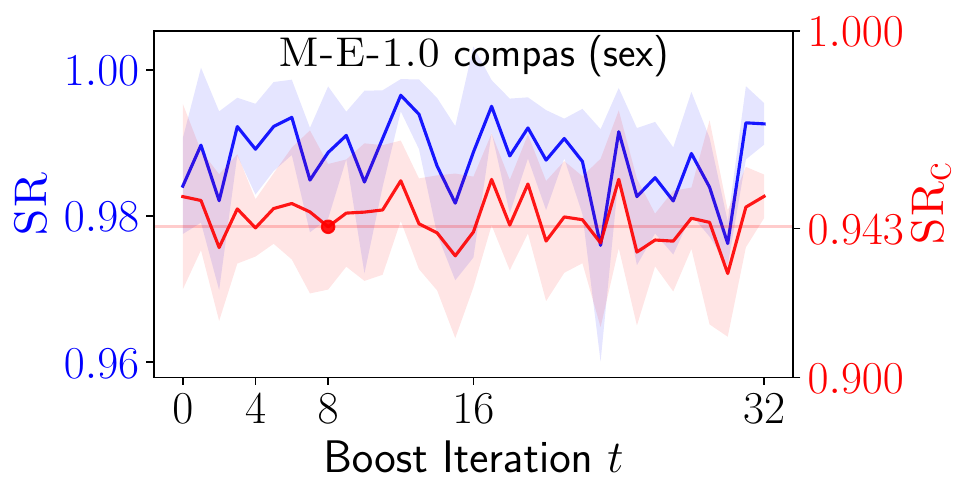}%
    \includegraphics[width=0.24\columnwidth]{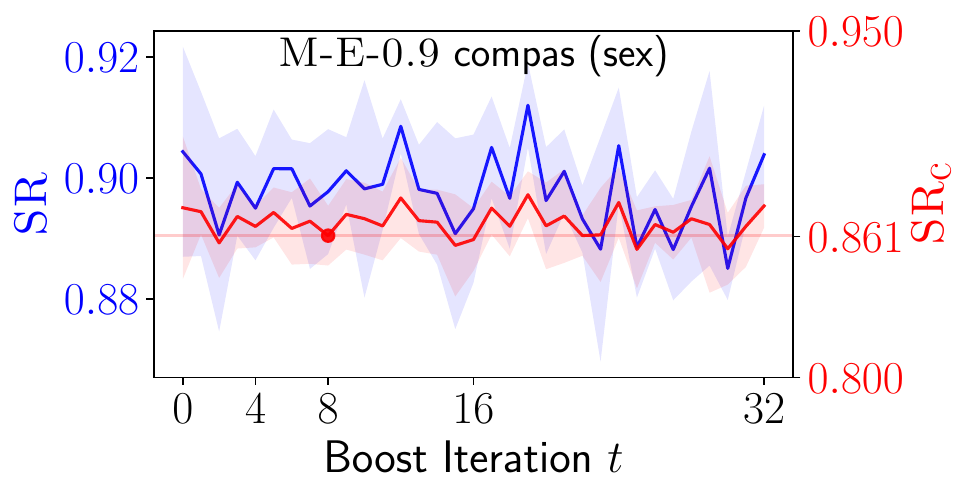}%
    \includegraphics[width=0.24\columnwidth]{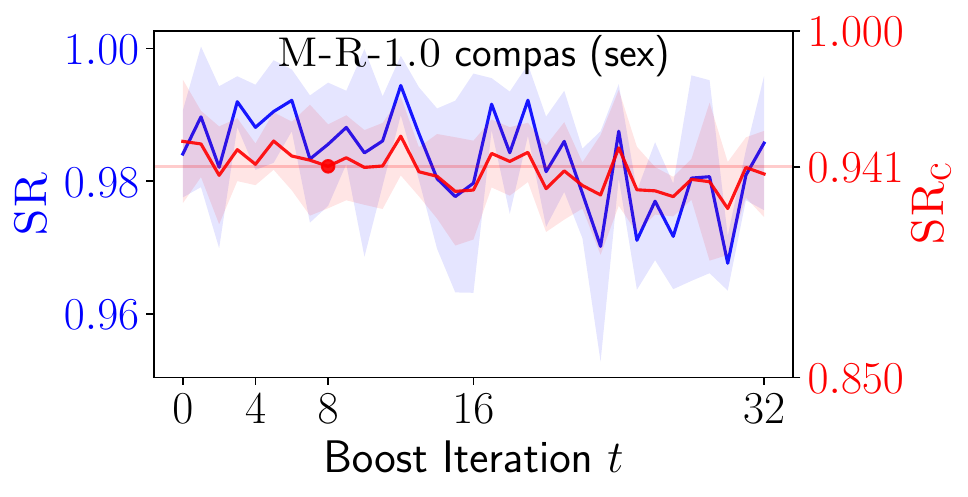}%
    \includegraphics[width=0.24\columnwidth]{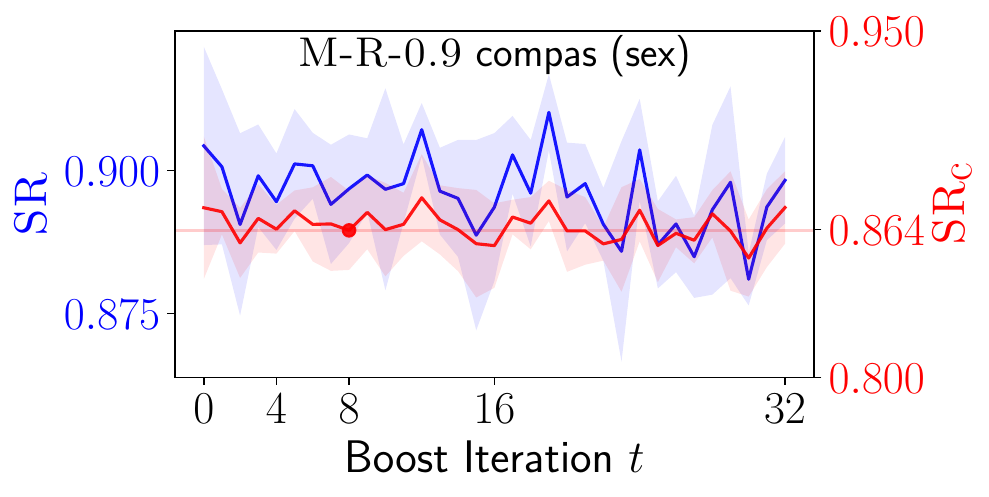}%
    \\
    \includegraphics[width=0.24\columnwidth]{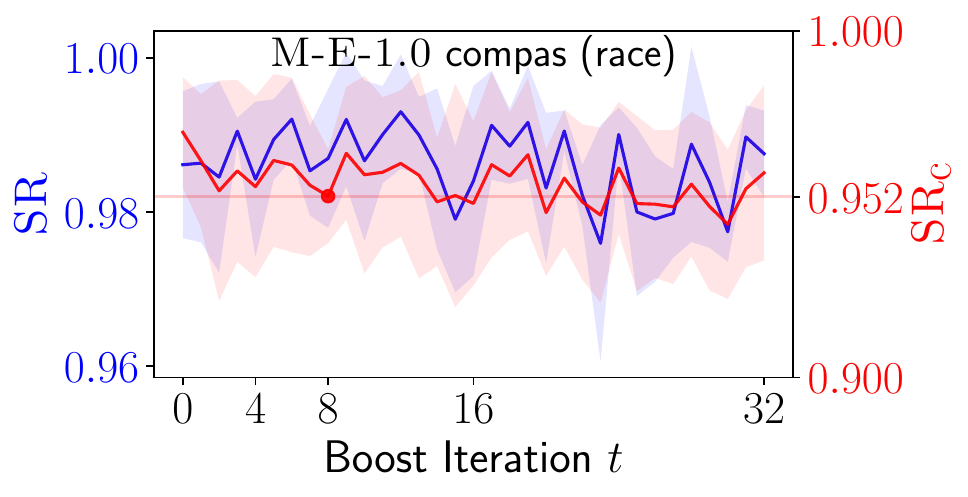}%
    \includegraphics[width=0.24\columnwidth]{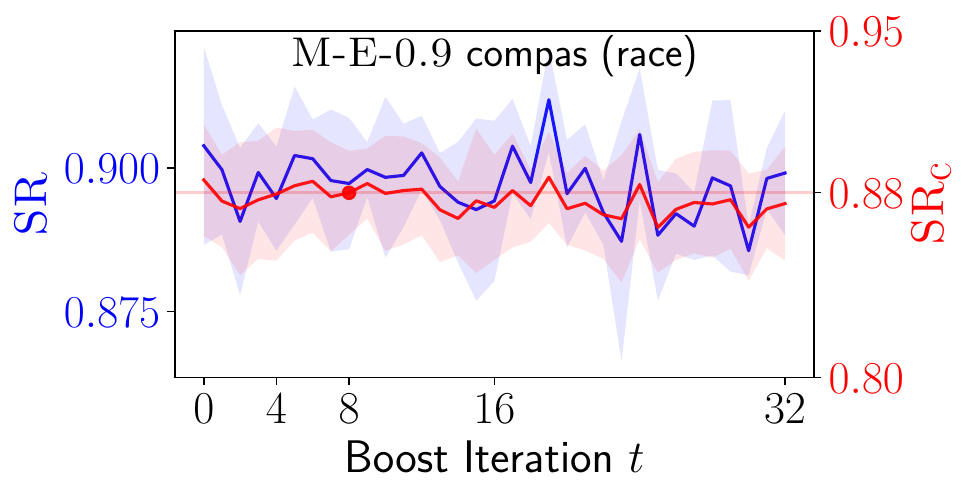}%
    \includegraphics[width=0.24\columnwidth]{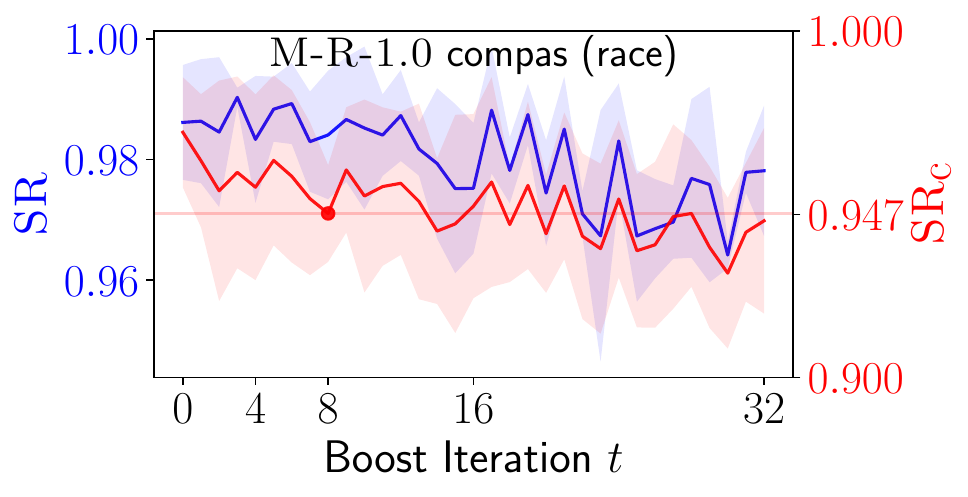}%
    \includegraphics[width=0.24\columnwidth]{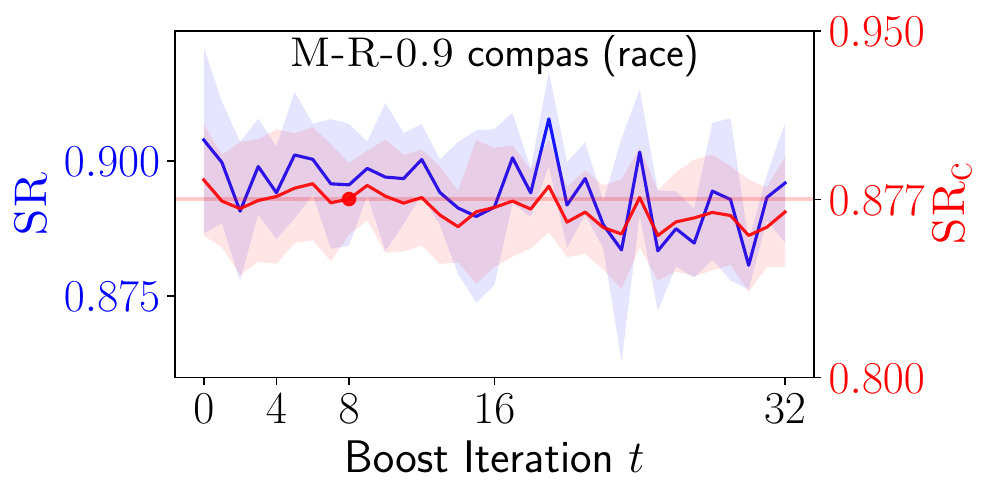}%
    \\
    \includegraphics[width=0.24\columnwidth]{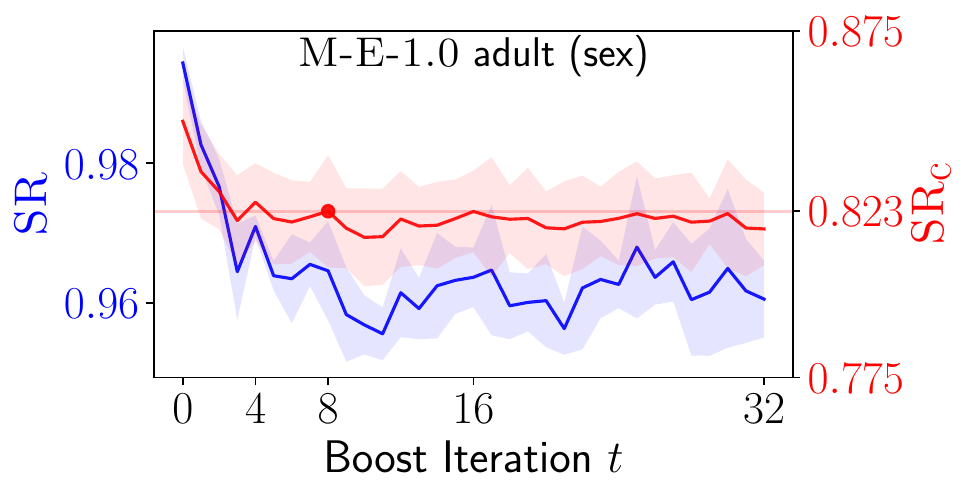}%
    \includegraphics[width=0.24\columnwidth]{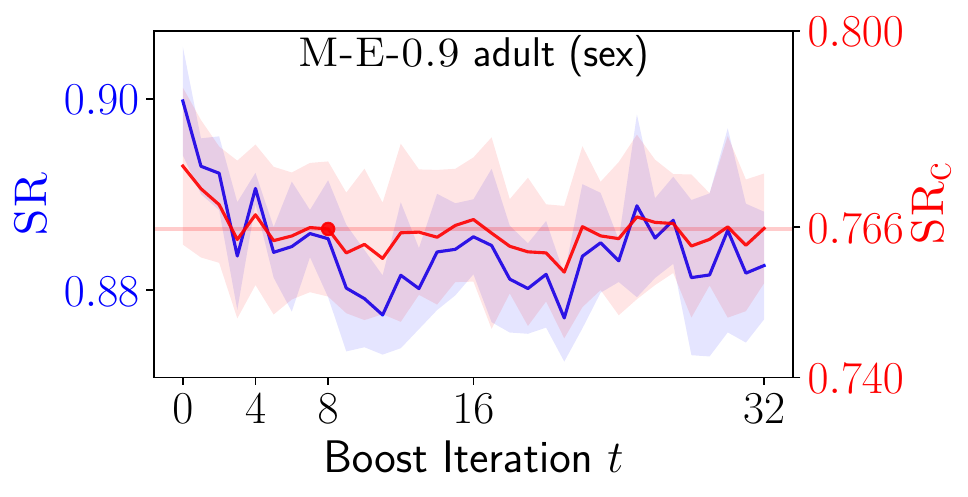}%
    \includegraphics[width=0.24\columnwidth]{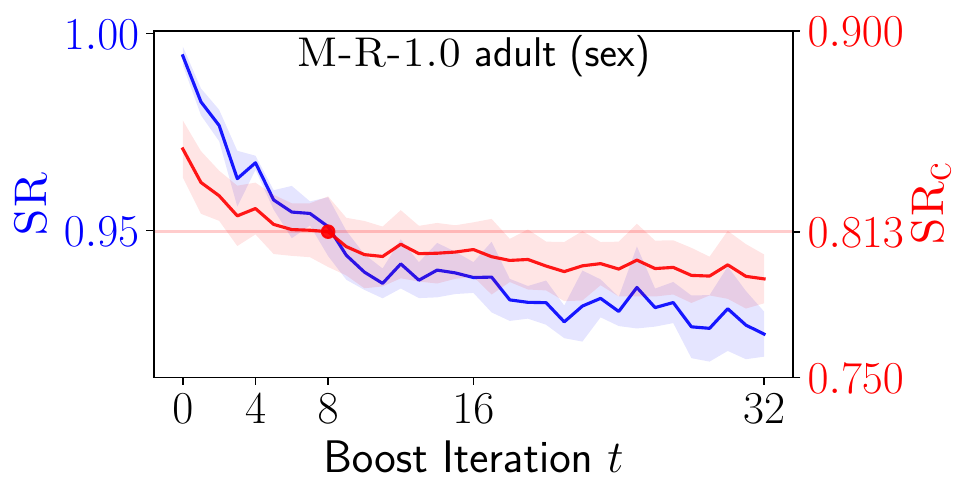}%
    \includegraphics[width=0.24\columnwidth]{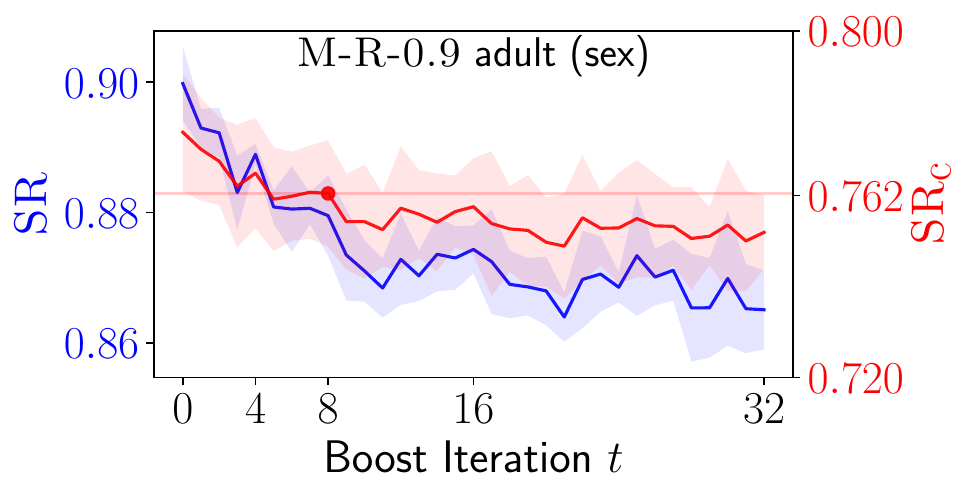}%
    \\
    \includegraphics[width=0.24\columnwidth]{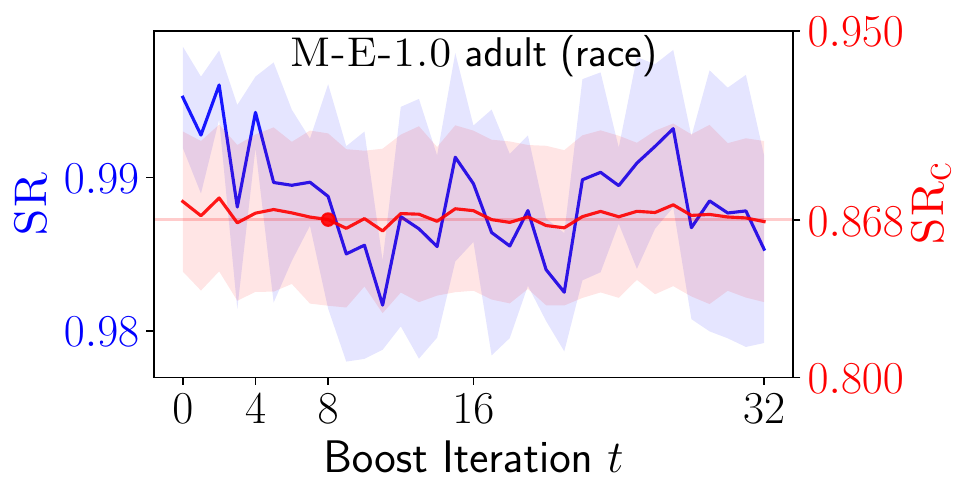}%
    \includegraphics[width=0.24\columnwidth]{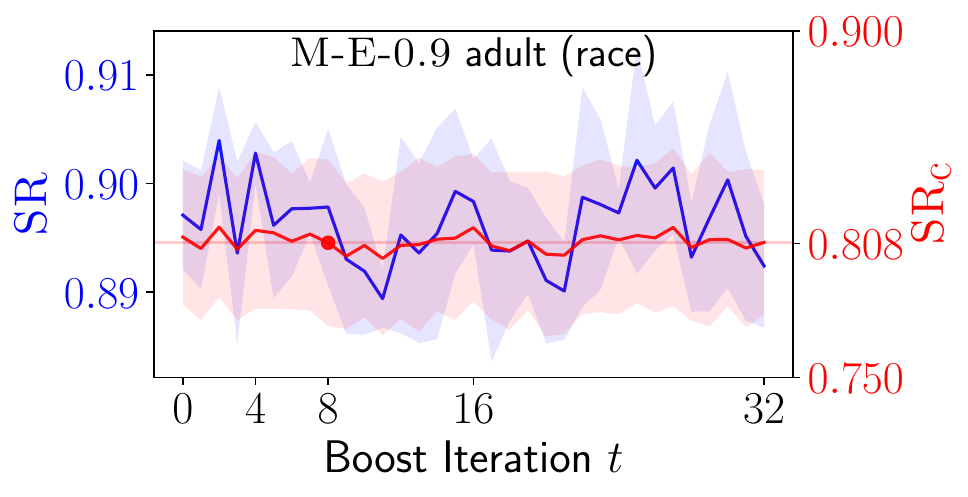}%
    \includegraphics[width=0.24\columnwidth]{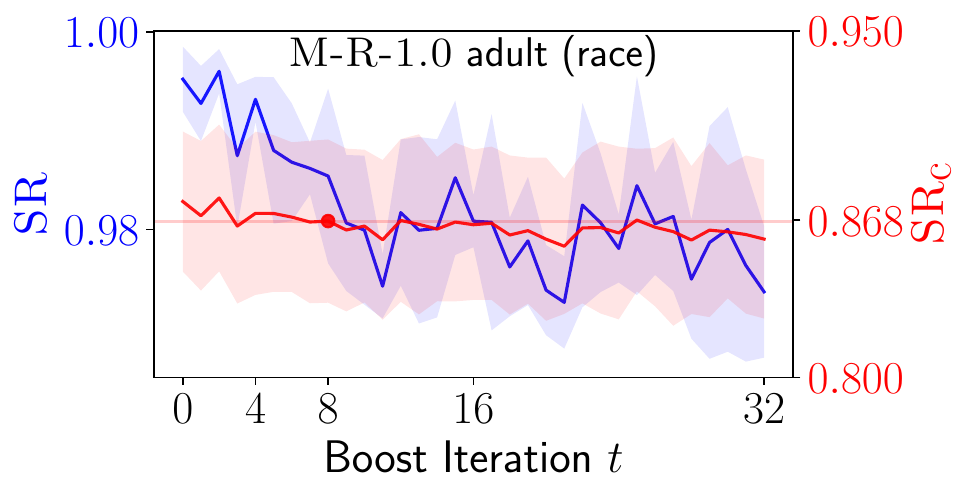}%
    \includegraphics[width=0.24\columnwidth]{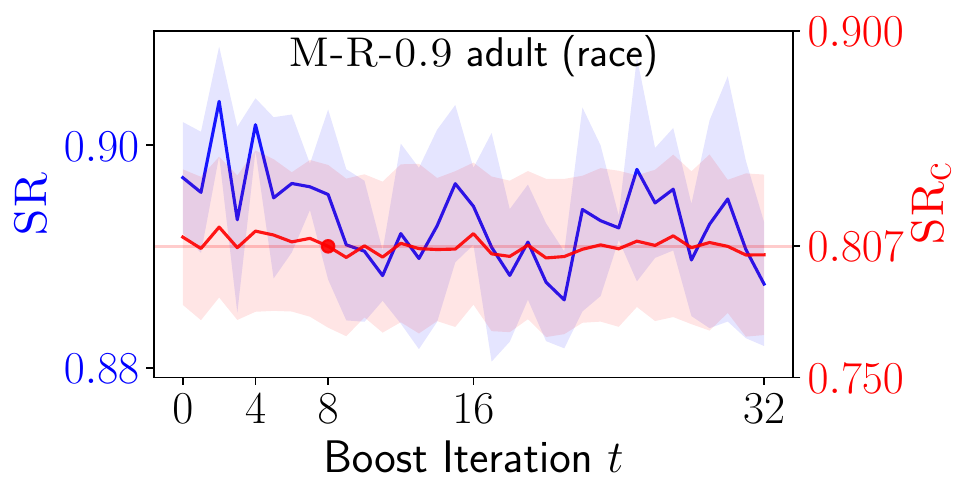}%
    \caption{All SR vs Classification SR over boosting iterations plots for \compas and \adult. Horizontal line depicts the \( T = 8 \) \( \SR_{\textrm{c}}\) value.}
    \label{fig:srclf_all_plots}
\end{sidewaysfigure}

\begin{sidewaysfigure}
    \centering
    \includegraphics[width=\columnwidth]{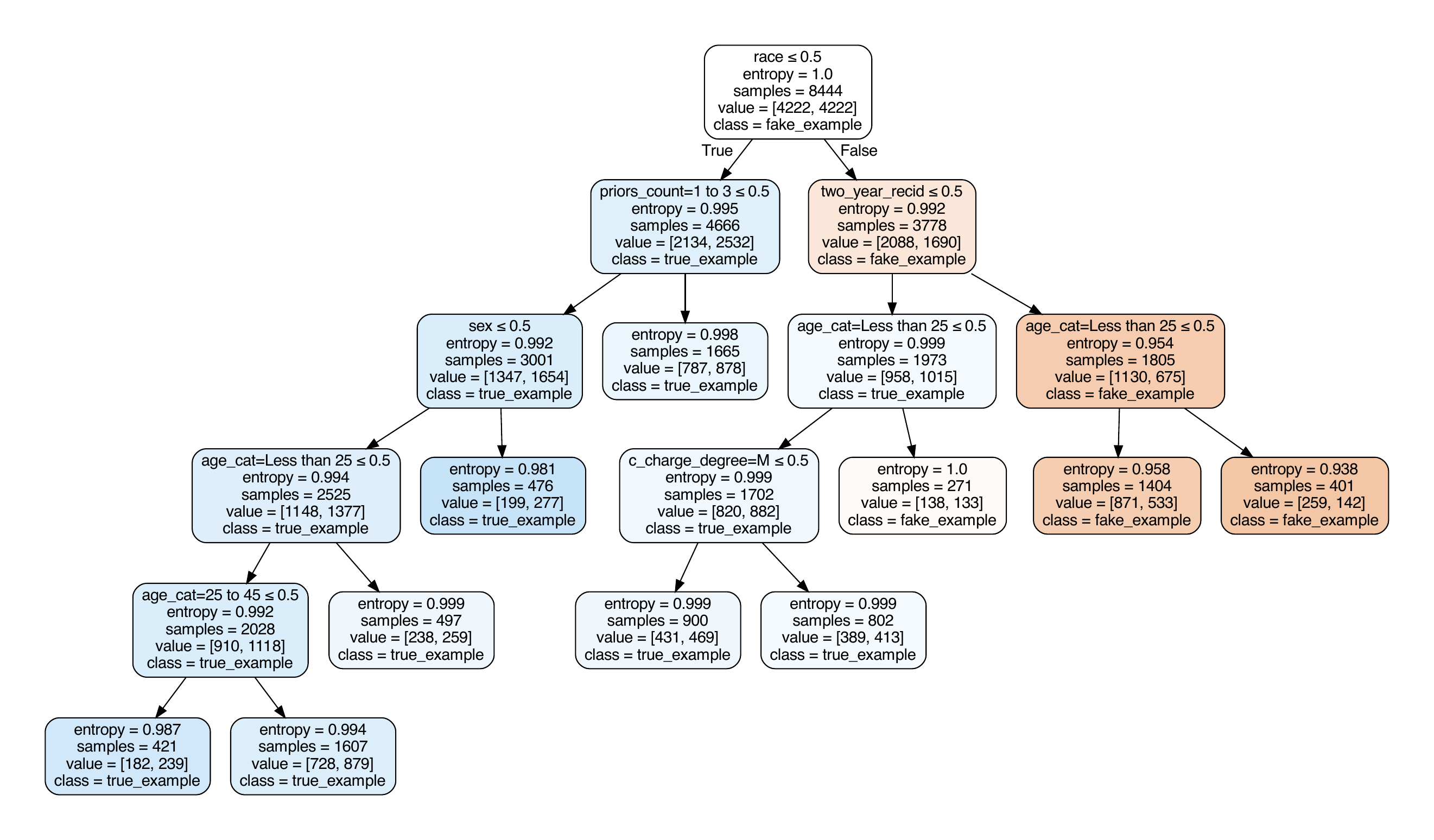}
    \caption{Full WL example of \cref{fig:wl_teaser}: \mollifiera with \compas (\( \mathcal{S} = \textrm{race}\)).}
    \label{fig:full_wl}
\end{sidewaysfigure}

%\begin{figure}
%    \centering
%    \includegraphics[width=\columnwidth]{figures/full_dt_example.pdf}
%    \caption{Full WL example of \cref{fig:wl_teaser}: \mollifiera with \compas (\( \mathcal{S} = \textrm{race}\)).}
%    \label{fig:full_wl}
%\end{figure}
\section{Additional Continuous Experiments}
\label{sec:additional_continuous_experiments}

We present additional experiments and figures for the \textsc{Minneapolis} dataset. We note that all computation was done without a GPU. In addition to \mollifiera explored in the main-text, we present results for \mollifierb, \mollifierc, and \mollifierd similar to the discrete dataset experiments. In \cref{tab:summary_continuous_table}, we present a summary of evaluations used in the discrete experiments. In \cref{fig:cts_all_plots} we present RR vs KL, SR vs KL, and SR vs \( \SR_{\textrm{c}} \) plots (one each row), as per \cref{fig:boost_iters_joint}.

\begin{table*}[t]
    \caption{\fbde~(\(T = 16\)) evaluated for the \textsc{Minneapolis} dataset. The table reports the mean and s.t.d.}
    \label{tab:summary_continuous_table}
    \begin{center}
    \begin{small}
    \begin{sc}
    \begin{tabularx}{\textwidth}{LLPDDDDD}
        \toprule
        && & Data & \mollifiera & \mollifierb & \mollifierc & \mollifierd \\
        \midrule

{\multirow{10}{*}{\rotatebox[origin=c]{90}{\textsc{Minneapolis} (\(\mathcal{S}\) = race)}}} &
{\multirow{3}{*}{\rotatebox[origin=c]{90}{\underline{data}}}} 
& $ \textrm{RR} $                 &  \(.183 \pm .001\) & \(.846 \pm .004\) & \(.918 \pm .000\) & \(.756 \pm .003\) & \(.860 \pm .004\) \\
&& $ \textrm{SR} $                &  \(.684 \pm .005\) & \(.995 \pm .003\) & \(.901 \pm .000\) & \(.991 \pm .004\) & \(.902 \pm .004\) \\
&& $ \textrm{KL} $                &  \( - \) & \(.218 \pm .005\) & \(.236 \pm .000\) & \(.186 \pm .004\) & \(.219 \pm .005\) \\
        \cmidrule(l{5pt}){2-8}

&{\multirow{3}{*}{\rotatebox[origin=c]{90}{\underline{pred}}}} 
& $ \textrm{SR}_{\textrm{c}}$     &  \(.685 \pm .012\) & \(.930 \pm .018\) & \(.754 \pm .000\) & \(.929 \pm .038\) & \(.835 \pm .038\) \\
&& $ \textrm{EO} $                &  \(.778 \pm .028\) & \(.974 \pm .018\) & \(.856 \pm .000\) & \(.943 \pm .031\) & \(.929 \pm .047\) \\
&& $ \textrm{Acc} $               &  \(.649 \pm .004\) & \(.594 \pm .016\) & \(.625 \pm .000\) & \(.600 \pm .015\) & \(.601 \pm .015\) \\
        \cmidrule(l{5pt}){2-8}
        
&{\multirow{3}{*}{\rotatebox[origin=c]{90}{\underline{clus}}}} 
& $ \cPR $                        &  \(.134 \pm .013\) & \(.137 \pm .011\) & \(.155 \pm .000\) & \(.138 \pm .012\) & \(.153 \pm .023\) \\
&& $ \cSR $                       &  \(.282 \pm .072\) & \(.235 \pm .033\) & \(.310 \pm .000\) & \(.443 \pm .057\) & \(.262 \pm .044\) \\
&& $ \textrm{dist} $              &  \(.013 \pm .000\) & \(.016 \pm .000\) & \(.016 \pm .000\) & \(.016 \pm .000\) & \(.016 \pm .000\) \\

        \cmidrule(l{5pt}){2-8}
& \multicolumn{2}{l}{Time(s)} & \( - \) & \(1540.731 \pm 2.615\) & \(1542.730 \pm .000\) & \(1369.129 \pm 55.183\) & \(1410.402 \pm 45.461\) \\
        \bottomrule
    \end{tabularx}
    \end{sc}
    \end{small}
    \end{center}
\end{table*}

\begin{sidewaysfigure}[t]
    \centering
    \includegraphics[width=0.24\columnwidth]{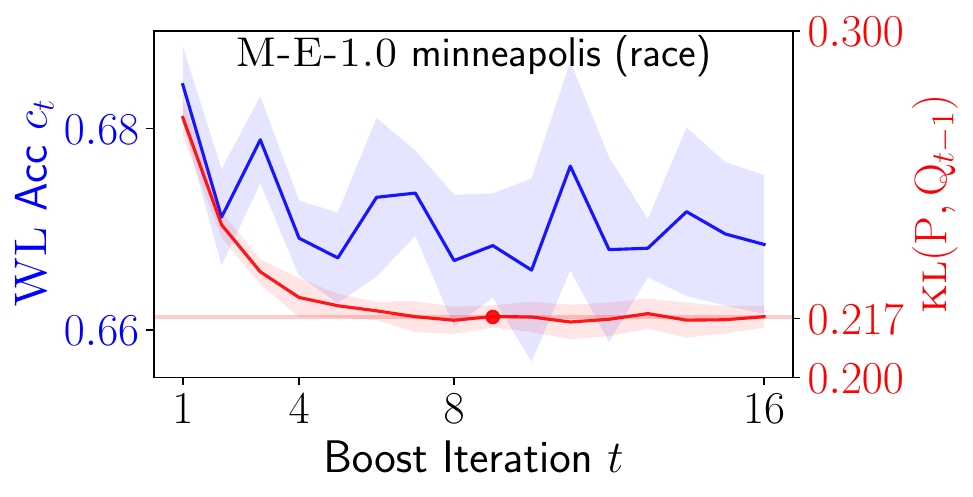}%
    \includegraphics[width=0.24\columnwidth]{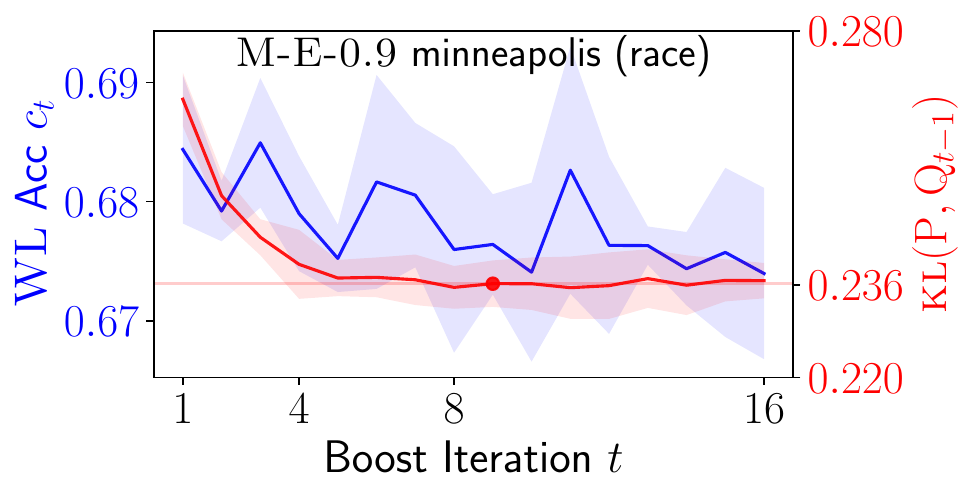}%
    \includegraphics[width=0.24\columnwidth]{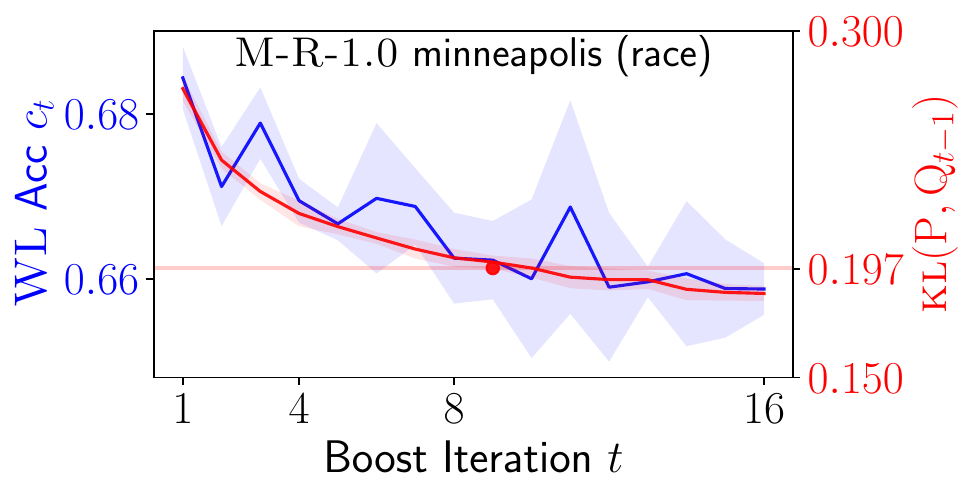}%
    \includegraphics[width=0.24\columnwidth]{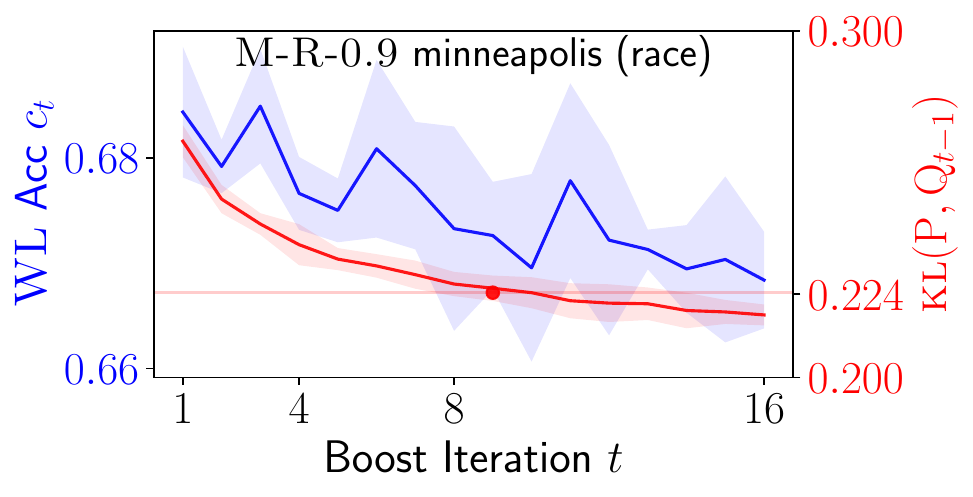}%
    \\
    \includegraphics[width=0.24\columnwidth]{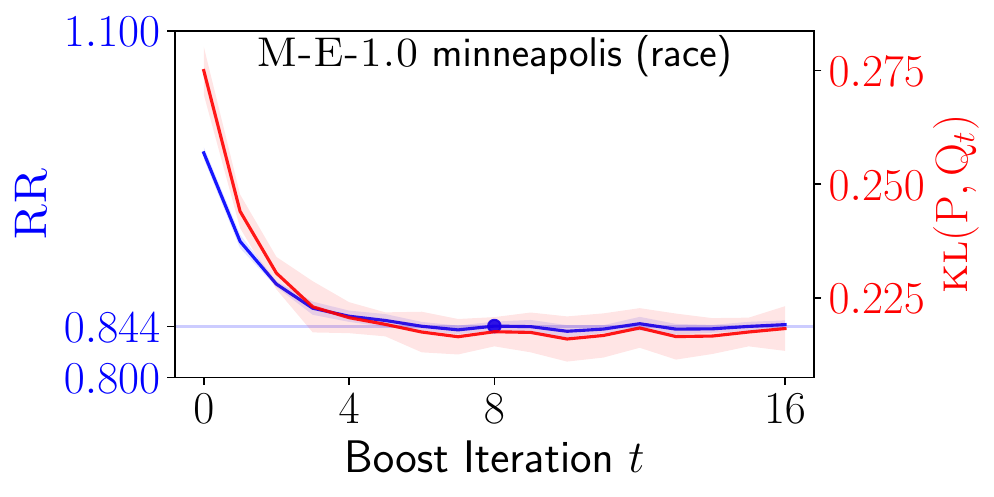}%
    \includegraphics[width=0.24\columnwidth]{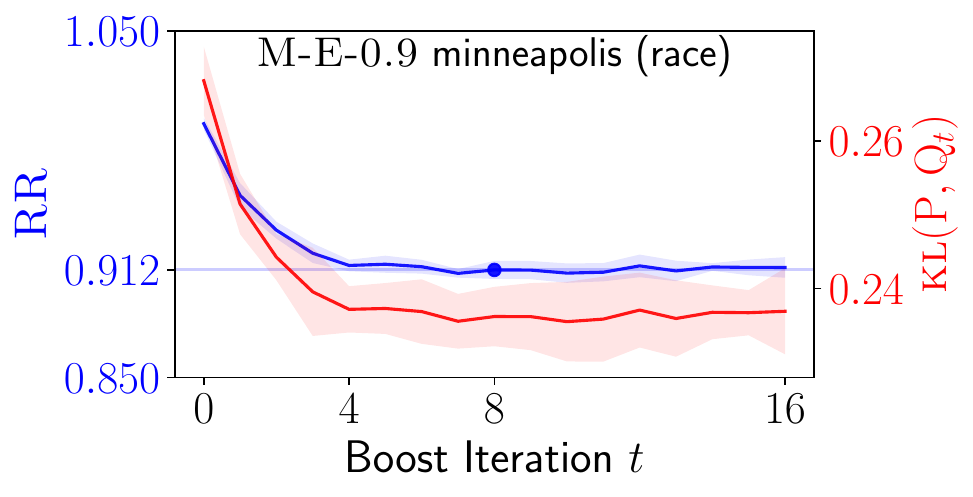}%
    \includegraphics[width=0.24\columnwidth]{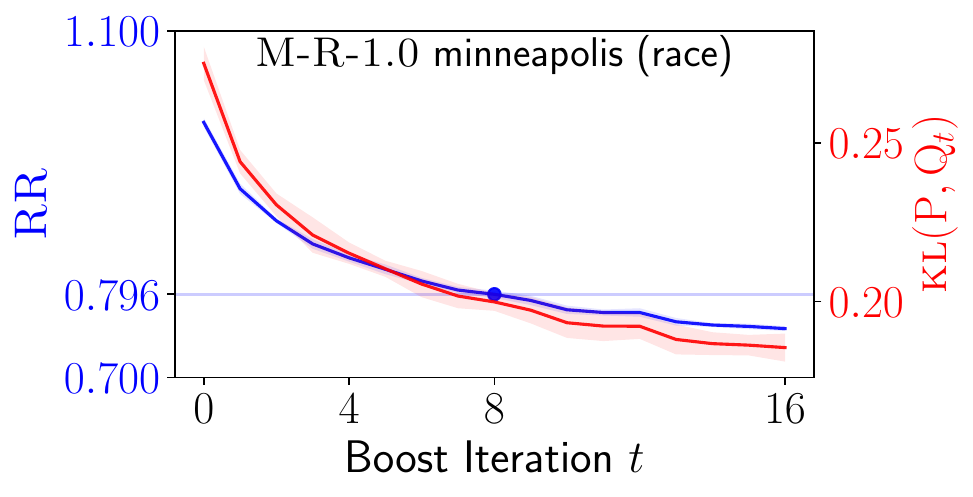}%
    \includegraphics[width=0.24\columnwidth]{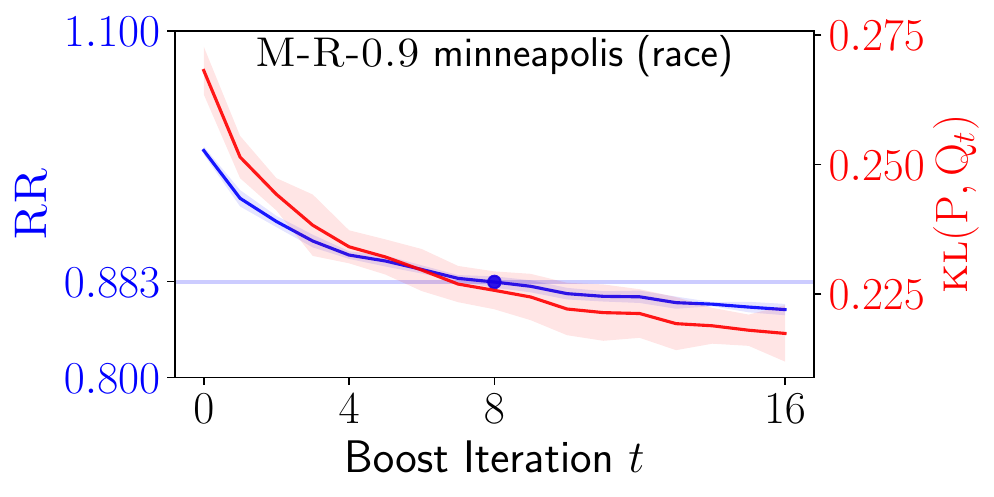}%
    \\
    \includegraphics[width=0.24\columnwidth]{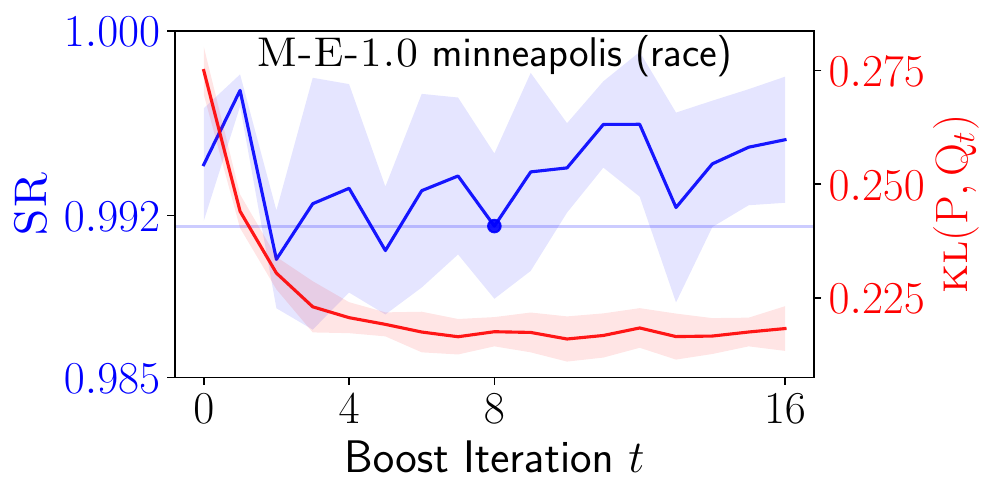}%
    \includegraphics[width=0.24\columnwidth]{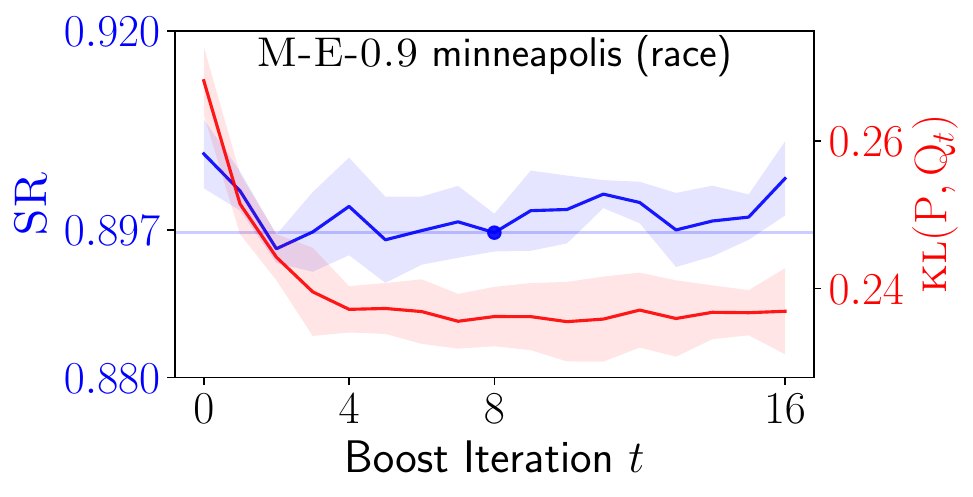}%
    \includegraphics[width=0.24\columnwidth]{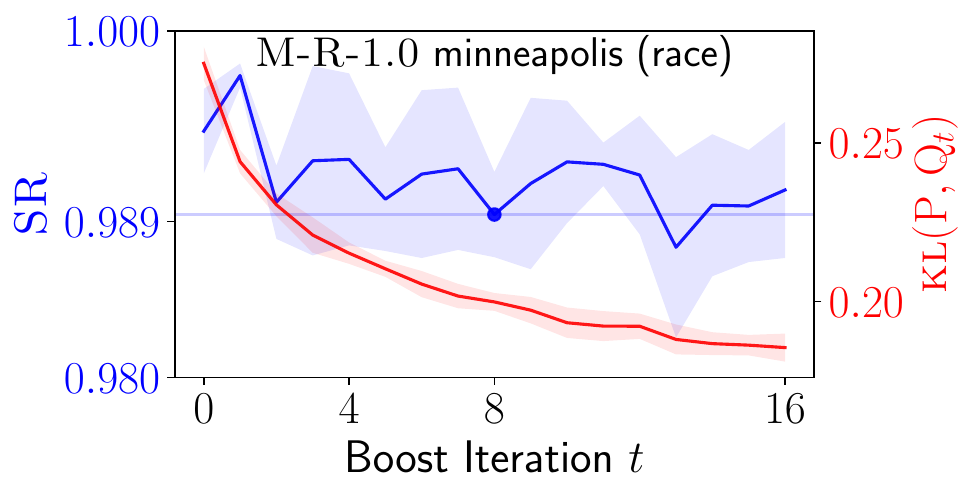}%
    \includegraphics[width=0.24\columnwidth]{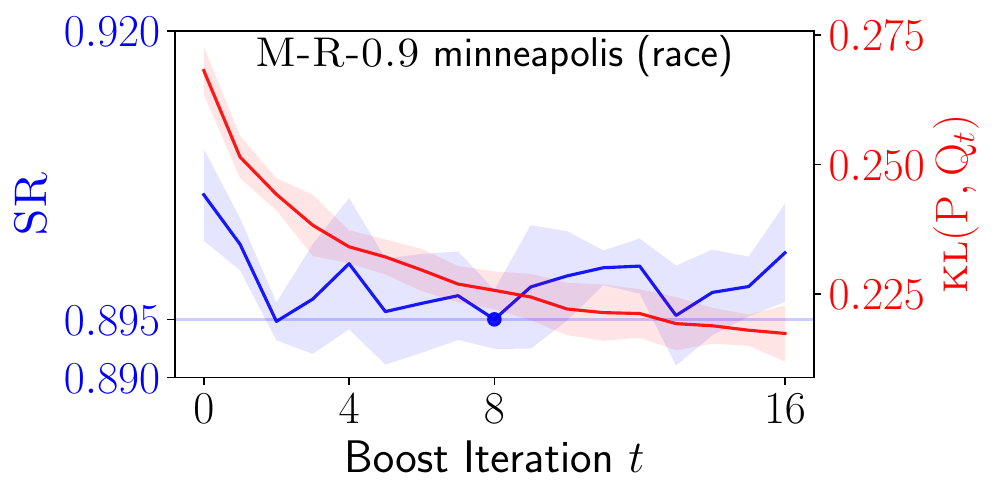}%
    \\
    \includegraphics[width=0.24\columnwidth]{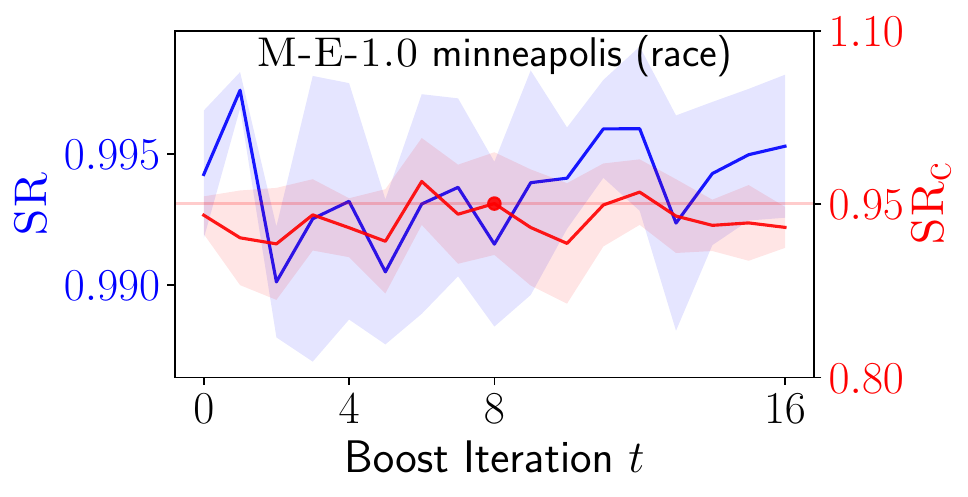}%
    \includegraphics[width=0.24\columnwidth]{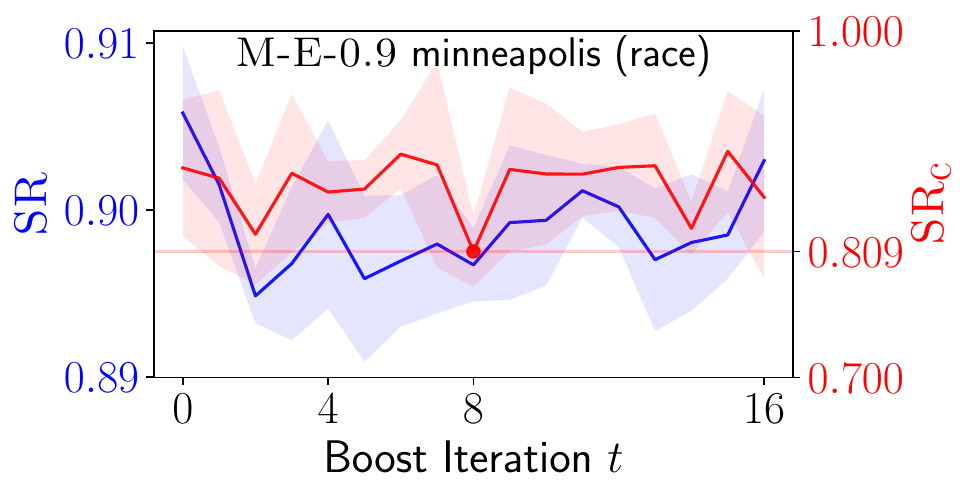}%
    \includegraphics[width=0.24\columnwidth]{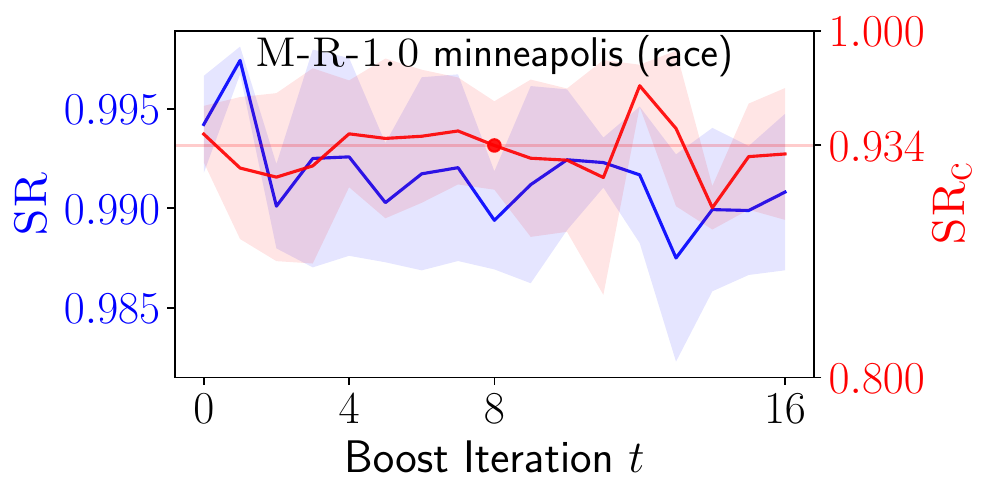}%
    \includegraphics[width=0.24\columnwidth]{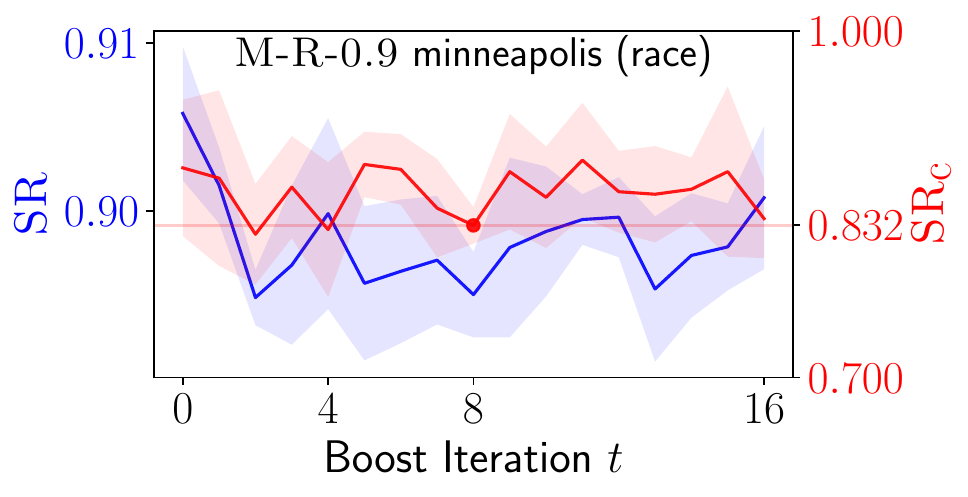}%
    \caption{All boosting iterations plots for \textsc{Minneapolis}. Horizontal line depicts the  \( \kl(\meas{P}, \meas{Q}_{8})\) value.}
    \label{fig:cts_all_plots}
\end{sidewaysfigure}

\section{Extra Discrete Experiments}
\label{sec:extra_discrete_experiments}

We present the exact same plots and tables for \dutch and \german as we did for \dutch and \german. All settings for \fbde are identical except for the initial distribution utilized. Here we incorporate a ``prior'' into the initial distribution, similar to \citet{ckvDP}. An extended discussion of the trick we use is presented in \cref{sec:mixing_prior}.
We plot weak learner accuracy vs KL, \( \SR \) vs KL, \( \RR \) vs KL, and \( \SR \) vs \( \SR_{\textrm{c}} \) over boosting iterations in \cref{fig:wl_all_plots_extra,fig:rr_all_plots_extra,fig:sr_all_plots_extra,fig:srclf_all_plots_extra}, respectively.

We would like to note one particularly poor showing in \cref{tab:no_mixing_experiments_table}: \mollifierd on \dutch with \( \mathcal{S} = \textrm{sex} \). In this case, \mollifierd has a critical failure in \( \SRclf \). However, one should note that in \cref{fig:srclf_all_plots_extra}, there actually include boosting iterations which \mollifierd significantly performs better in \( \SRclf \). Despite this, we should still note that the overall average \( \SRclf \) is still quite poor. This is similar to \mollifierb in this dataset setting. This indicates that having a weaker (in terms of fairness) initial distribution might not work for sparse datasets. Nevertheless, one wants to make sure that \fbde is performing correctly in these downstream fairness considerations by examining performance per boosting iterations.

\begin{table*}[t]
    \caption{\fbde~(\(T = 32\)) and baselines evaluated for \dutch and \german datasets. The table reports the mean and s.t.d.}
    \label{tab:full_experiments_table_extra}
    \scriptsize
    \begin{center}
    %\begin{small}
    \begin{sc}
    \begin{tabularx}{\textwidth}{LLPDDDDDDDD} %{LLMXXXXXXXX}
        \toprule
        && & Data & \mollifiera & \mollifierb & \mollifierc & \mollifierd & \celisa & \tabfairgan & \fairkmeans \\
        \midrule

{\multirow{9}{*}{\rotatebox[origin=c]{90}{\dutch (\(\mathcal{S}\) = sex)}}} &
{\multirow{3}{*}{\rotatebox[origin=c]{90}{\underline{data}}}} 
& $ \textrm{RR} $                 & 
\(.996 \pm .002\) & \(.987 \pm .001\) & \(.993 \pm .000\) & \(.980 \pm .001\) & \(.992 \pm .000\) & \(.997 \pm .002\) & \(.411 \pm .260\) & \(-\) \\
&& $ \textrm{SR} $                & 
\(.523 \pm .002\) & \(.988 \pm .001\) & \(.894 \pm .000\) & \(.971 \pm .001\) & \(.887 \pm .000\) & \(.961 \pm .004\) & \(.344 \pm .338\) & \(-\) \\
&& $ \textrm{KL} $                & 
\(-\) & \(.939 \pm .004\) & \(.916 \pm .004\) & \(.920 \pm .004\) & \(.906 \pm .004\) & \(1.637 \pm .017\) & \(4.367 \pm .668\) & \(-\) \\
        \cmidrule(l{5pt}){2-11}
&{\multirow{3}{*}{\rotatebox[origin=c]{90}{\underline{pred}}}} 
& $ \textrm{SR}_{\textrm{c}}$     & 
\(.464 \pm .022\) & \(.696 \pm .193\) & \(.518 \pm .255\) & \(.628 \pm .139\) & \(.275 \pm .068\) & \(.496 \pm .020\) & \(.628 \pm .223\) & \(-\) \\
&& $ \textrm{EO} $                & 
\(.890 \pm .006\) & \(.973 \pm .005\) & \(.986 \pm .008\) & \(.980 \pm .007\) & \(.987 \pm .007\) & \(.967 \pm .007\) & \(.585 \pm .203\) & \(-\) \\
&& $ \textrm{Acc} $               & 
\(.827 \pm .001\) & \(.763 \pm .003\) & \(.778 \pm .004\) & \(.770 \pm .003\) & \(.781 \pm .001\) & \(.814 \pm .003\) & \(.667 \pm .068\) & \(-\) \\
        \cmidrule(l{5pt}){2-11}
        
&{\multirow{3}{*}{\rotatebox[origin=c]{90}{\underline{clus}}}} 
 & $ \cPR $           & 
\(.080 \pm .029\) & \(.070 \pm .021\) & \(.136 \pm .053\) & \(.113 \pm .017\) & \(.069 \pm .023\) & \(.075 \pm .006\) & \(.094 \pm .056\) & \(.203 \pm .103\) \\

&& $ \cSR $           & 
\(.341 \pm .126\) & \(.273 \pm .066\) & \(.254 \pm .053\) & \(.449 \pm .052\) & \(.254 \pm .058\) & \(.204 \pm .038\) & \(.291 \pm .121\) & \(.428 \pm .195\) \\
&& $ \textrm{dist} $  & 
\(.051 \pm .000\) & \(.056 \pm .000\) & \(.056 \pm .000\) & \(.056 \pm .000\) & \(.056 \pm .000\) & \(.056 \pm .000\) & \(.056 \pm .000\) & \(.053 \pm .000\) \\

\midrule

{\multirow{9}{*}{\rotatebox[origin=c]{90}{\dutch (\(\mathcal{S}\) = age)}}} &
{\multirow{3}{*}{\rotatebox[origin=c]{90}{\underline{data}}}} 
& $ \textrm{RR} $                 &  
\(.621 \pm .001\) & \(.980 \pm .001\) & \(.990 \pm .000\) & \(.964 \pm .001\) & \(.980 \pm .000\) & \(.996 \pm .004\) & \(.227 \pm .026\) & \(-\) \\
&& $ \textrm{SR} $                &  
\(.610 \pm .003\) & \(.993 \pm .001\) & \(.895 \pm .000\) & \(.984 \pm .001\) & \(.893 \pm .000\) & \(.963 \pm .005\) & \(.430 \pm .011\) & \(-\) \\
&& $ \textrm{KL} $                &  
\(-\) & \(.935 \pm .004\) & \(.929 \pm .005\) & \(.916 \pm .004\) & \(.919 \pm .005\) & \(1.656 \pm .015\) & \(4.132 \pm .165\) & \(-\) \\
        \cmidrule(l{5pt}){2-11}

&{\multirow{3}{*}{\rotatebox[origin=c]{90}{\underline{pred}}}} 
& $ \textrm{SR}_{\textrm{c}}$     &  
\(.695 \pm .030\) & \(.816 \pm .094\) & \(.676 \pm .257\) & \(.747 \pm .084\) & \(.709 \pm .180\) & \(.722 \pm .023\) & \(.271 \pm .046\) & \(-\) \\
&& $ \textrm{EO} $                &  
\(.892 \pm .009\) & \(.993 \pm .005\) & \(.960 \pm .007\) & \(.987 \pm .007\) & \(.961 \pm .009\) & \(.991 \pm .003\) & \(.742 \pm .018\) & \(-\) \\
&& $ \textrm{Acc} $               &  
\(.827 \pm .001\) & \(.781 \pm .004\) & \(.788 \pm .004\) & \(.785 \pm .005\) & \(.790 \pm .005\) & \(.820 \pm .003\) & \(.661 \pm .004\) & \(-\) \\
        \cmidrule(l{5pt}){2-11}
        
&{\multirow{3}{*}{\rotatebox[origin=c]{90}{\underline{clus}}}} 
& $ \cPR $                        &  
\(.699 \pm .080\) & \(.454 \pm .212\) & \(.660 \pm .096\) & \(.675 \pm .070\) & \(.718 \pm .056\) & \(.440 \pm .175\) & \(.436 \pm .190\) & \(.373 \pm .051\) \\
&& $ \cSR $                       &  
\(.525 \pm .079\) & \(.421 \pm .138\) & \(.489 \pm .103\) & \(.492 \pm .098\) & \(.542 \pm .036\) & \(.340 \pm .058\) & \(.376 \pm .068\) & \(.443 \pm .126\) \\
&& $ \textrm{dist} $              &  
\(.052 \pm .000\) & \(.056 \pm .000\) & \(.056 \pm .000\) & \(.056 \pm .000\) & \(.056 \pm .000\) & \(.056 \pm .000\) & \(.056 \pm .000\) & \(.053 \pm .000\) \\

\midrule

{\multirow{9}{*}{\rotatebox[origin=c]{90}{\german (\(\mathcal{S}\) = sex)}}} &
{\multirow{3}{*}{\rotatebox[origin=c]{90}{\underline{data}}}} 
& $ \textrm{RR} $                 & 
\(.449 \pm .014\) & \(.934 \pm .033\) & \(.960 \pm .027\) & \(.903 \pm .038\) & \(.946 \pm .037\) & \(.980 \pm .013\) & \(.575 \pm .154\) & \( - \) \\
&& $ \textrm{SR} $                & 
\(.897 \pm .018\) & \(.982 \pm .018\) & \(.917 \pm .044\) & \(.983 \pm .017\) & \(.916 \pm .044\) & \(.985 \pm .011\) & \(.985 \pm .010\) & \( - \) \\
&& $ \textrm{KL} $                & 
\(-\) & \(.606 \pm .027\) & \(.612 \pm .026\) & \(.588 \pm .026\) & \(.602 \pm .025\) & \(.531 \pm .035\) & \(9.02 \pm 1.44\) & \( - \) \\
        \cmidrule(l{5pt}){2-11}

&{\multirow{3}{*}{\rotatebox[origin=c]{90}{\underline{pred}}}} 
& $ \textrm{SR}_{\textrm{c}}$     & 
\(.903 \pm .055\) & \(.971 \pm .014\) & \(.887 \pm .049\) & \(.970 \pm .019\) & \(.886 \pm .050\) & \(.951 \pm .030\) & \(.837 \pm .120\) & \( - \) \\
&& $ \textrm{EO} $                & 
\(.928 \pm .055\) & \(.980 \pm .021\) & \(.909 \pm .046\) & \(.969 \pm .023\) & \(.907 \pm .051\) & \(.954 \pm .031\) & \(.819 \pm .120\) & \( - \) \\
&& $ \textrm{Acc} $               & 
\(.690 \pm .022\) & \(.712 \pm .011\) & \(.705 \pm .011\) & \(.714 \pm .019\) & \(.706 \pm .007\) & \(.671 \pm .044\) & \(.509 \pm .170\) & \( - \) \\
        \cmidrule(l{5pt}){2-11}
        
&{\multirow{3}{*}{\rotatebox[origin=c]{90}{\underline{clus}}}} 
& $ \cPR $                        & 
\(.234 \pm .103\) & \(.229 \pm .144\) & \(.244 \pm .104\) & \(.197 \pm .082\) & \(.210 \pm .055\) & \(.180 \pm .075\) & \( - \) & \(.168 \pm .070\) \\
&& $ \cSR $                       & 
\(.268 \pm .100\) & \(.359 \pm .310\) & \(.327 \pm .094\) & \(.422 \pm .265\) & \(.393 \pm .298\) & \(.441 \pm .144\) & \(.396 \pm .315\) & \(.376 \pm .299\) \\
&& $ \textrm{dist} $              & 
\(.158 \pm .001\) & \(.197 \pm .001\) & \(.197 \pm .001\) & \(.197 \pm .000\) & \(.197 \pm .001\) & \(.195 \pm .001\) & \(.196 \pm .011\) & \(.157 \pm .001\) \\
        \midrule
        
{\multirow{9}{*}{\rotatebox[origin=c]{90}{\german (\(\mathcal{S}\) = age)}}} &
{\multirow{3}{*}{\rotatebox[origin=c]{90}{\underline{data}}}} 
& $ \textrm{RR} $                 & 
\(.235 \pm .008\) & \(.911 \pm .037\) & \(.951 \pm .036\) & \(.847 \pm .033\) & \(.915 \pm .033\) & \(.971 \pm .011\) & \(.419 \pm .083\) & \( - \) \\
&& $ \textrm{SR} $                & 
\(.794 \pm .029\) & \(.982 \pm .020\) & \(.911 \pm .019\) & \(.982 \pm .019\) & \(.908 \pm .022\) & \(.980 \pm .013\) & \(.968 \pm .014\) & \( - \) \\
&& $ \textrm{KL} $                & 
\(-\) & \(.717 \pm .035\) & \(.725 \pm .034\) & \(.687 \pm .034\) & \(.708 \pm .034\) & \(.691 \pm .043\) & \(9.812 \pm 1.357\) & \( - \) \\
        \cmidrule(l{5pt}){2-11}

&{\multirow{3}{*}{\rotatebox[origin=c]{90}{\underline{pred}}}} 
& $ \textrm{SR}_{\textrm{c}}$     & 
\(.802 \pm .067\) & \(.947 \pm .027\) & \(.881 \pm .045\) & \(.950 \pm .044\) & \(.890 \pm .048\) & \(.969 \pm .016\) & \(.780 \pm .212\) & \( - \) \\
&& $ \textrm{EO} $                & 
\(.823 \pm .078\) & \(.940 \pm .028\) & \(.894 \pm .059\) & \(.943 \pm .027\) & \(.914 \pm .058\) & \(.951 \pm .034\) & \(.744 \pm .193\) & \( - \) \\
&& $ \textrm{Acc} $               & 
\(.690 \pm .022\) & \(.700 \pm .016\) & \(.693 \pm .020\) & \(.697 \pm .013\) & \(.707 \pm .011\) & \(.671 \pm .023\) & \(.486 \pm .128\) & \( - \) \\
        \cmidrule(l{5pt}){2-11}
        
&{\multirow{3}{*}{\rotatebox[origin=c]{90}{\underline{clus}}}} 
& $ \cPR $            & 
\(.162 \pm .039\) & \(.183 \pm .042\) & \(.210 \pm .080\) & \(.183 \pm .097\) & \(.174 \pm .052\) & \(.200 \pm .060\) & \(.228 \pm .025\) & \(.208 \pm .102\) \\
&& $ \cSR $           & 
\(.387 \pm .333\) & \(.494 \pm .231\) & \(.784 \pm .250\) & \(.729 \pm .270\) & \(.665 \pm .441\) & \(.601 \pm .226\) & \(.499 \pm .298\) & \(.406 \pm .236\) \\
&& $ \textrm{dist} $  & 
\(.161 \pm .001\) & \(.198 \pm .001\) & \(.198 \pm .000\) & \(.198 \pm .001\) & \(.198 \pm .001\) & \(.196 \pm .001\) & \(.198 \pm .011\) & \(.158 \pm .001\) \\
        \bottomrule
    \end{tabularx}
    \end{sc}
    %\end{small}
    \end{center}
\end{table*}

\begin{sidewaysfigure}[t]
    \centering
    \includegraphics[width=0.24\columnwidth]{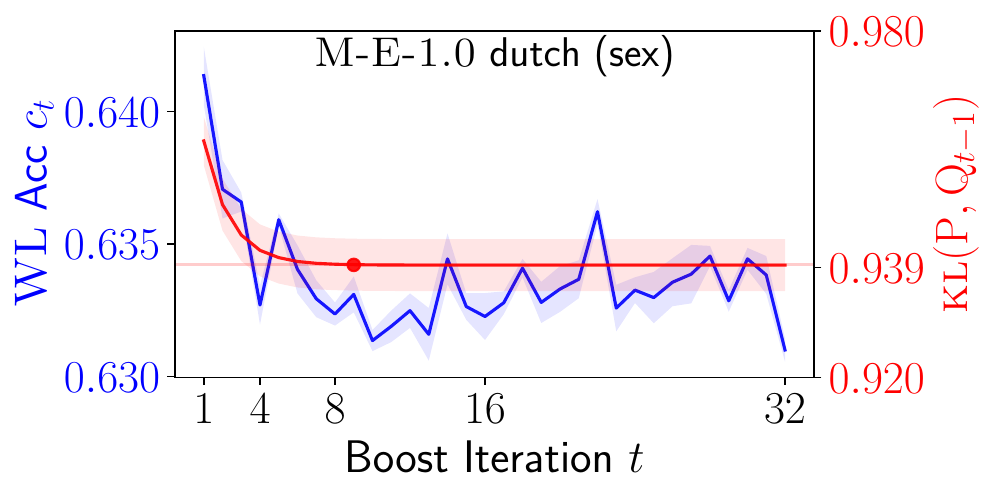}%
    \includegraphics[width=0.24\columnwidth]{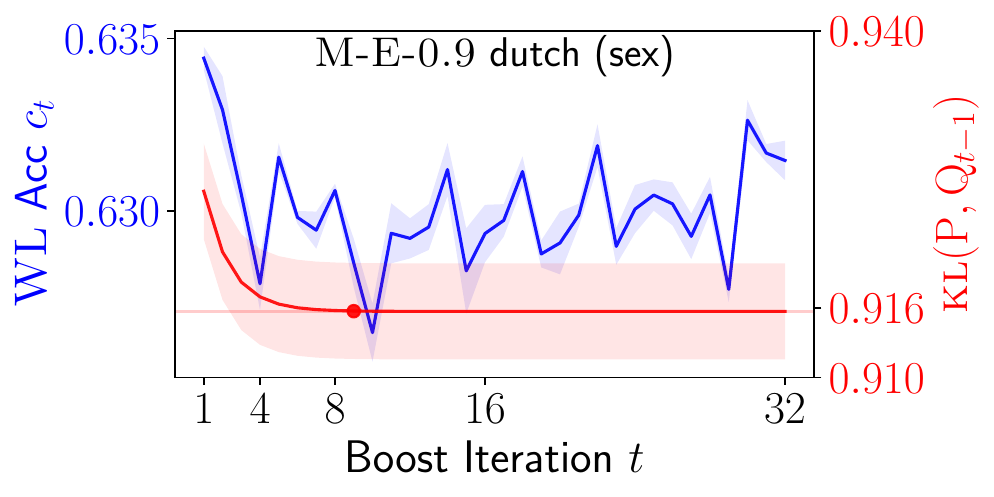}%
    \includegraphics[width=0.24\columnwidth]{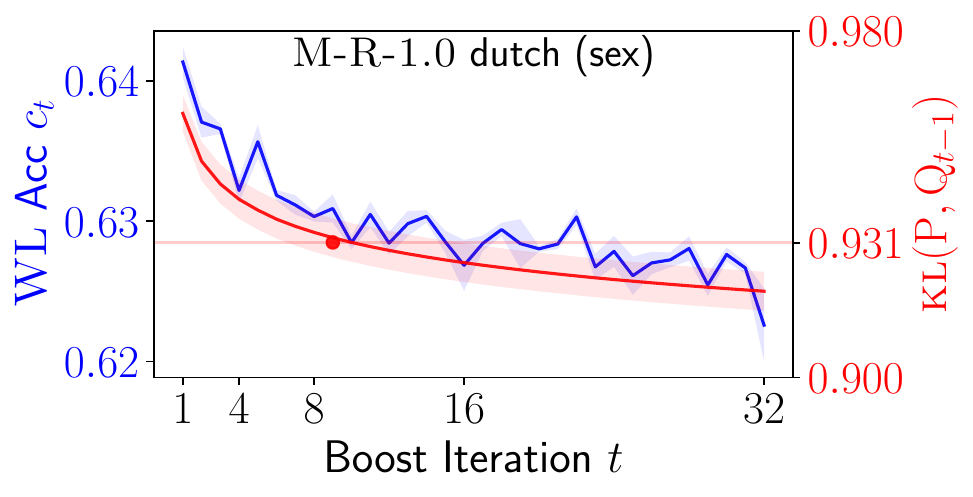}%
    \includegraphics[width=0.24\columnwidth]{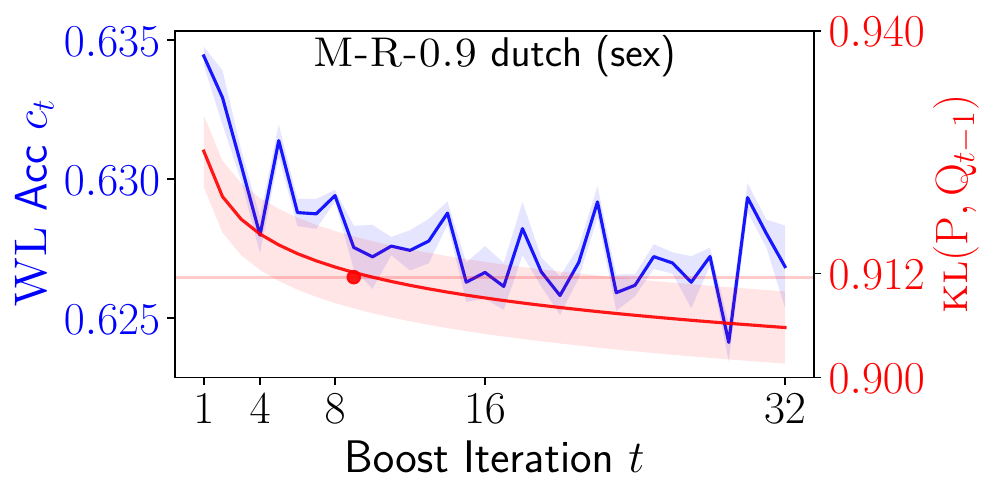}%
    \\
    \includegraphics[width=0.24\columnwidth]{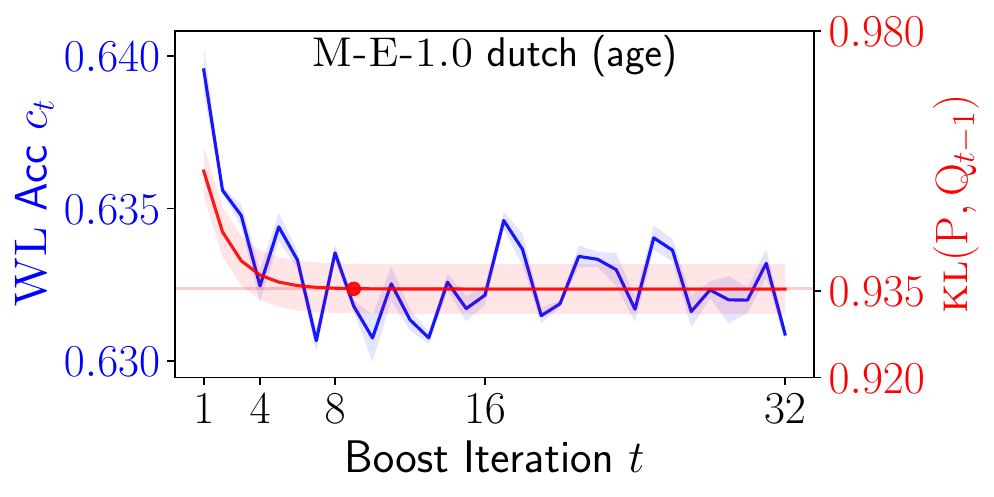}%
    \includegraphics[width=0.24\columnwidth]{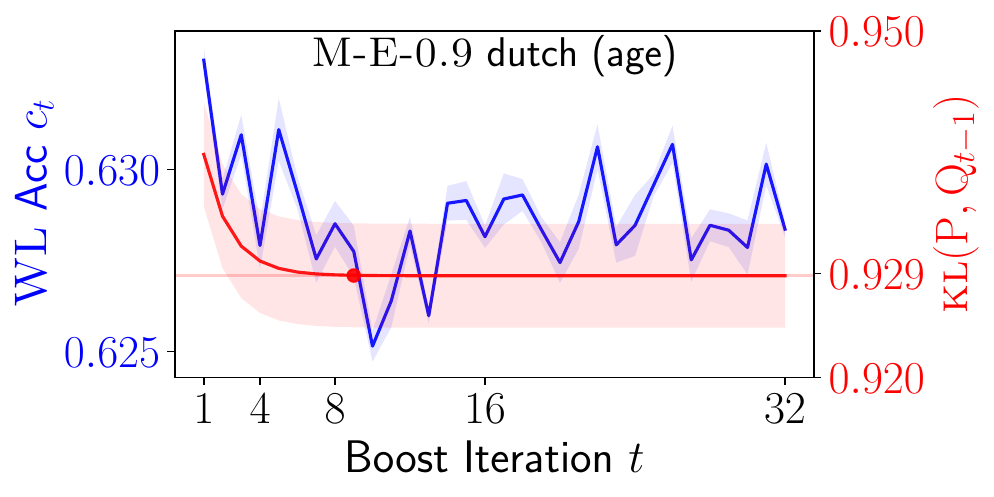}%
    \includegraphics[width=0.24\columnwidth]{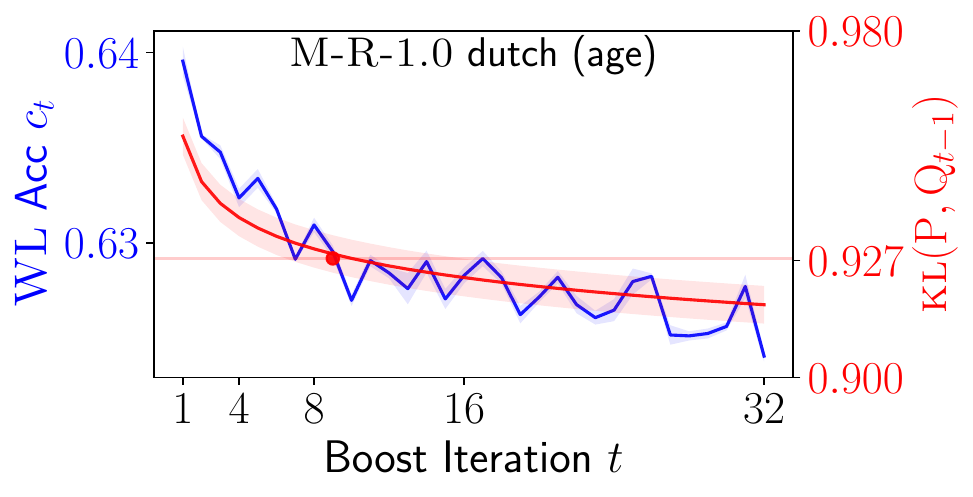}%
    \includegraphics[width=0.24\columnwidth]{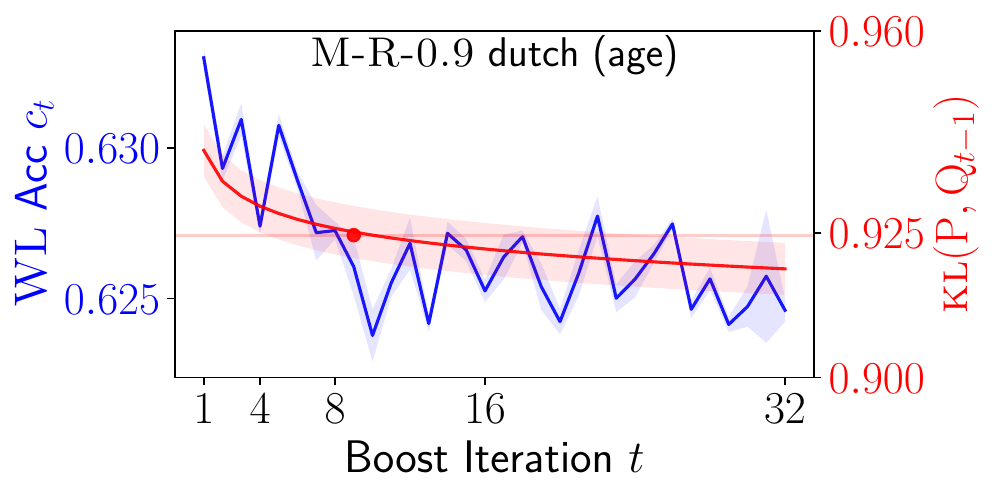}%
    \\
    \includegraphics[width=0.24\columnwidth]{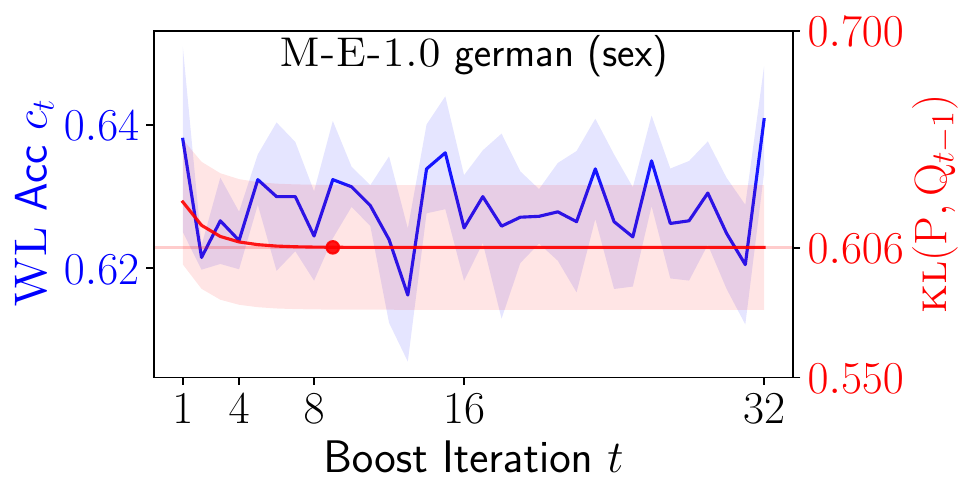}%
    \includegraphics[width=0.24\columnwidth]{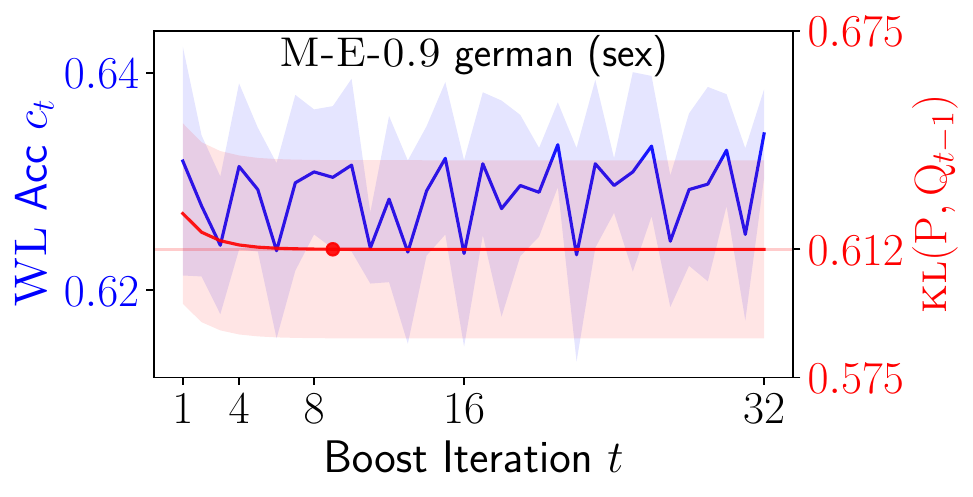}%
    \includegraphics[width=0.24\columnwidth]{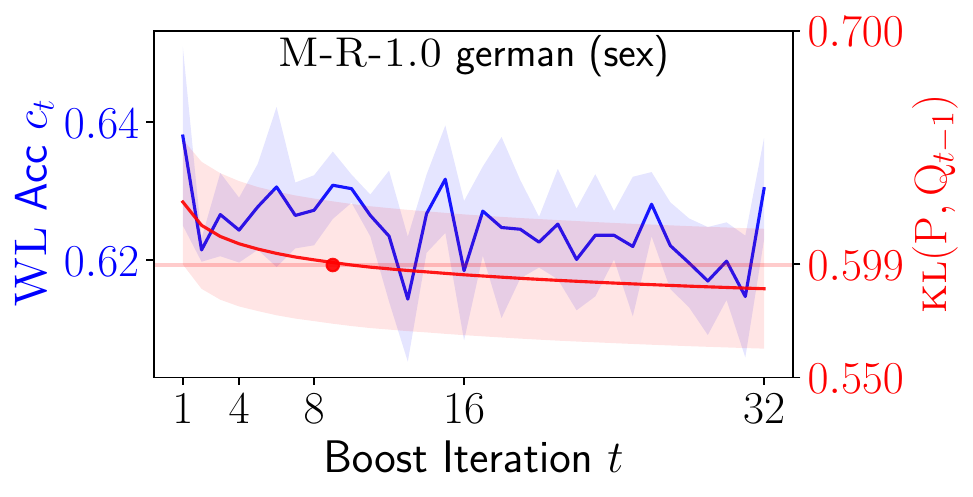}%
    \includegraphics[width=0.24\columnwidth]{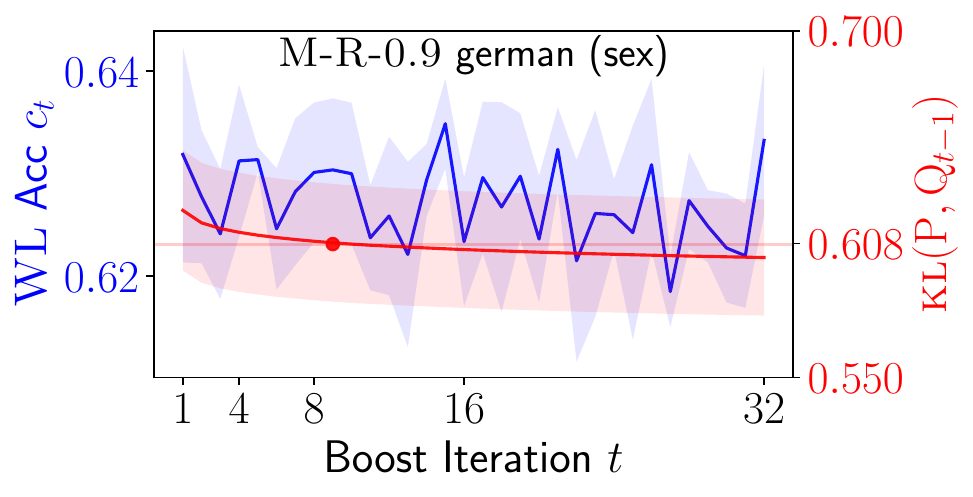}%
    \\
    \includegraphics[width=0.24\columnwidth]{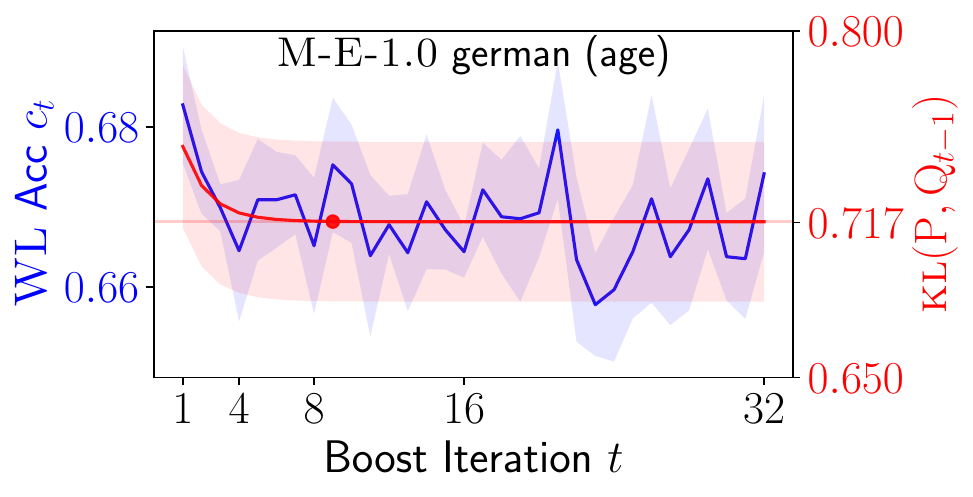}%
    \includegraphics[width=0.24\columnwidth]{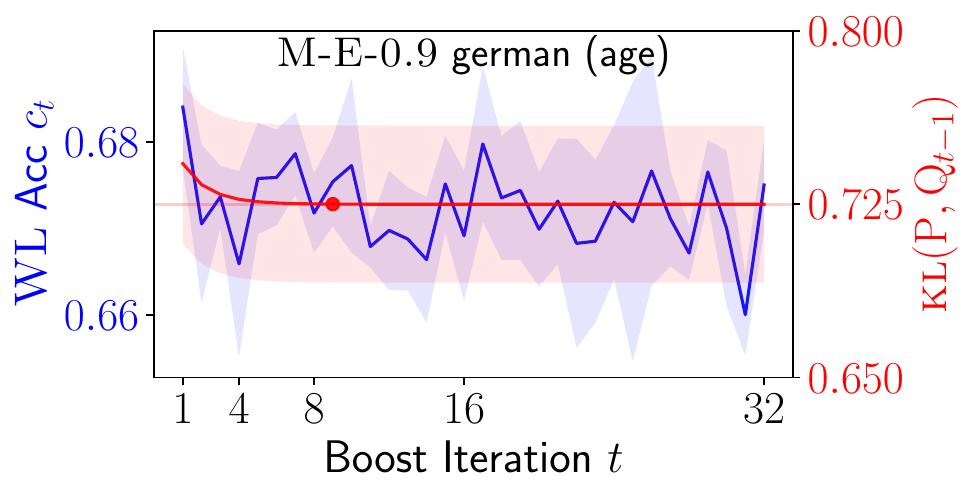}%
    \includegraphics[width=0.24\columnwidth]{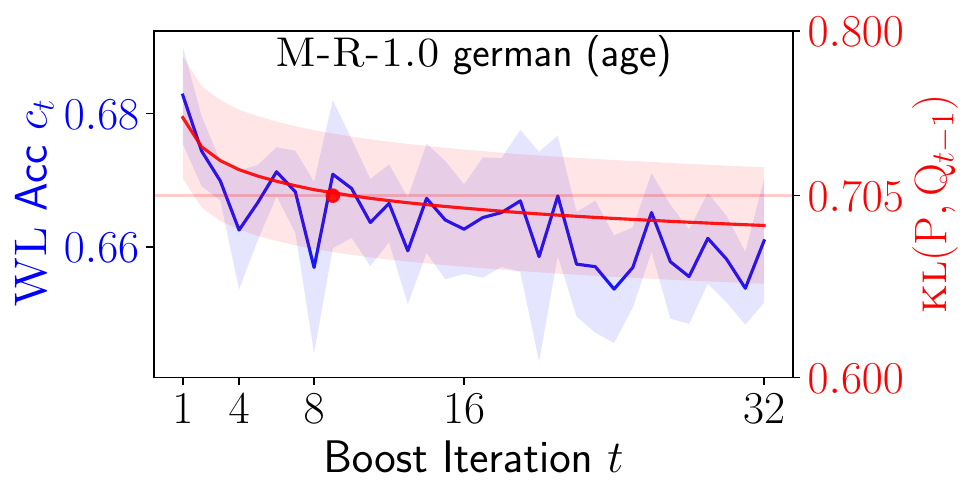}%
    \includegraphics[width=0.24\columnwidth]{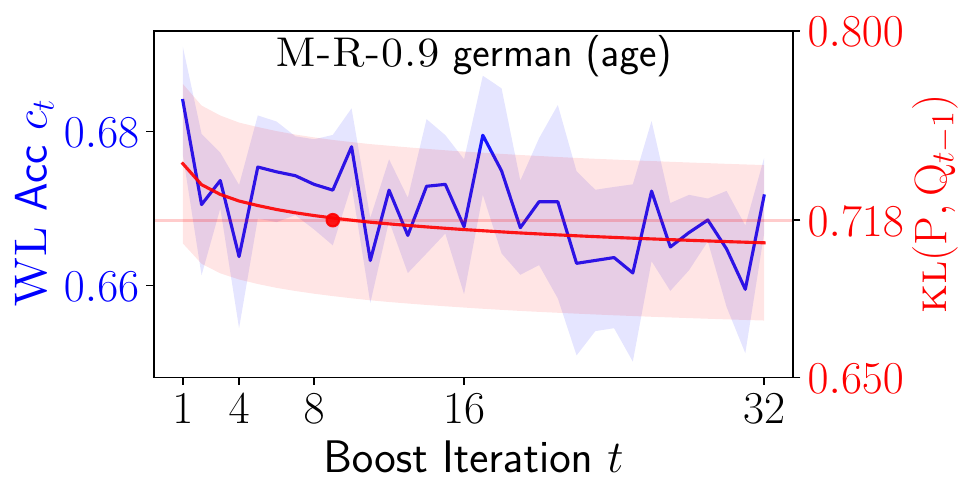}%
    \caption{All WL accuracy vs KL over boosting iterations plots for \dutch and \german. horizontal line depicts the  \( \kl(\meas{p}, \meas{q}_{8})\) value.}
    \label{fig:wl_all_plots_extra}
\end{sidewaysfigure}

\begin{sidewaysfigure}[t]
    \centering
    \includegraphics[width=0.24\columnwidth]{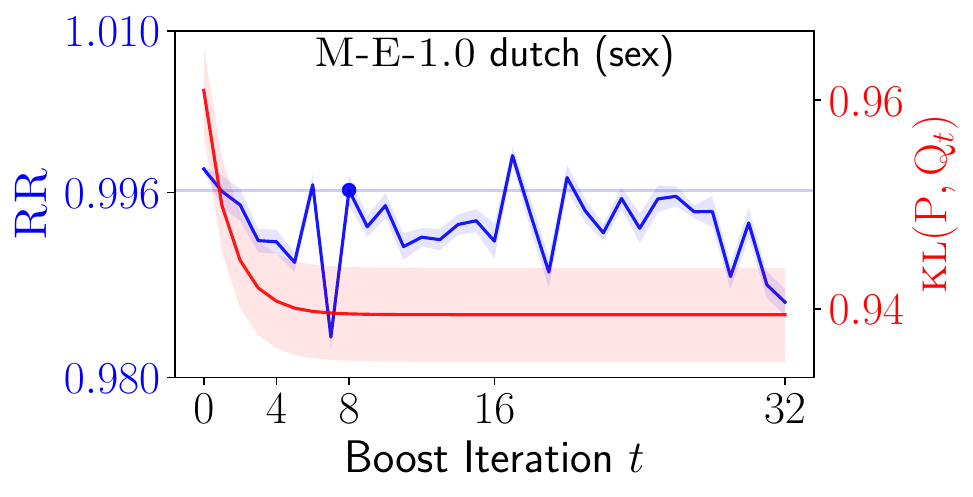}%
    \includegraphics[width=0.24\columnwidth]{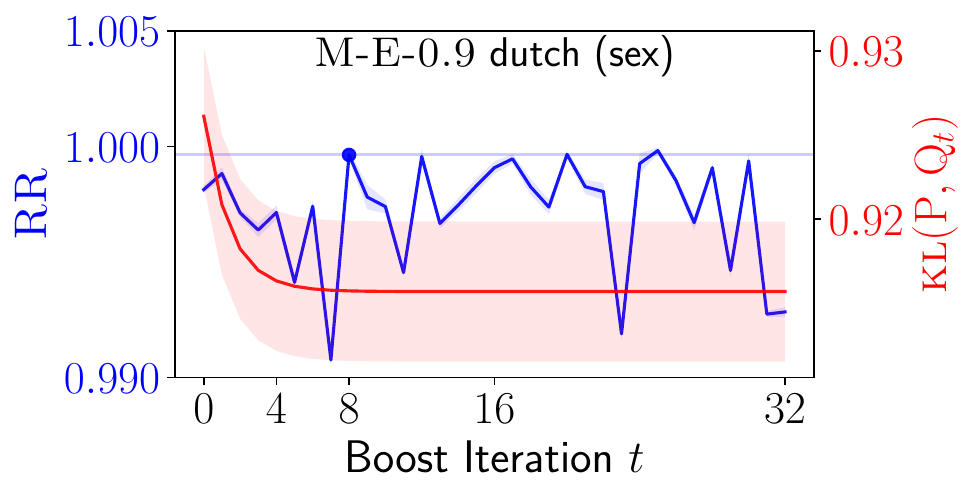}%
    \includegraphics[width=0.24\columnwidth]{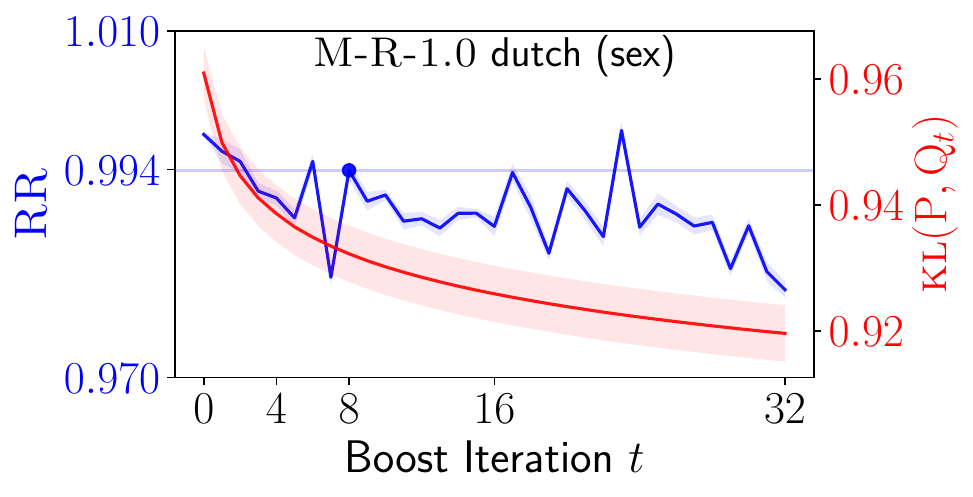}%
    \includegraphics[width=0.24\columnwidth]{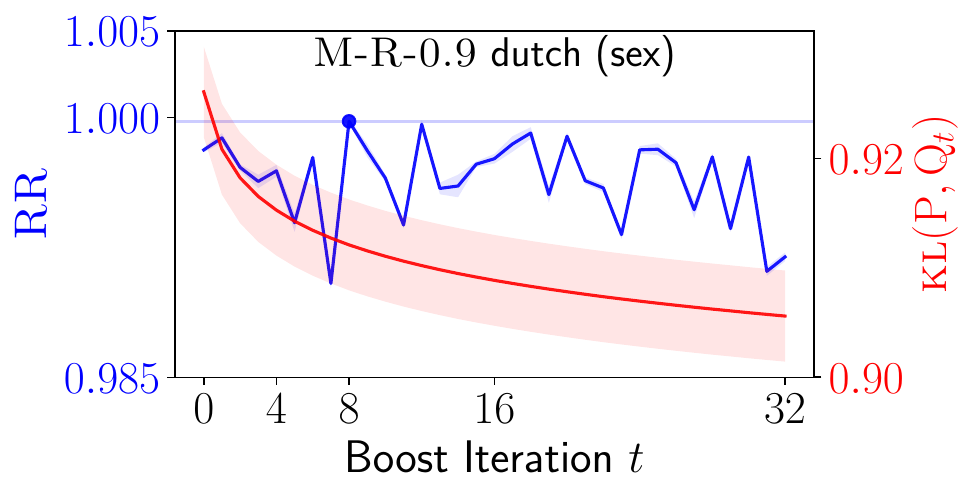}%
    \\
    \includegraphics[width=0.24\columnwidth]{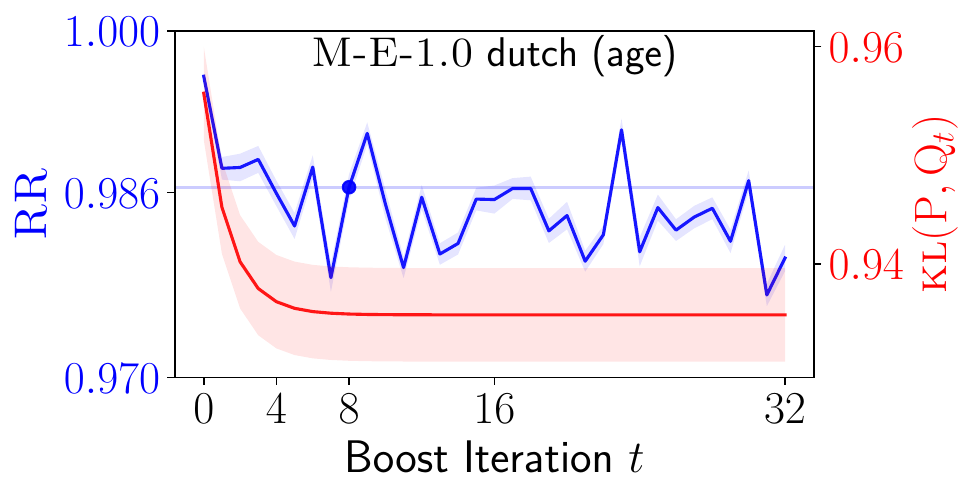}%
    \includegraphics[width=0.24\columnwidth]{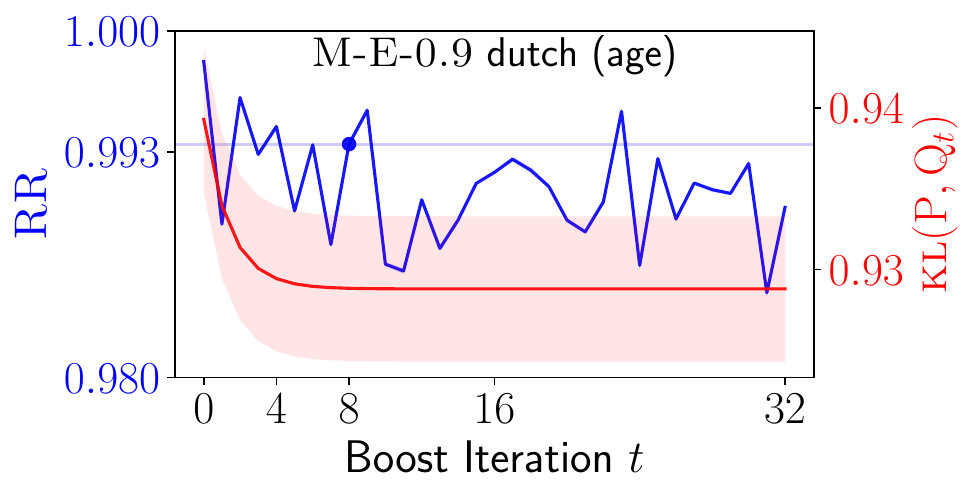}%
    \includegraphics[width=0.24\columnwidth]{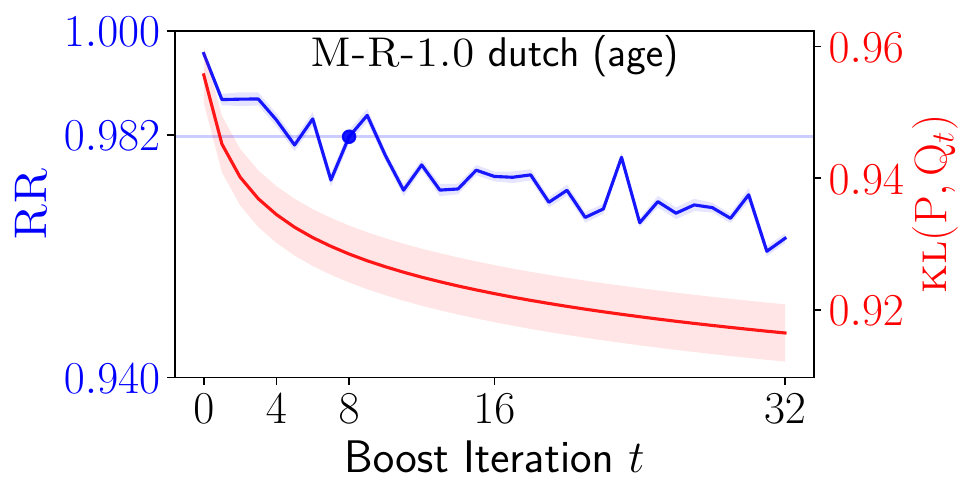}%
    \includegraphics[width=0.24\columnwidth]{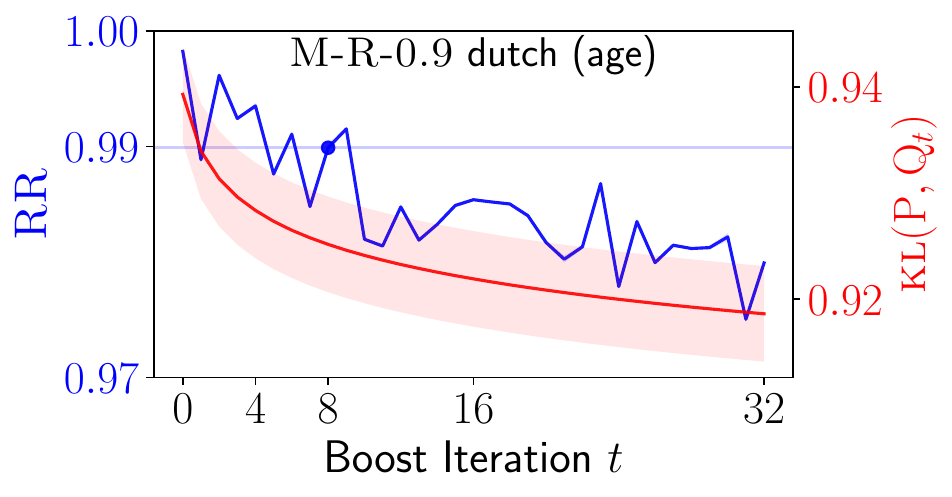}%
    \\
    \includegraphics[width=0.24\columnwidth]{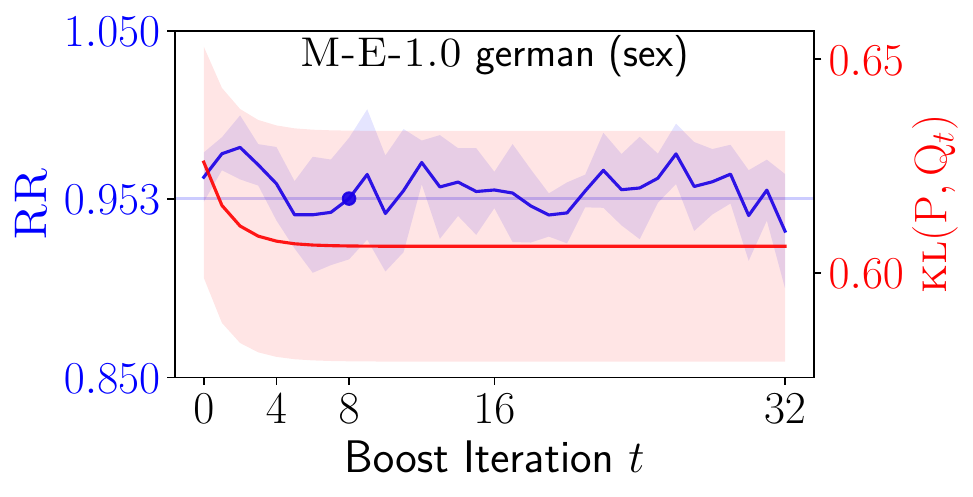}%
    \includegraphics[width=0.24\columnwidth]{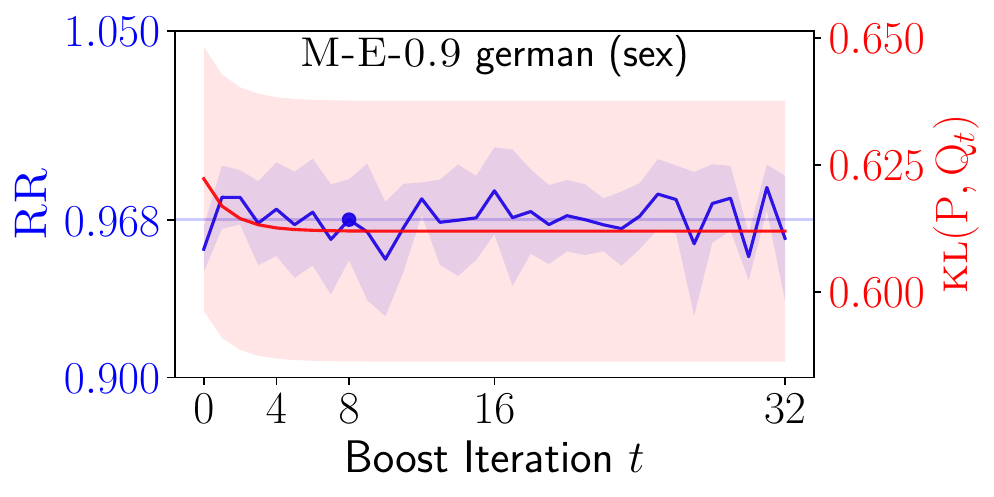}%
    \includegraphics[width=0.24\columnwidth]{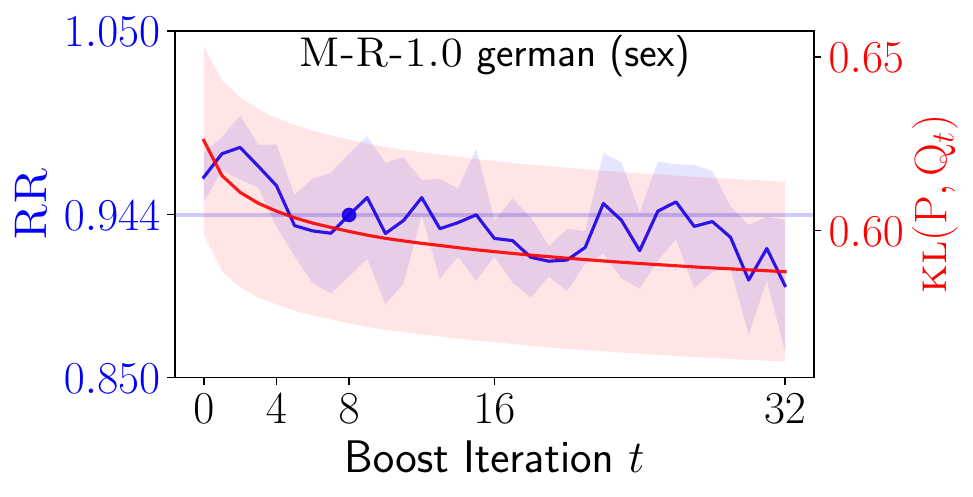}%
    \includegraphics[width=0.24\columnwidth]{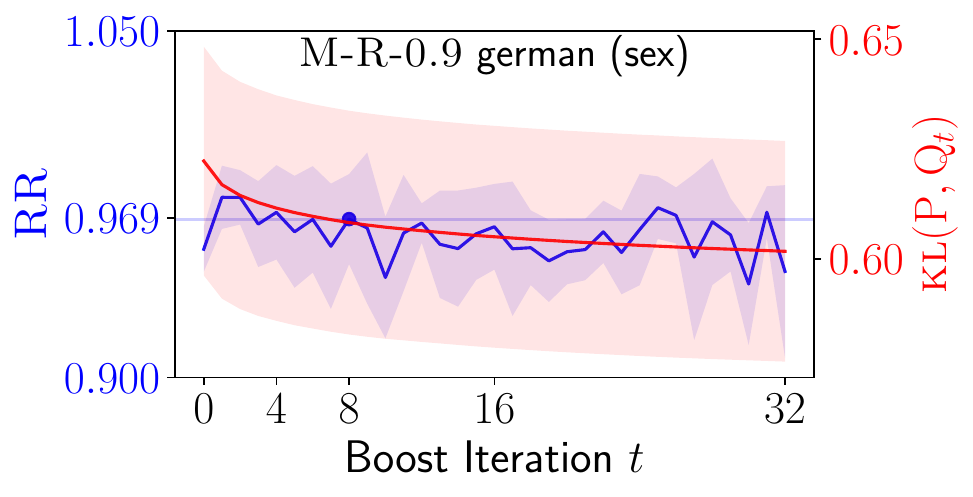}%
    \\
    \includegraphics[width=0.24\columnwidth]{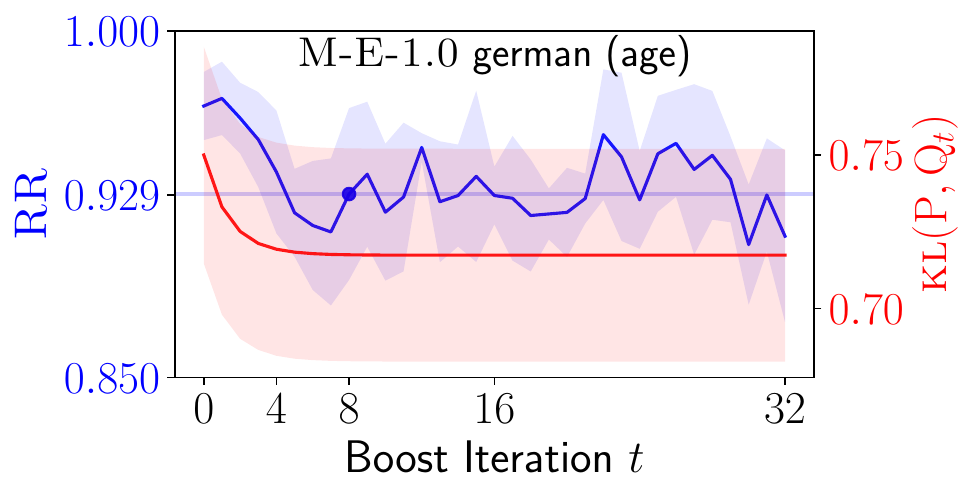}%
    \includegraphics[width=0.24\columnwidth]{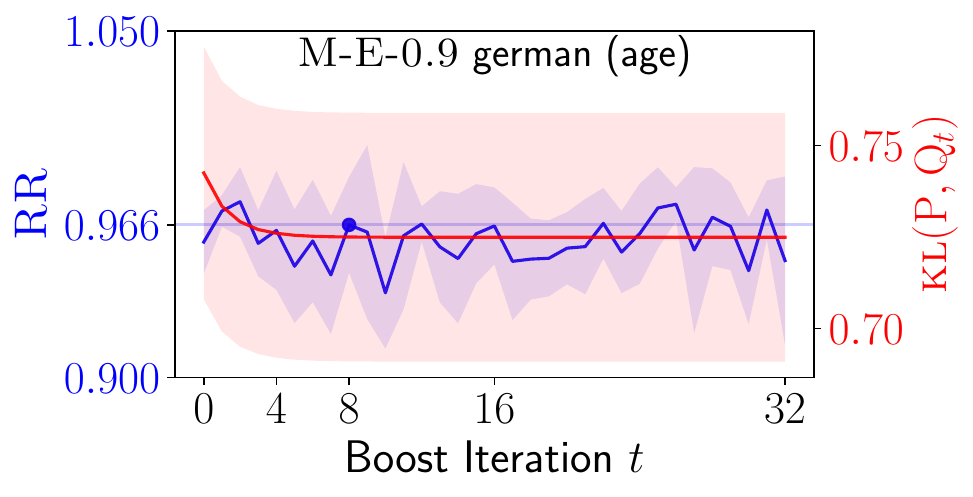}%
    \includegraphics[width=0.24\columnwidth]{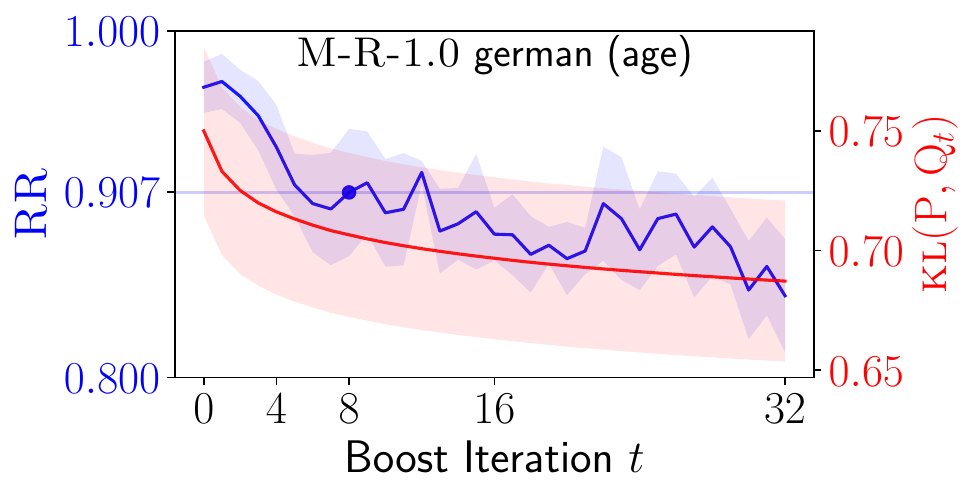}%
    \includegraphics[width=0.24\columnwidth]{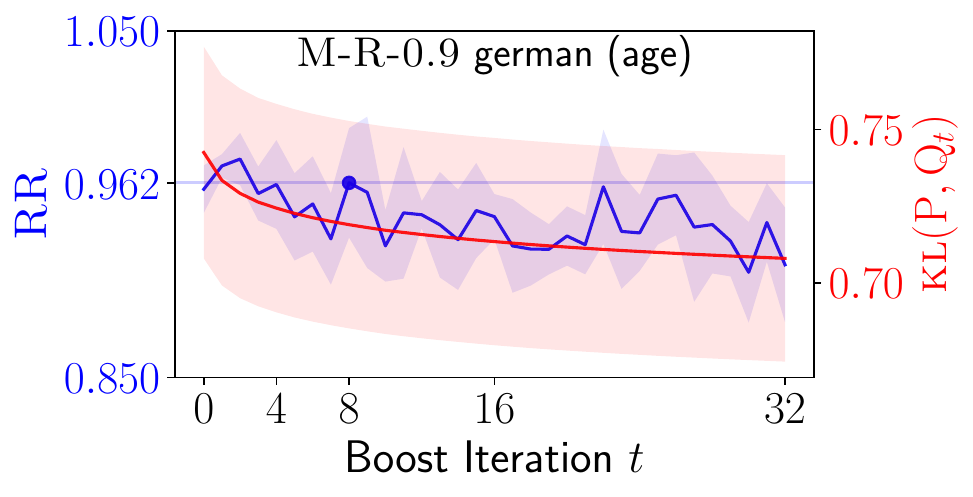}%
    \caption{All RR vs KL over boosting iterations plots for \dutch and \german. horizontal line depicts the \( t = 8 \) rr value.}
    \label{fig:rr_all_plots_extra}
\end{sidewaysfigure}

\begin{sidewaysfigure}[t]
    \centering
    \includegraphics[width=0.24\columnwidth]{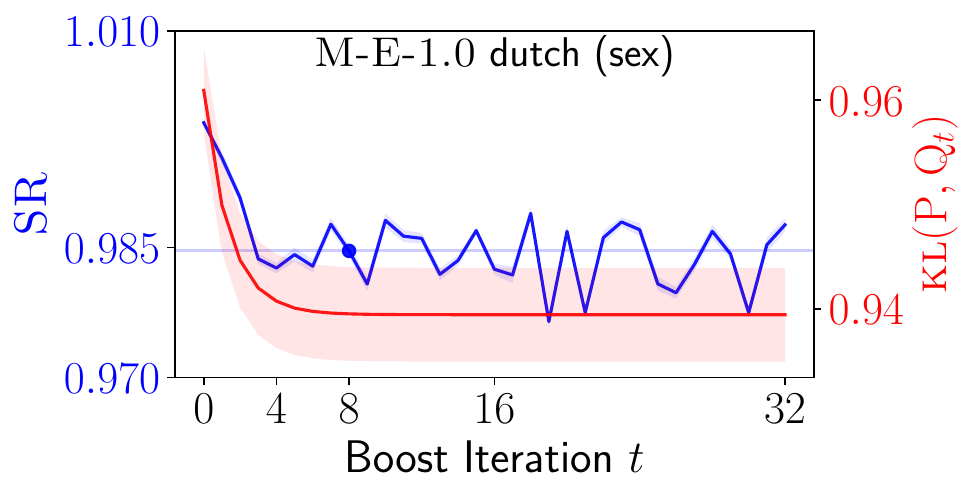}%
    \includegraphics[width=0.24\columnwidth]{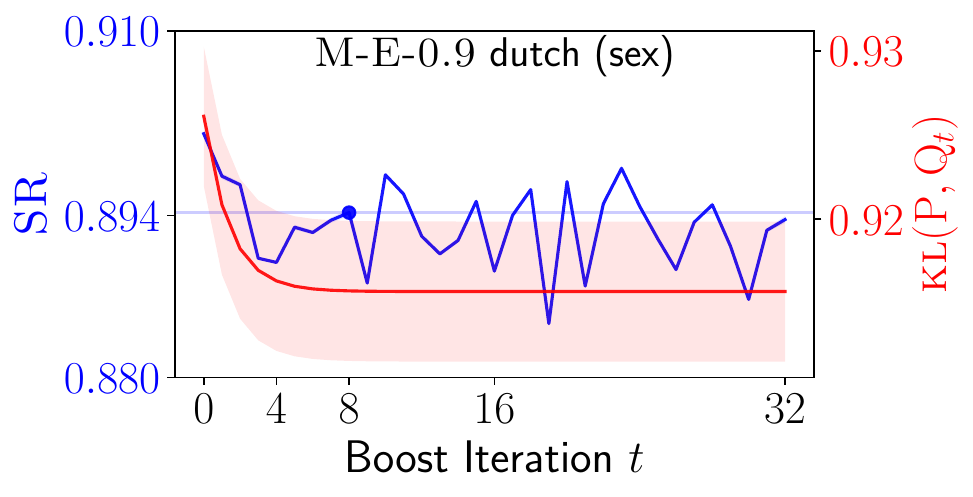}%
    \includegraphics[width=0.24\columnwidth]{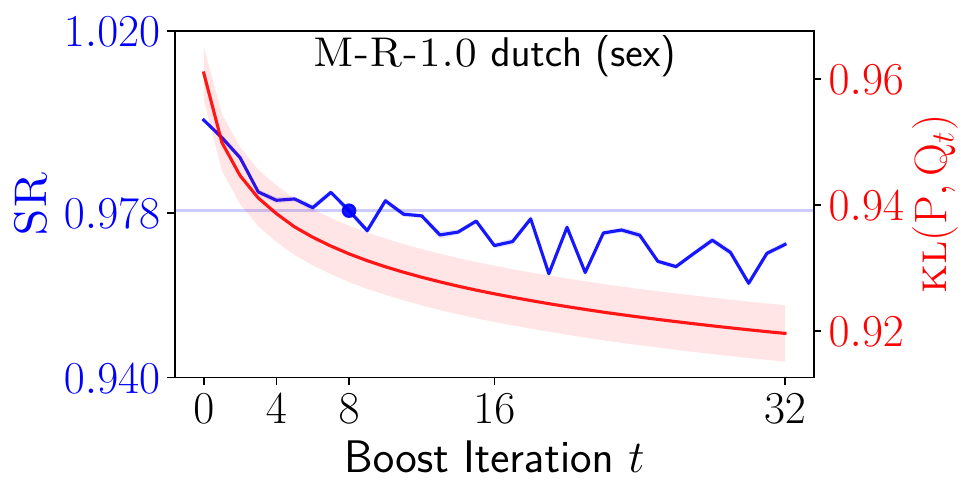}%
    \includegraphics[width=0.24\columnwidth]{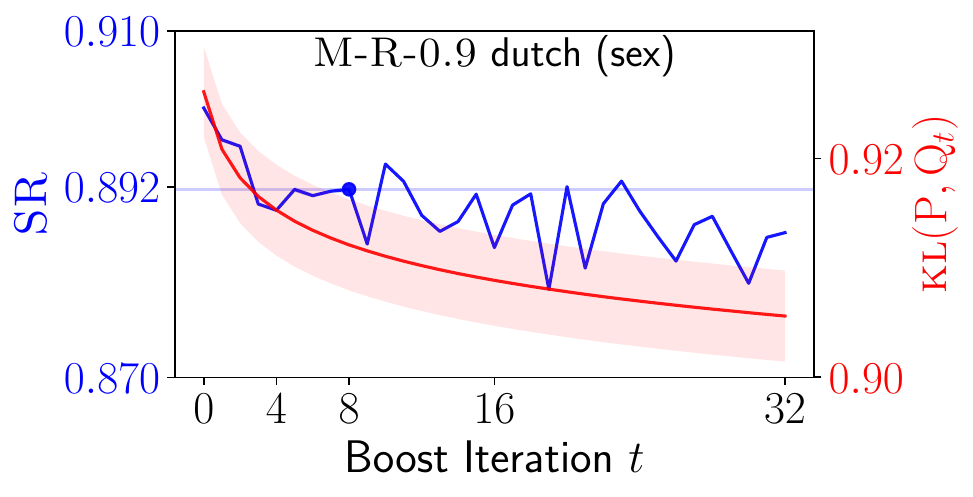}%
    \\
    \includegraphics[width=0.24\columnwidth]{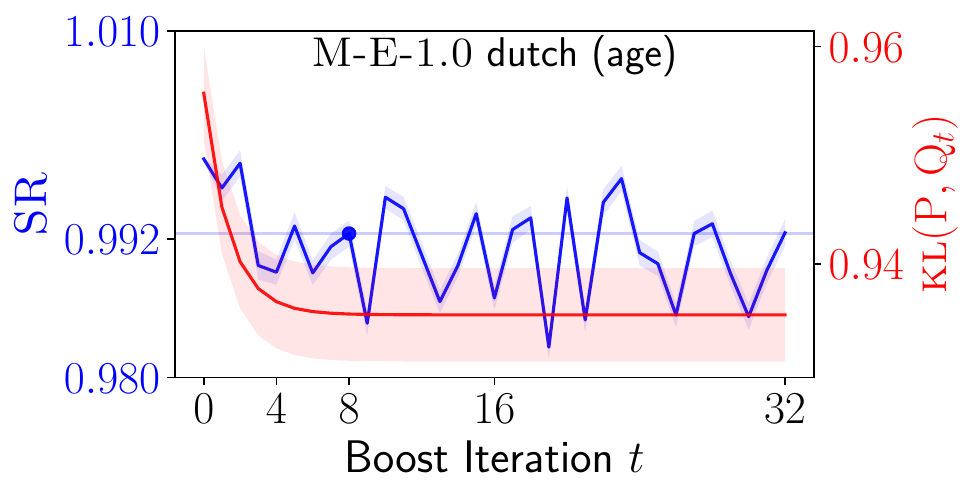}%
    \includegraphics[width=0.24\columnwidth]{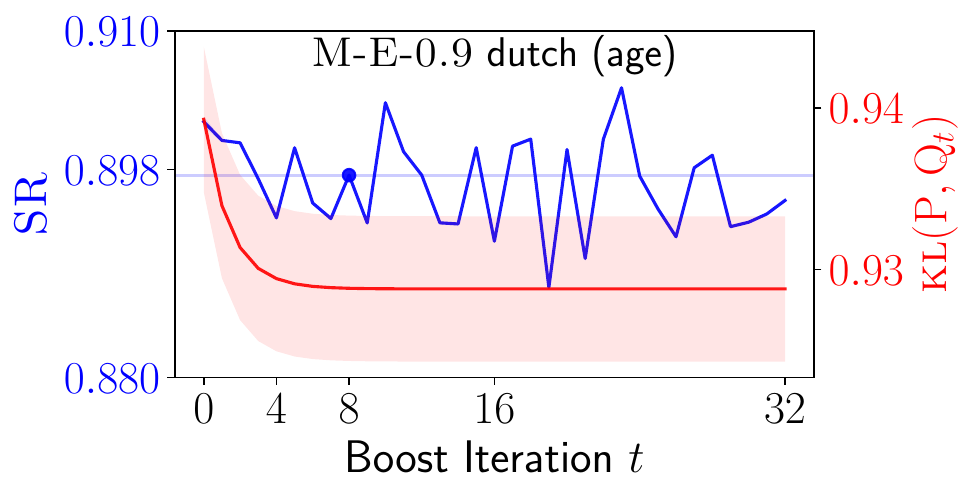}%
    \includegraphics[width=0.24\columnwidth]{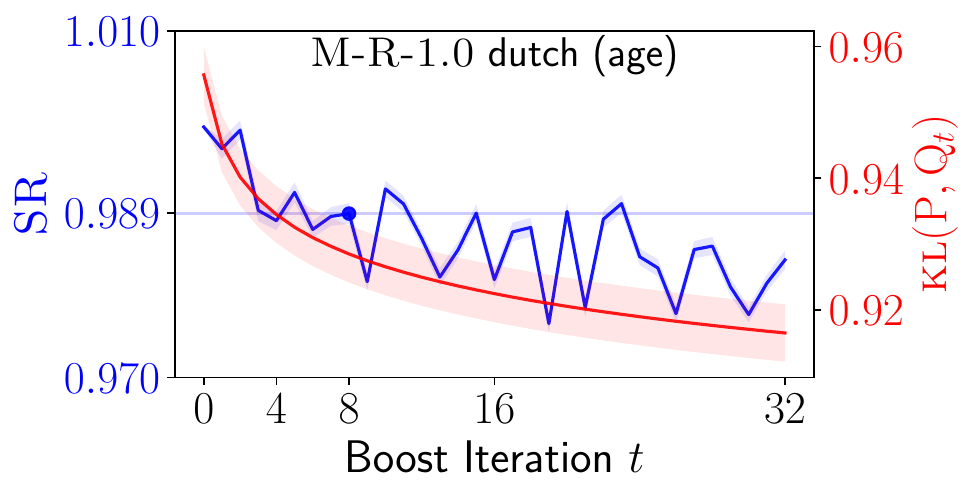}%
    \includegraphics[width=0.24\columnwidth]{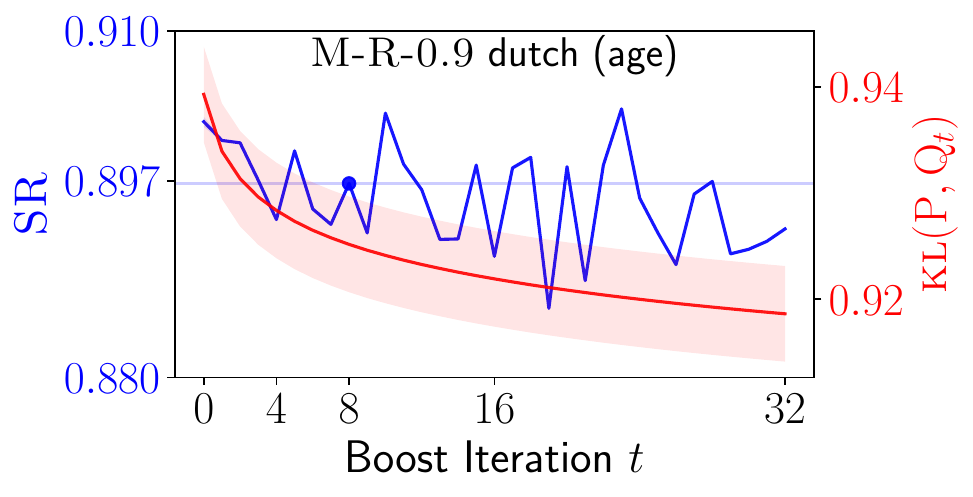}%
    \\
    \includegraphics[width=0.24\columnwidth]{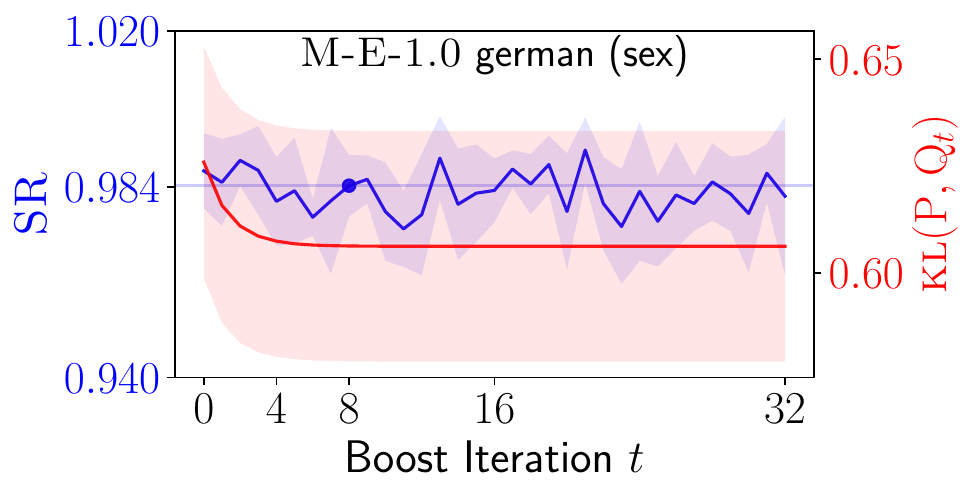}%
    \includegraphics[width=0.24\columnwidth]{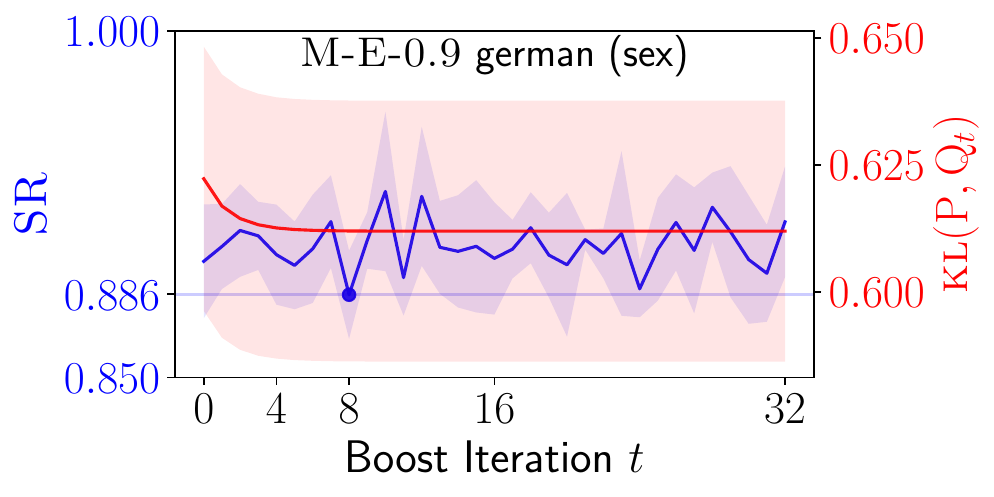}%
    \includegraphics[width=0.24\columnwidth]{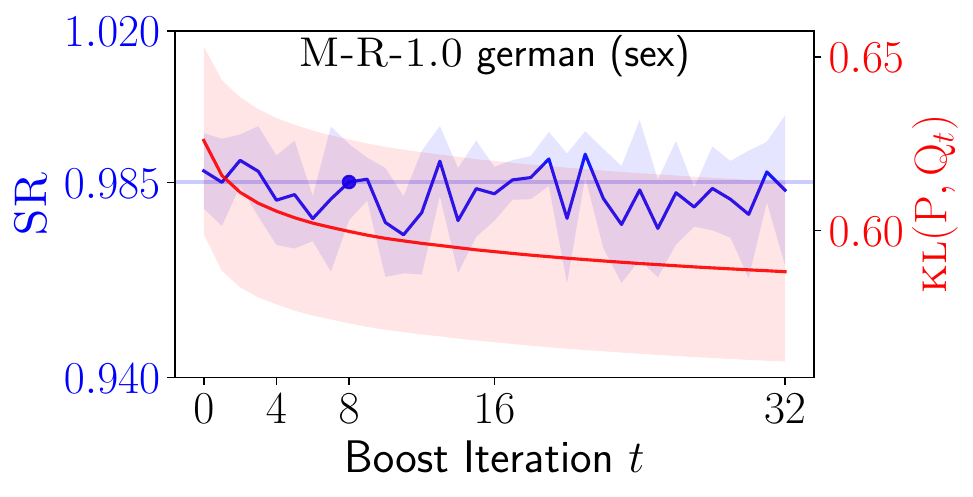}%
    \includegraphics[width=0.24\columnwidth]{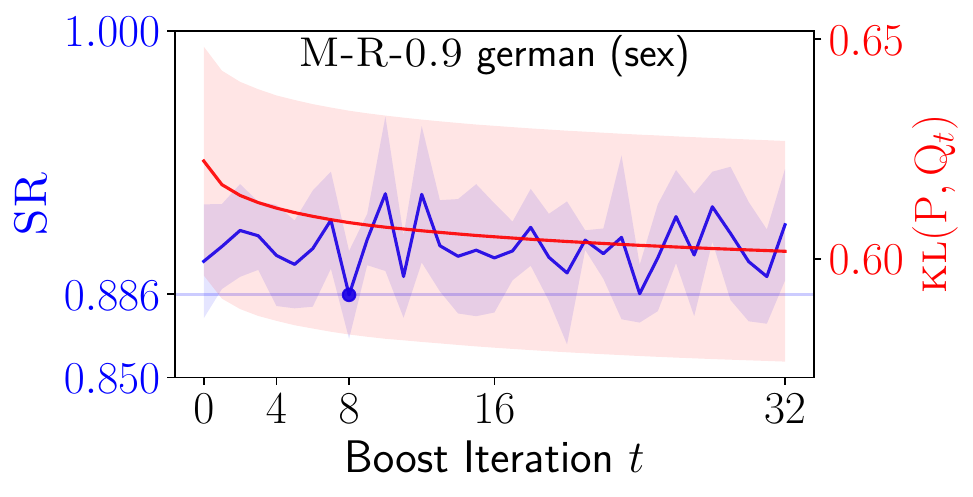}%
    \\
    \includegraphics[width=0.24\columnwidth]{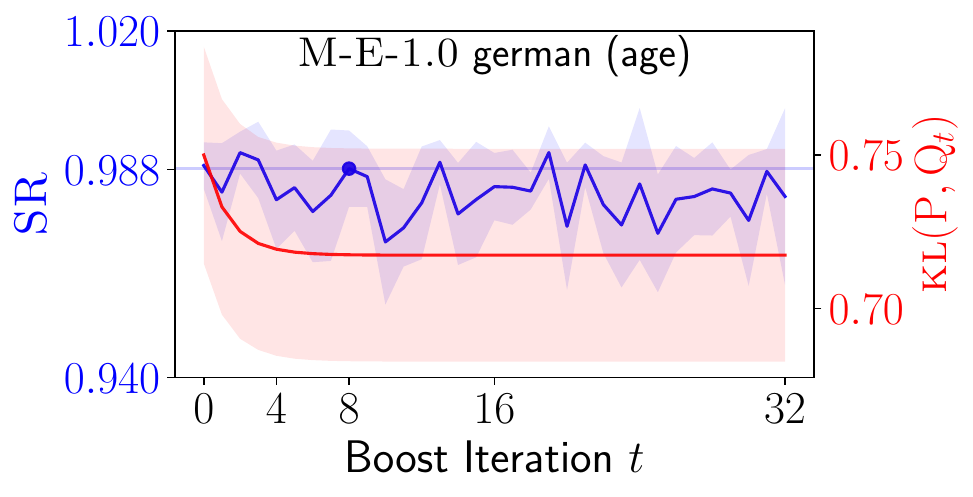}%
    \includegraphics[width=0.24\columnwidth]{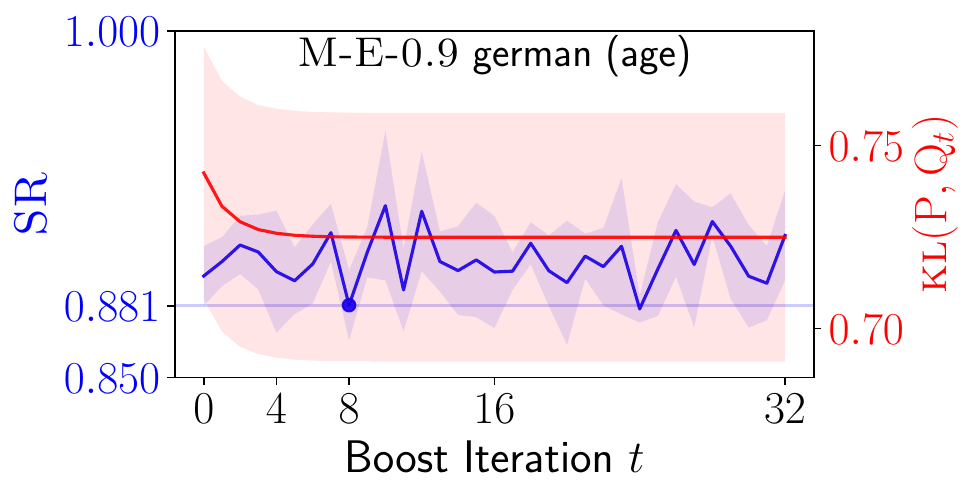}%
    \includegraphics[width=0.24\columnwidth]{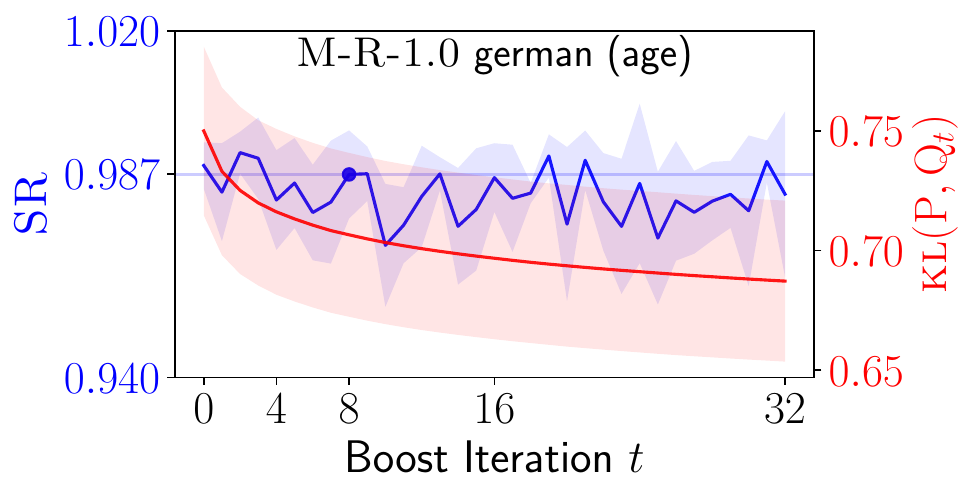}%
    \includegraphics[width=0.24\columnwidth]{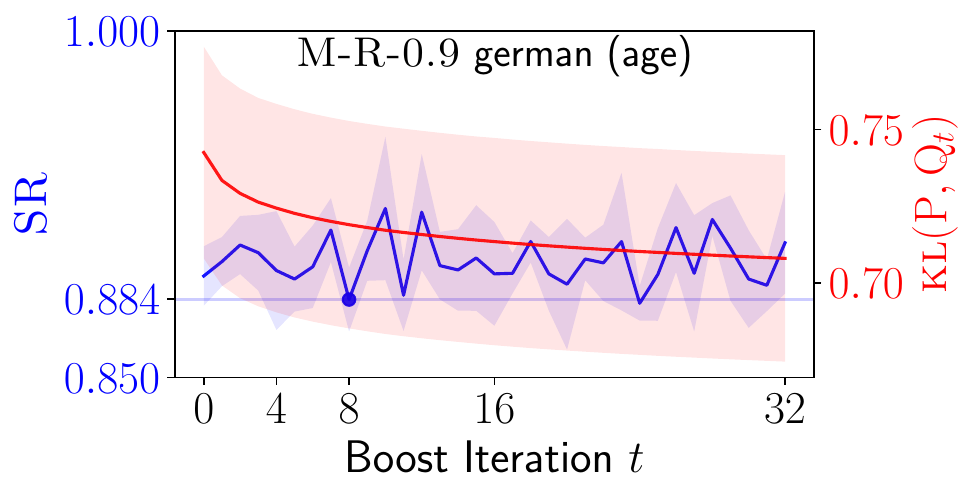}%
    \caption{All SR vs KL over boosting iterations plots for \dutch and \german. horizontal line depicts the \( t = 8 \) sr value.}
    \label{fig:sr_all_plots_extra}
\end{sidewaysfigure}

\begin{sidewaysfigure}[t]
    \centering
    \includegraphics[width=0.24\columnwidth]{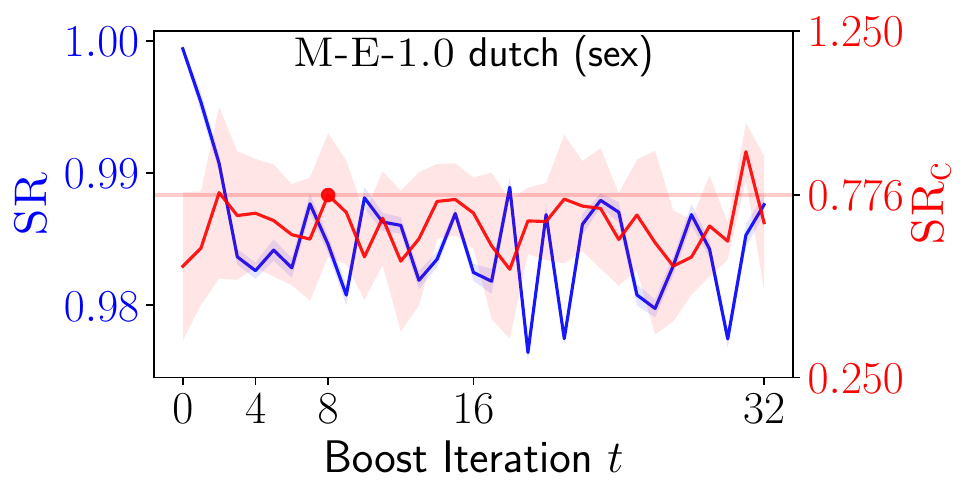}%
    \includegraphics[width=0.24\columnwidth]{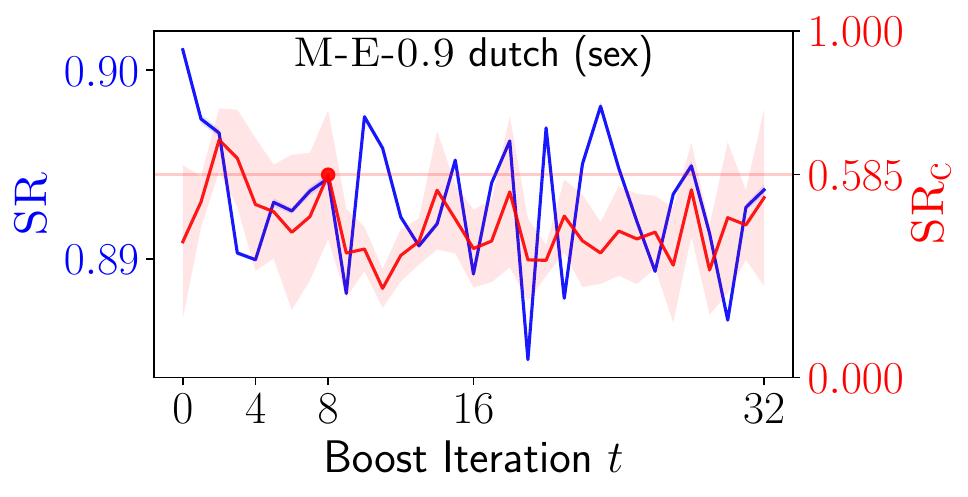}%
    \includegraphics[width=0.24\columnwidth]{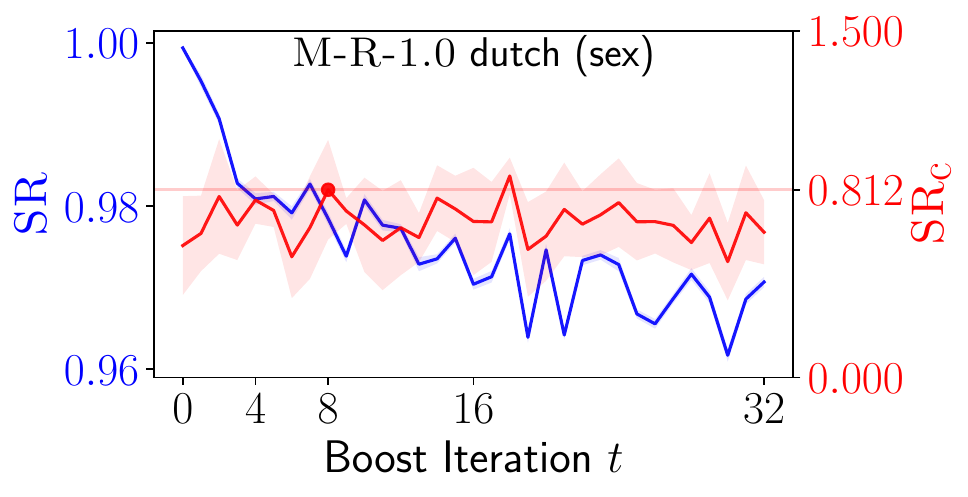}%
    \includegraphics[width=0.24\columnwidth]{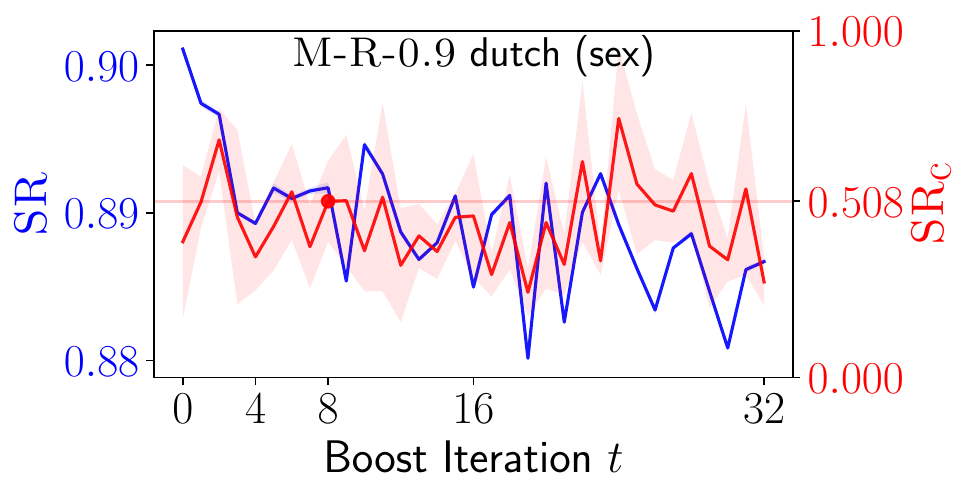}%
    \\
    \includegraphics[width=0.24\columnwidth]{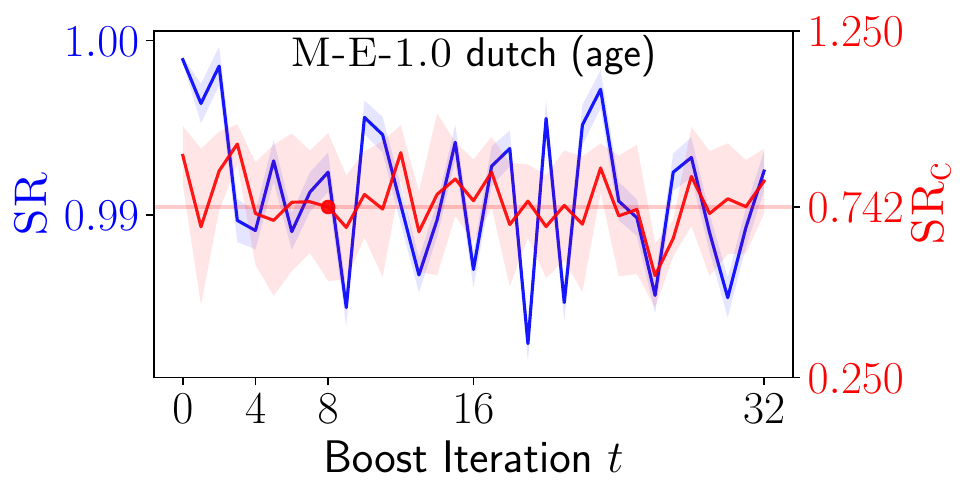}%
    \includegraphics[width=0.24\columnwidth]{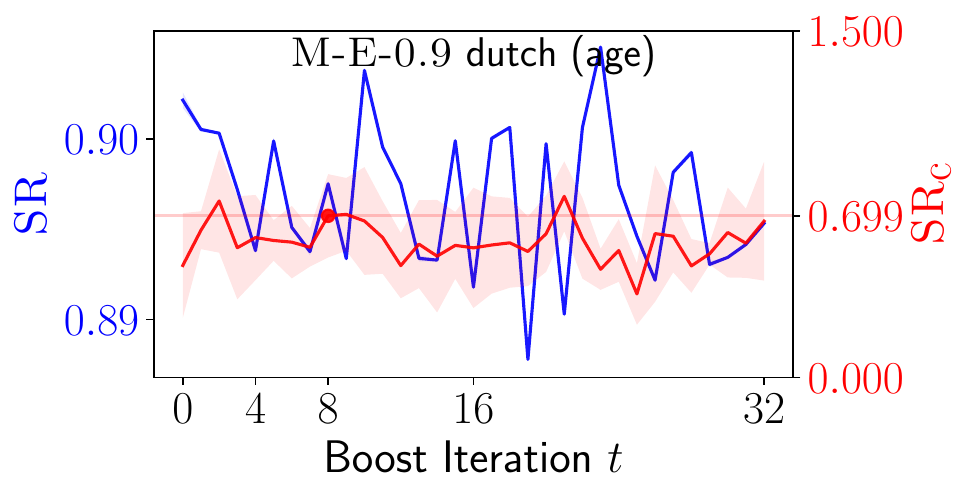}%
    \includegraphics[width=0.24\columnwidth]{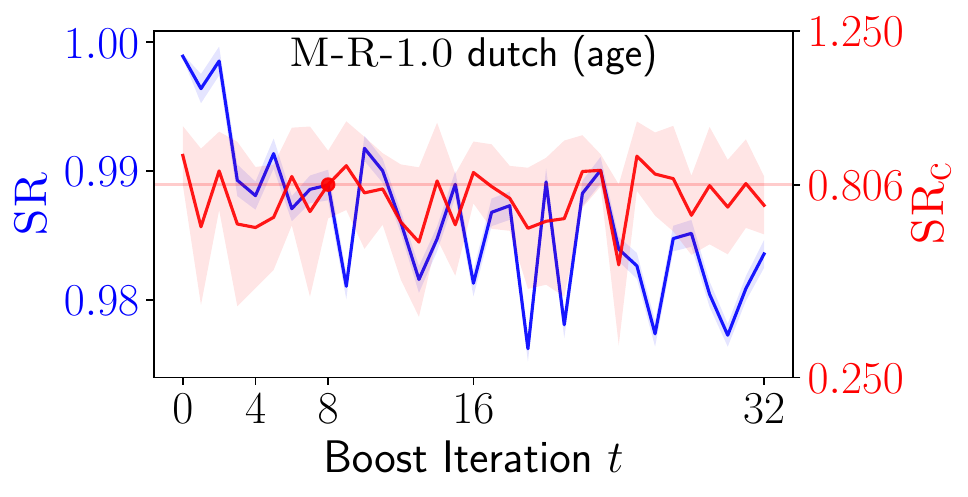}%
    \includegraphics[width=0.24\columnwidth]{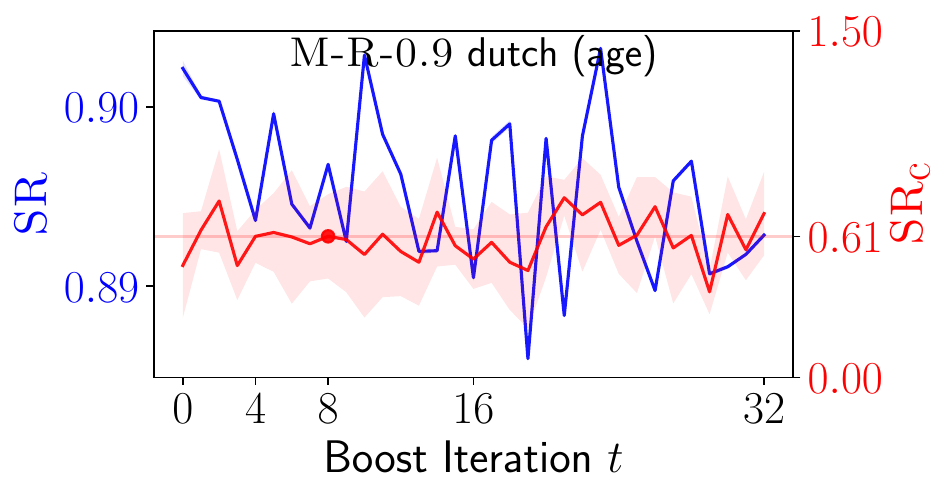}%
    \\
    \includegraphics[width=0.24\columnwidth]{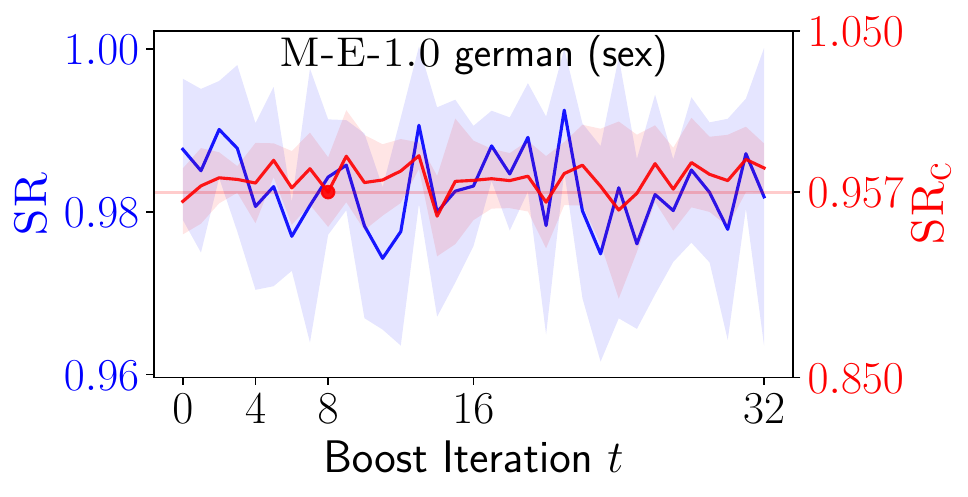}%
    \includegraphics[width=0.24\columnwidth]{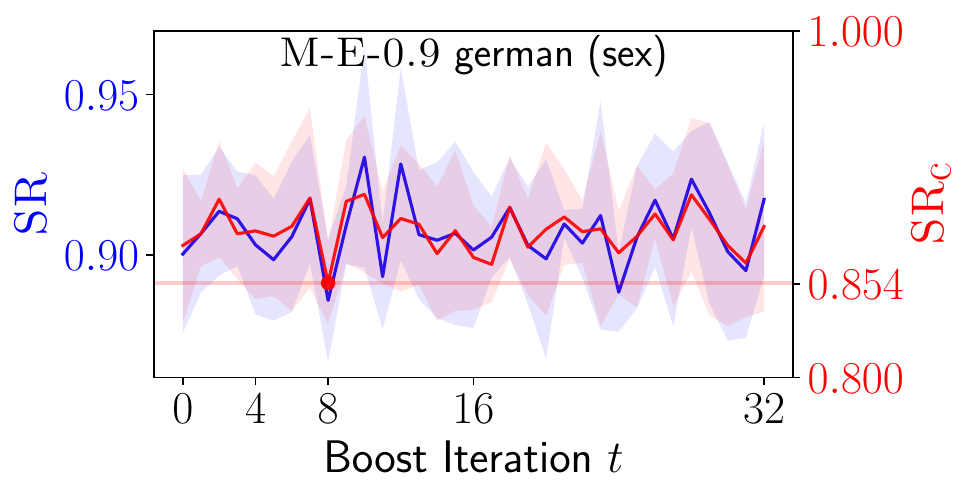}%
    \includegraphics[width=0.24\columnwidth]{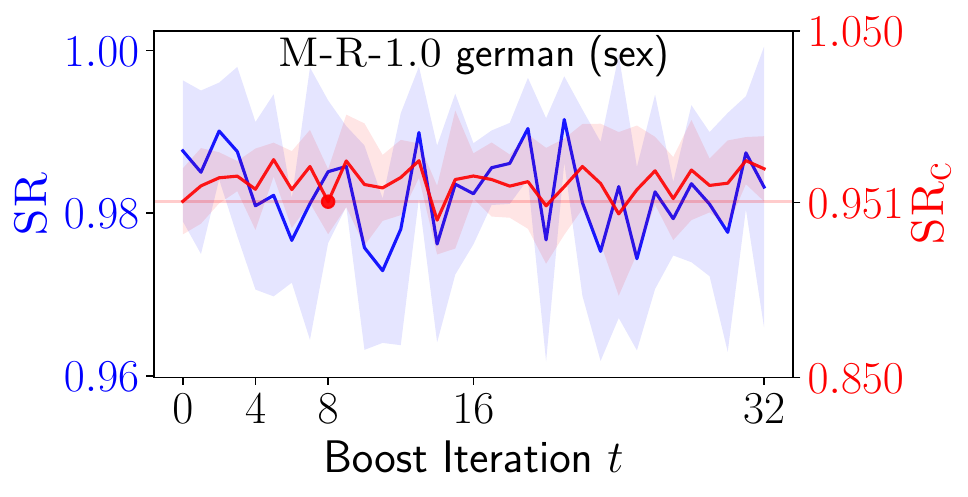}%
    \includegraphics[width=0.24\columnwidth]{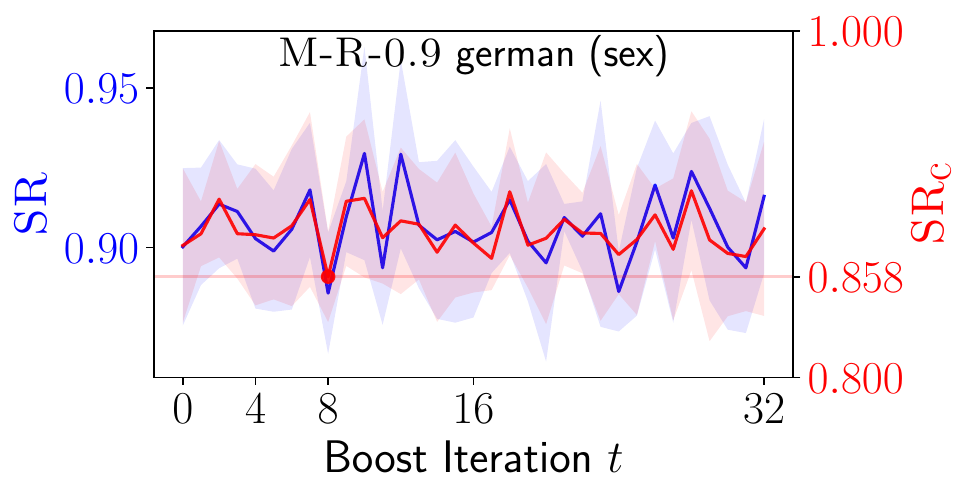}%
    \\
    \includegraphics[width=0.24\columnwidth]{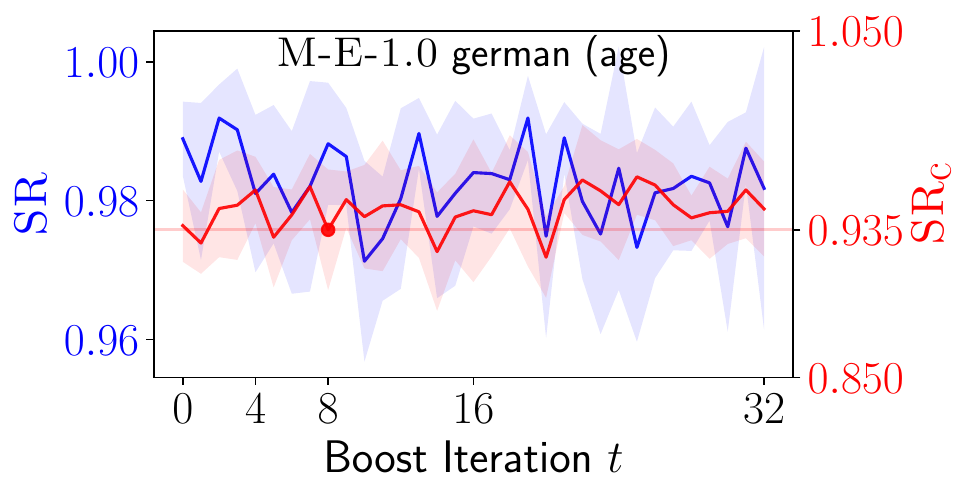}%
    \includegraphics[width=0.24\columnwidth]{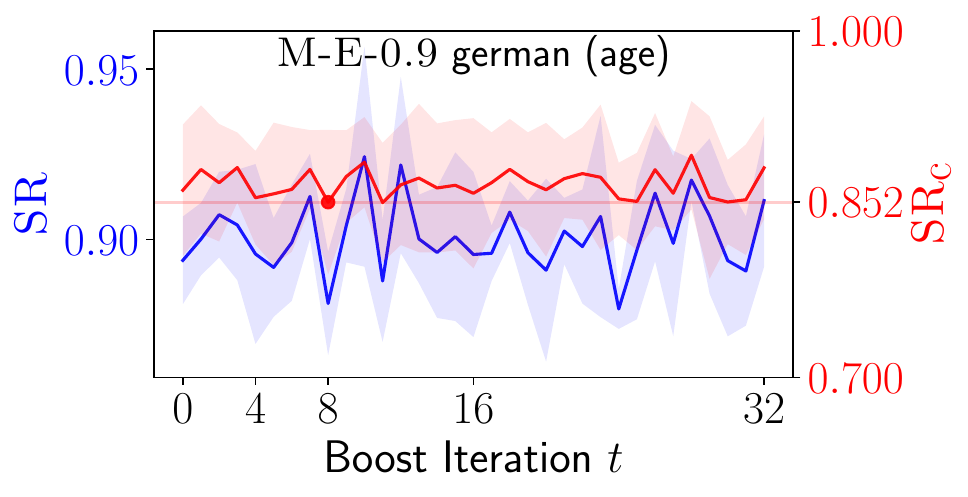}%
    \includegraphics[width=0.24\columnwidth]{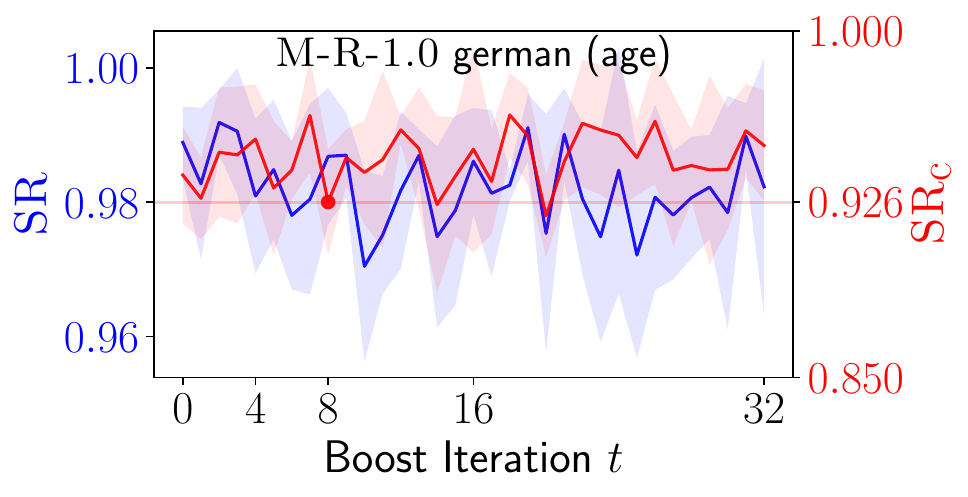}%
    \includegraphics[width=0.24\columnwidth]{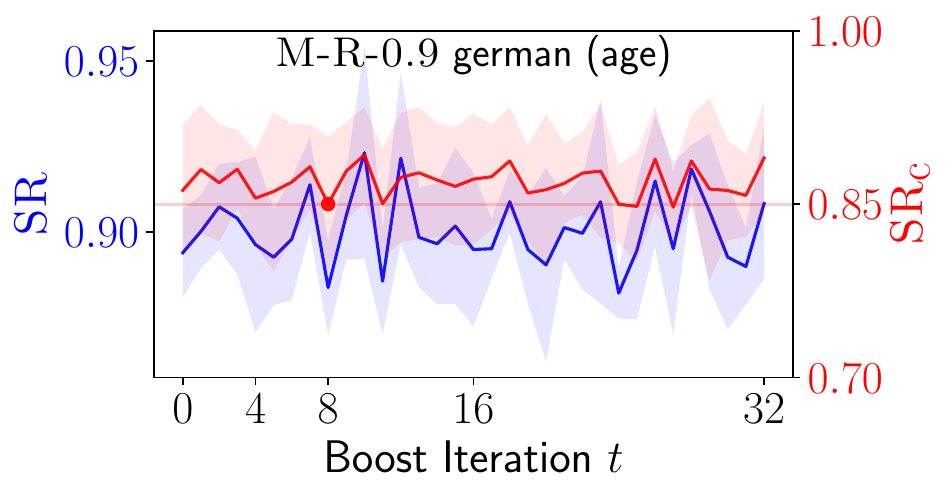}%
    \caption{All SR vs Classification SR over boosting iterations plots for \dutch and \german. Horizontal line depicts the \( T = 8 \) \( \SR_{\textrm{c}}\) value.}
    \label{fig:srclf_all_plots_extra}
\end{sidewaysfigure}

\section{\dutch and \german Without Mixing Priors}
\label{sec:dutch_german_without_mixing}

Experimentally, in \cref{sec:extra_discrete_experiments} we utilize a initial distribution mixing with \( \meas{Q}_{\prior} \) as a uniform distribution and \( \alpha = 1 / 2\) (as per \cref{sec:mixing_prior}). In this case, our initial distribution is very similar to the prior distribution utilized in \citet[Section 3.1 ``Prior distributions'']{ckvDP}. This seems to explain why \celisa performs so well in KL for \dutch and \german datasets, where \tabfairgan does not.

Indeed, \cref{tab:no_mixing_experiments_table} presents a repeat of those experiments without prior mixing. In this case, the KL divergence is significantly worse. Despite this, it should also be noted that the empirical distribution's KL (comparing the training set vs the test set) is also quite large.

\begin{table*}[t]
    \caption{\fbde~(\(T = 32\)) for \dutch and \german datasets without mixing. The table reports the mean and s.t.d.}
    \label{tab:no_mixing_experiments_table}
    \begin{center}
    \begin{small}
    \begin{sc}
    \begin{tabularx}{\textwidth}{LLPDDDDD} %{LLMXXXXXXXX}
        \toprule
        && & Data & \mollifiera & \mollifierb & \mollifierc & \mollifierd \\
        \midrule

{\multirow{6}{*}{\rotatebox[origin=c]{90}{\dutch}}} &
{\multirow{3}{*}{\rotatebox[origin=c]{90}{sex}}} 
& $ \textrm{RR} $                 & 
\(.996 \pm .002\) & \(.984 \pm .000\) & \(.992 \pm .000\) & \(.977 \pm .000\) & \(.990 \pm .000\) \\
&& $ \textrm{SR} $                & 
\(.523 \pm .002\) & \(.978 \pm .000\) & \(.889 \pm .000\) & \(.952 \pm .000\) & \(.878 \pm .000\) \\
&& $ \textrm{KL} $                & 
\(1.715 \pm .021\) & \(1.797 \pm .022\) & \(1.768 \pm .022\) & \(1.788 \pm .022\) & \(1.765 \pm .022\) \\

        \cmidrule(l{5pt}){2-8}

&
{\multirow{3}{*}{\rotatebox[origin=c]{90}{age}}} 
& $ \textrm{RR} $                 &  
\(.621 \pm .001\) & \(.966 \pm .000\) & \(.985 \pm .000\) & \(.938 \pm .000\) & \(.969 \pm .000\) \\
&& $ \textrm{SR} $                &  
\(.610 \pm .003\) & \(.979 \pm .000\) & \(.889 \pm .000\) & \(.958 \pm .000\) & \(.880 \pm .000\) \\
&& $ \textrm{KL} $                &  
\(1.716 \pm .021\) & \(1.773 \pm .022\) & \(1.761 \pm .022\) & \(1.767 \pm .022\) & \(1.758 \pm .022\) \\

        \midrule

{\multirow{6}{*}{\rotatebox[origin=c]{90}{\german}}} &
{\multirow{3}{*}{\rotatebox[origin=c]{90}{sex}}} 
& $ \textrm{RR} $                 & 
\(.449 \pm .014\) & \(.933 \pm .036\) & \(.959 \pm .027\) & \(.896 \pm .039\) & \(.943 \pm .039\) \\
&& $ \textrm{SR} $                & 
\(.897 \pm .018\) & \(.983 \pm .018\) & \(.919 \pm .024\) & \(.985 \pm .014\) & \(.917 \pm .024\) \\
&& $ \textrm{KL} $                & 
\(1.178 \pm .278\) & \(1.245 \pm .286\) & \(1.246 \pm .286\) & \(1.238 \pm .285\) & \(1.242 \pm .286\) \\

        \cmidrule(l{5pt}){2-8}

&
{\multirow{3}{*}{\rotatebox[origin=c]{90}{age}}} 
& $ \textrm{RR} $                 & 
\(.235 \pm .008\) & \(.906 \pm .037\) & \(.950 \pm .037\) & \(.838 \pm .032\) & \(.908 \pm .033\) \\
&& $ \textrm{SR} $                & 
\(.794 \pm .029\) & \(.984 \pm .018\) & \(.910 \pm .020\) & \(.983 \pm .015\) & \(.905 \pm .022\) \\
&& $ \textrm{KL} $                & 
\(1.192 \pm .294\) & \(1.382 \pm .264\) & \(1.386 \pm .266\) & \(1.360 \pm .265\) & \(1.374 \pm .266\) \\

        \bottomrule
    \end{tabularx}
    \end{sc}
    \end{small}
    \end{center}
\end{table*}
\section{In-Processing Experiments}
\label{sec:in_processing_experiments}

To evaluate \fbde learned predictions against in-processing algorithms, we consider the following in-processing approaches:
\begin{itemize}
    \item \faircons: from \citet{zvggFC};
    \item \reduct: from \citet{abdlwAR};
    \item \fairglm: from \citet{do2022fair}.
\end{itemize}
We utilize \citet{do2022fair}'s implementation of these in-processing algorithms, provided in \url{www.github.com/hyungrok-do/fair-glm-cvx}. \cref{tab:inprocessing} provides a summary of these comparisons.

\begin{table*}[t]
    \caption{\fbde~(\(T = 32\)) for \compas, \adult, \dutch, and \german datasets compared against in-processing algorithms. The table reports the mean and s.t.d.}
    \label{tab:inprocessing}
    \scriptsize
    \begin{center}
    \begin{sc}
    \begin{tabularx}{\textwidth}{LLPDDDDDDDD} %{LLMXXXXXXXX}
        \toprule
        && & Data & \mollifiera & \mollifierb & \mollifierc & \mollifierd & \faircons & \reduct & \fairglm  \\
        \midrule

{\multirow{6}{*}{\rotatebox[origin=c]{90}{\compas}}} &
{\multirow{3}{*}{\rotatebox[origin=c]{90}{sex}}} 
& $ \textrm{SR}_{\textrm{c}}$     &  
\(.726 \pm .025\) & \(.952 \pm .006\) & \(.874 \pm .009\) & \(.938 \pm .019\) & \(.874 \pm .016\) & \(.980 \pm .012\) & \(.874 \pm .021\) & \(.874 \pm .016\) \\
&& $ \textrm{EO} $                &  
\(.764 \pm .031\) & \(.966 \pm .021\) & \(.905 \pm .027\) & \(.963 \pm .027\) & \(.907 \pm .030\) & \(.981 \pm .012\) & \(.899 \pm .028\) & \(.899 \pm .027\) \\
&& $ \textrm{Acc} $               &  
\(.660 \pm .004\) & \(.651 \pm .008\) & \(.654 \pm .005\) & \(.657 \pm .009\) & \(.655 \pm .008\) & \(.573 \pm .027\) & \(.665 \pm .009\) & \(.667 \pm .007\) \\

        \cmidrule(l{5pt}){2-11}

&
{\multirow{3}{*}{\rotatebox[origin=c]{90}{race}}} 
& $ \textrm{SR}_{\textrm{c}}$     &  
\(.747 \pm .020\) & \(.959 \pm .025\) & \(.875 \pm .025\) & \(.945 \pm .027\) & \(.872 \pm .024\) & \(.997 \pm .001\) & \(.822 \pm .043\) & \(.797 \pm .008\) \\
&& $ \textrm{EO} $                &  
\(.781 \pm .034\) & \(.960 \pm .026\) & \(.900 \pm .041\) & \(.950 \pm .028\) & \(.895 \pm .039\) & \(.987 \pm .007\) & \(.858 \pm .033\) & \(.838 \pm .033\) \\
&& $ \textrm{Acc} $               &  
\(.660 \pm .004\) & \(.641 \pm .014\) & \(.653 \pm .012\) & \(.642 \pm .010\) & \(.656 \pm .012\) & \(.515 \pm .016\) & \(.664 \pm .009\) & \(.664 \pm .009\) \\
        \midrule

{\multirow{6}{*}{\rotatebox[origin=c]{90}{\adult}}} &
{\multirow{3}{*}{\rotatebox[origin=c]{90}{sex}}} 
& $ \textrm{SR}_{\textrm{c}}$     &  
\(.360 \pm .003\) & \(.818 \pm .010\) & \(.766 \pm .010\) & \(.793 \pm .011\) & \(.753 \pm .008\) & \(.920 \pm .013\) & \(.814 \pm .011\) & \(.857 \pm .009\) \\
&& $ \textrm{EO} $                &  
\(.471 \pm .008\) & \(.959 \pm .016\) & \(.908 \pm .018\) & \(.935 \pm .016\) & \(.895 \pm .016\) & \(.983 \pm .009\) & \(.964 \pm .023\) & \(.979 \pm .018\) \\
&& $ \textrm{Acc} $               &  
\(.803 \pm .003\) & \(.785 \pm .002\) & \(.788 \pm .002\) & \(.787 \pm .002\) & \(.788 \pm .002\) & \(.786 \pm .003\) & \(.789 \pm .002\) & \(.787 \pm .003\) \\

        \cmidrule(l{5pt}){2-11}

&
{\multirow{3}{*}{\rotatebox[origin=c]{90}{race}}} 
& $ \textrm{SR}_{\textrm{c}}$     &  
\(.600 \pm .031\) & \(.867 \pm .035\) & \(.808 \pm .031\) & \(.860 \pm .034\) & \(.803 \pm .035\) & \(.806 \pm .020\) & \(.817 \pm .018\) & \(.843 \pm .015\) \\
&& $ \textrm{EO} $                &  
\(.787 \pm .053\) & \(.933 \pm .019\) & \(.960 \pm .037\) & \(.935 \pm .017\) & \(.957 \pm .034\) & \(.964 \pm .036\) & \(.969 \pm .027\) & \(.970 \pm .029\) \\
&& $ \textrm{Acc} $               &  
\(.803 \pm .003\) & \(.800 \pm .002\) & \(.801 \pm .003\) & \(.800 \pm .002\) & \(.801 \pm .003\) & \(.803 \pm .002\) & \(.803 \pm .003\) & \(.803 \pm .002\) \\

        \midrule

{\multirow{6}{*}{\rotatebox[origin=c]{90}{\dutch}}} &
{\multirow{3}{*}{\rotatebox[origin=c]{90}{sex}}} 
& $ \textrm{SR}_{\textrm{c}}$     &  
\(.464 \pm .022\) & \(.696 \pm .193\) & \(.518 \pm .255\) & \(.628 \pm .139\) & \(.275 \pm .068\) & \(1.00 \pm .000\) & \(1.00 \pm .000\) & \(1.00 \pm .000\) \\
&& $ \textrm{EO} $                &  
\(.890 \pm .006\) & \(.973 \pm .005\) & \(.986 \pm .008\) & \(.980 \pm .007\) & \(.987 \pm .007\) & \(.821 \pm .011\) & \(.958 \pm .012\) & \(.938 \pm .008\) \\
&& $ \textrm{Acc} $               &  
\(.827 \pm .001\) & \(.763 \pm .003\) & \(.778 \pm .004\) & \(.770 \pm .003\) & \(.781 \pm .001\) & \(.771 \pm .004\) & \(.818 \pm .003\) & \(.816 \pm .003\) \\

        \cmidrule(l{5pt}){2-11}

&
{\multirow{3}{*}{\rotatebox[origin=c]{90}{age}}} 
& $ \textrm{SR}_{\textrm{c}}$     &  
\(.695 \pm .030\) & \(.816 \pm .094\) & \(.676 \pm .257\) & \(.747 \pm .084\) & \(.709 \pm .180\) & \(1.00 \pm .000\) & \(1.00 \pm .000\) & \(1.00 \pm .000\) \\
&& $ \textrm{EO} $                &  
\(.892 \pm .009\) & \(.993 \pm .005\) & \(.960 \pm .007\) & \(.987 \pm .007\) & \(.961 \pm .009\) & \(.929 \pm .005\) & \(.912 \pm .011\) & \(.912 \pm .010\) \\
&& $ \textrm{Acc} $               &  
\(.827 \pm .001\) & \(.781 \pm .004\) & \(.788 \pm .004\) & \(.785 \pm .005\) & \(.790 \pm .005\) & \(.797 \pm .003\) & \(.835 \pm .003\) & \(.827 \pm .002\) \\
        \midrule

{\multirow{6}{*}{\rotatebox[origin=c]{90}{\german}}} &
{\multirow{3}{*}{\rotatebox[origin=c]{90}{sex}}} 
& $ \textrm{SR}_{\textrm{c}}$     &  
\(.903 \pm .055\) & \(.971 \pm .014\) & \(.887 \pm .049\) & \(.970 \pm .019\) & \(.886 \pm .050\) & \(.997 \pm .001\) & \(.939 \pm .004\) & \(.936 \pm .024\) \\
&& $ \textrm{EO} $                &  
\(.928 \pm .055\) & \(.980 \pm .021\) & \(.909 \pm .046\) & \(.969 \pm .023\) & \(.907 \pm .051\) & \(.991 \pm .008\) & \(.929 \pm .002\) & \(.923 \pm .019\) \\
&& $ \textrm{Acc} $               &  
\(.690 \pm .022\) & \(.712 \pm .011\) & \(.705 \pm .011\) & \(.714 \pm .019\) & \(.706 \pm .007\) & \(.712 \pm .008\) & \(.700 \pm .015\) & \(.700 \pm .015\) \\

        \cmidrule(l{5pt}){2-11}

&
{\multirow{3}{*}{\rotatebox[origin=c]{90}{age}}} 
& $ \textrm{SR}_{\textrm{c}}$     &  
\(.802 \pm .067\) & \(.947 \pm .027\) & \(.881 \pm .045\) & \(.950 \pm .024\) & \(.890 \pm .048\) & \(.978 \pm .018\) & \(.872 \pm .057\) & \(.910 \pm .028\) \\
&& $ \textrm{EO} $                &  
\(.823 \pm .078\) & \(.940 \pm .028\) & \(.894 \pm .059\) & \(.943 \pm .027\) & \(.914 \pm .058\) & \(.963 \pm .018\) & \(.879 \pm .077\) & \(.916 \pm .047\) \\
&& $ \textrm{Acc} $               &  
\(.690 \pm .022\) & \(.700 \pm .016\) & \(.693 \pm .020\) & \(.697 \pm .013\) & \(.707 \pm .011\) & \(.698 \pm .015\) & \(.690 \pm .015\) & \(.698 \pm .017\) \\

        \bottomrule
    \end{tabularx}
    \end{sc}
    \end{center}
\end{table*}

% TODO Inputs

%\ifnum\ARXIV=1
%\newpage
%\onecolumn
%\appendix

%\setcounter{equation}{0}
%\setcounter{theorem}{0}

%\renewcommand{\theequation}{A.\arabic{equation}}
%\renewcommand{\thefigure}{A.\arabic{figure}}
%\renewcommand{\thetable}{A.\arabic{table}}
%\renewcommand{\thetheorem}{A.\arabic{theorem}}

%\addcontentsline{toc}{section}{Appendix} % Add the appendix text to the document TOC
%\part{Appendix} % Start the appendix part

%\parttoc % Insert the appendix TOC

%\input{supplementary/main}
%\fi

\end{document}